\documentclass{article}

\usepackage{amsmath,amsfonts,bm}

\def\1{\bm{1}}

\def\vx{{\bm{x}}}
\def\vy{{\bm{y}}}

\DeclareMathAlphabet{\mathsfit}{\encodingdefault}{\sfdefault}{m}{sl}
\SetMathAlphabet{\mathsfit}{bold}{\encodingdefault}{\sfdefault}{bx}{n}

\newcommand{\E}{\mathbb{E}}

\newcommand{\R}{\mathbb{R}}

\DeclareMathOperator*{\argmax}{arg\,max}

\usepackage{xcolor}
\definecolor{mydarkred}{rgb}{0.6,0,0}
\definecolor{mydarkgreen}{rgb}{0,0.6,0}
\definecolor{mydarkblue}{rgb}{0,0,0}
\usepackage{url}
\usepackage{microtype}
\usepackage{graphicx}
\usepackage{subfigure}
\usepackage{booktabs}
\usepackage{amsthm}
\usepackage[colorlinks,linkcolor=black]{hyperref}
\usepackage{algorithm}
\usepackage{algorithmic}
\usepackage{soul}

\usepackage{hyperref}

\renewcommand{\H}{\mathcal{H}} 
\newcommand{\F}{\mathcal{F}} 
\newcommand{\Z}{\mathbb{Z}} 
\newcommand{\N}{\mathcal{N}} 
\newcommand{\xadv}{\tilde{\bm{x}}} 
\renewcommand{\P}{\mathbb{P}} 
\newcommand{\Q}{\mathbb{Q}} 
\DeclareMathOperator{\bigO}{\mathcal{O}}

\DeclareMathOperator{\MMD}{MMD}

\newcommand{\httpsurl}[1]{\href{https://#1}{\nolinkurl{#1}}}

\usepackage[accepted]{icml2021}

\icmltitlerunning{Maximum Mean Discrepancy Test is Aware of Adversarial Attacks}

\begin{document}

\twocolumn[
\icmltitle{Maximum Mean Discrepancy Test is Aware of Adversarial Attacks}



\icmlsetsymbol{equal}{*}

\begin{icmlauthorlist}
\icmlauthor{Ruize Gao}{equal,bu,cuhk}
\icmlauthor{Feng Liu}{equal,uts}
\icmlauthor{Jingfeng Zhang}{equal,riken}
\icmlauthor{Bo Han}{bu}
\icmlauthor{Tongliang Liu}{usyd}
\icmlauthor{Gang Niu}{riken}
\icmlauthor{Masashi Sugiyama}{riken,ut}
\end{icmlauthorlist}

\icmlaffiliation{bu}{Department of Computer Science, Hong Kong Baptist University}
\icmlaffiliation{cuhk}{Department of Computer Science and Engineering, The Chinese University of Hong Kong}
\icmlaffiliation{uts}{DeSI Lab, AAII, University of Technology Sydney}
\icmlaffiliation{riken}{RIKEN-AIP}
\icmlaffiliation{usyd}{TML Lab, University of Sydney}
\icmlaffiliation{ut}{Graduate School of Frontier Sciences, University of Tokyo}
\icmlcorrespondingauthor{Bo Han}{bhanml@comp.hkbu.edu.hk}
\icmlcorrespondingauthor{Tongliang Liu}{tongliang.liu@sydney.edu.au}
\icmlcorrespondingauthor{Gang Niu}{gang.niu@riken.jp}

\icmlkeywords{Machine Learning, ICML}

\vskip 0.3in
]



\printAffiliationsAndNotice{\icmlEqualContribution} 

\newtheorem{theorem}{Theorem}
\newtheorem{lemma}{Lemma} 
\newtheorem{corollary}{Corollary} 
\newtheorem{definition}{Definition} 
\newtheorem{problem}{Problem} 
\newtheorem{prop}{Proposition}  
\newtheorem{remark}{Remark}  

\begin{abstract}


The \emph{maximum mean discrepancy}~(MMD) test could in principle detect any distributional discrepancy between two datasets.
However, it has been shown that the MMD test is unaware of \emph{adversarial attacks}---the MMD test failed to detect the discrepancy between \emph{natural} and \emph{adversarial data}.
Given this phenomenon, we raise a question: are natural and adversarial data really from different distributions? 
The answer is affirmative---the previous use of the MMD test on the purpose missed three key factors, and accordingly, we propose three components.
Firstly, \emph{Gaussian kernel} has limited \emph{representation power}, and we replace it with an effective \emph{deep kernel}.
Secondly, \emph{test power} of the MMD test was neglected, and we maximize it following \emph{asymptotic statistics}.
Finally, adversarial data may be \emph{non-independent}, and we overcome this issue with the \emph{wild bootstrap}.
By taking care of the three factors, we verify that \emph{the MMD test is aware of adversarial attacks}, which lights up a novel road for adversarial data detection based on two-sample tests.

\end{abstract}

\section{Introduction}\label{Sec:intro}

The \emph{maximum mean discrepancy}~(MMD) aims to measure the closeness between two distributions $\P$ and $\Q$:
\begin{align}\label{eq:mmd_def_basic}
&\MMD(\P,\Q;\F):= \sup_{f\in\F} |{ \E[ f(X) ] - \E[ f(Y) ] }|,
\end{align}
where $\F$ is a set containing all continuous functions \cite{Gretton2012}. 
To obtain an analytic solution regarding the $\sup$ in Eq.~\eqref{eq:mmd_def_basic}, \citet{Gretton2012} restricted $\F$ to be a unit ball in the \emph{reproducing kernel Hilbert space} (RKHS) and obtain the kernel-based MMD defined in the following.
\begin{align}
\MMD(\P,\Q;\H_k):= &\sup_{f\in\H, {\|f\|}_{\H_k} \le 1} |{ \E[ f(X) ] - \E[ f(Y) ] }| \nonumber \\
= &{\|\mu_\P - \mu_\Q\|}_{\H_k}, \label{eq:mmd_kernel_form}
\end{align}
where $k$ is a bounded kernel regarding a RKHS $\mathcal{H}_{k}$ (i.e., $|k(\cdot,\cdot)|<+\infty$), and
$X\sim \P$, $Y\sim \Q$ are two random variables, and $\mu_\P := \E[ k(\cdot, X) ]$ and $\mu_\Q := \E[ k(\cdot, Y) ]$ are kernel mean embeddings of $\P$ and $\Q$, respectively \cite{gretton2005measuring,Gretton2012,Jitkrittum2016,Jitkrittum2017,sutherland:mmd-opt,liu2020learning}.
According to Eq.~\eqref{eq:mmd_def_basic}, it is clear that MMD equals zero \emph{if and only if} $\P=\Q$ \cite{gretton2008kernel}. As for the MMD defined in Eq.~\eqref{eq:mmd_kernel_form}, \citet{Gretton2012} also prove this property.
Namely, we could \emph{in principle} use the MMD to show whether two distributions are the same, which drives researchers to develop the MMD-based two-sample test \cite{Gretton2012}.  

In the MMD test, we are given two samples observed from $\P$ and $\Q$ and aim to check whether two samples come from the same distribution. Specifically, we first \emph{estimate} MMD value from two samples, and then compute the $p$-value corresponding to the estimated MMD value \cite{sutherland:mmd-opt}. If the $p$-value is above a given threshold $\alpha$, then two samples are from the same distribution. In the last decade, MMD test has been used to detect the distributional discrepancy within several real-world datasets,
including high-energy physics data \cite{Chwialkowski2015}, amplitude modulated
signals \cite{Gretton2012NeurIPS}, and challenging image datasets, e.g., the \emph{MNIST} and the \textit{CIFAR-10} \cite{sutherland:mmd-opt,liu2020learning}.

\begin{figure*}[tp]
    \begin{center}
        \subfigure[Test power]
        {\includegraphics[width=0.246\textwidth]{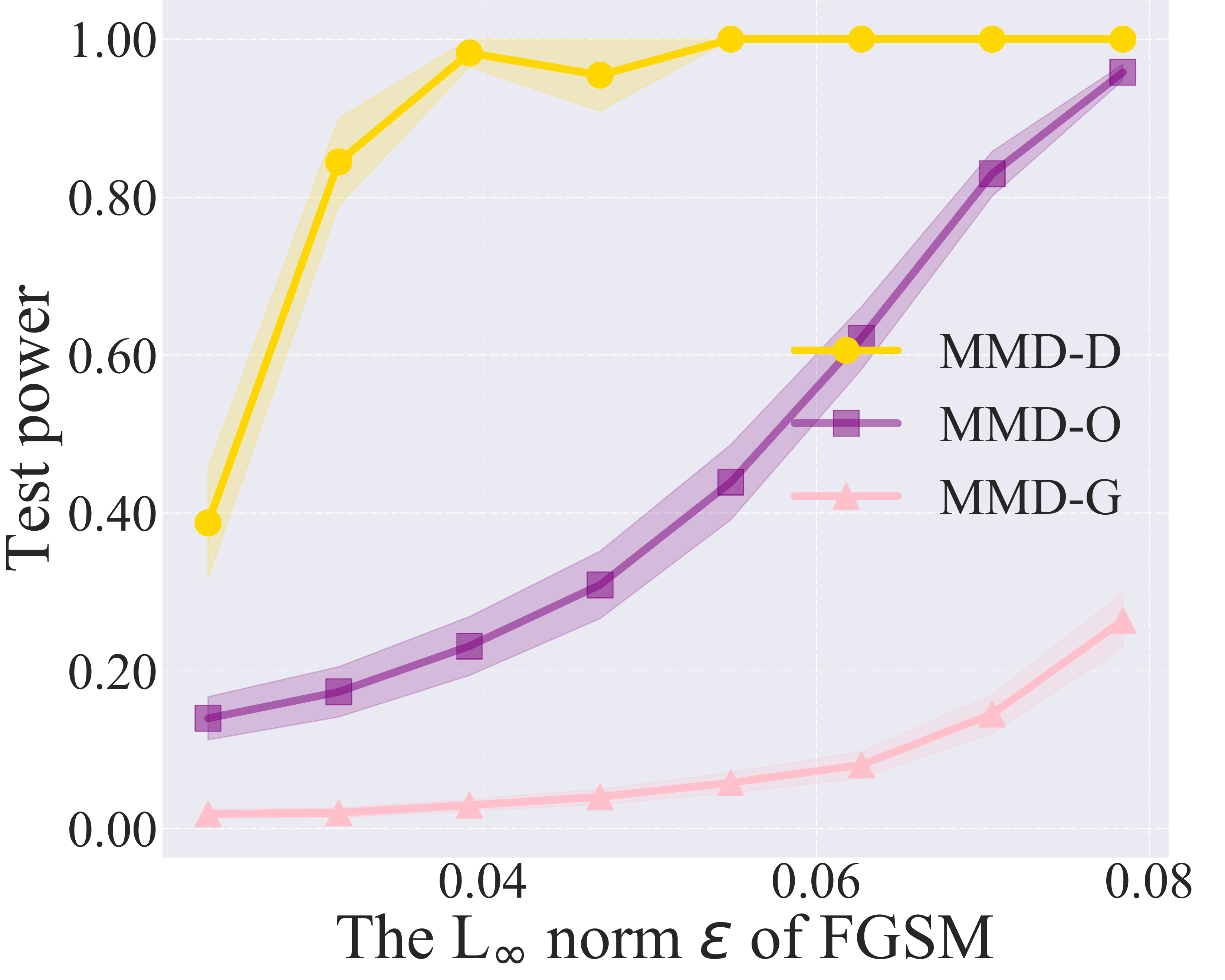}}
        \subfigure[Test power]
        {\includegraphics[width=0.246\textwidth]{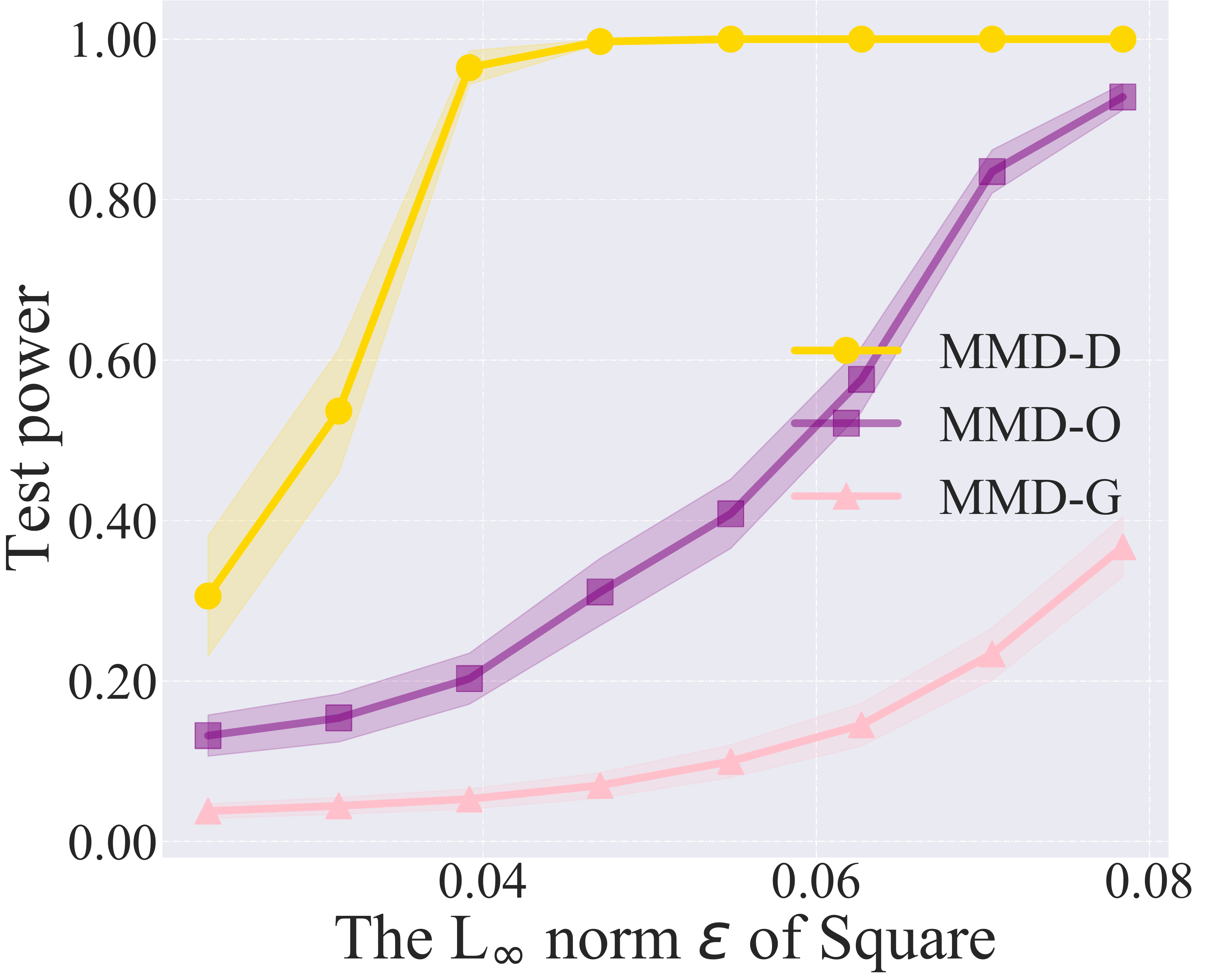}}
        \subfigure[Type I error]
        {\includegraphics[width=0.243\textwidth]{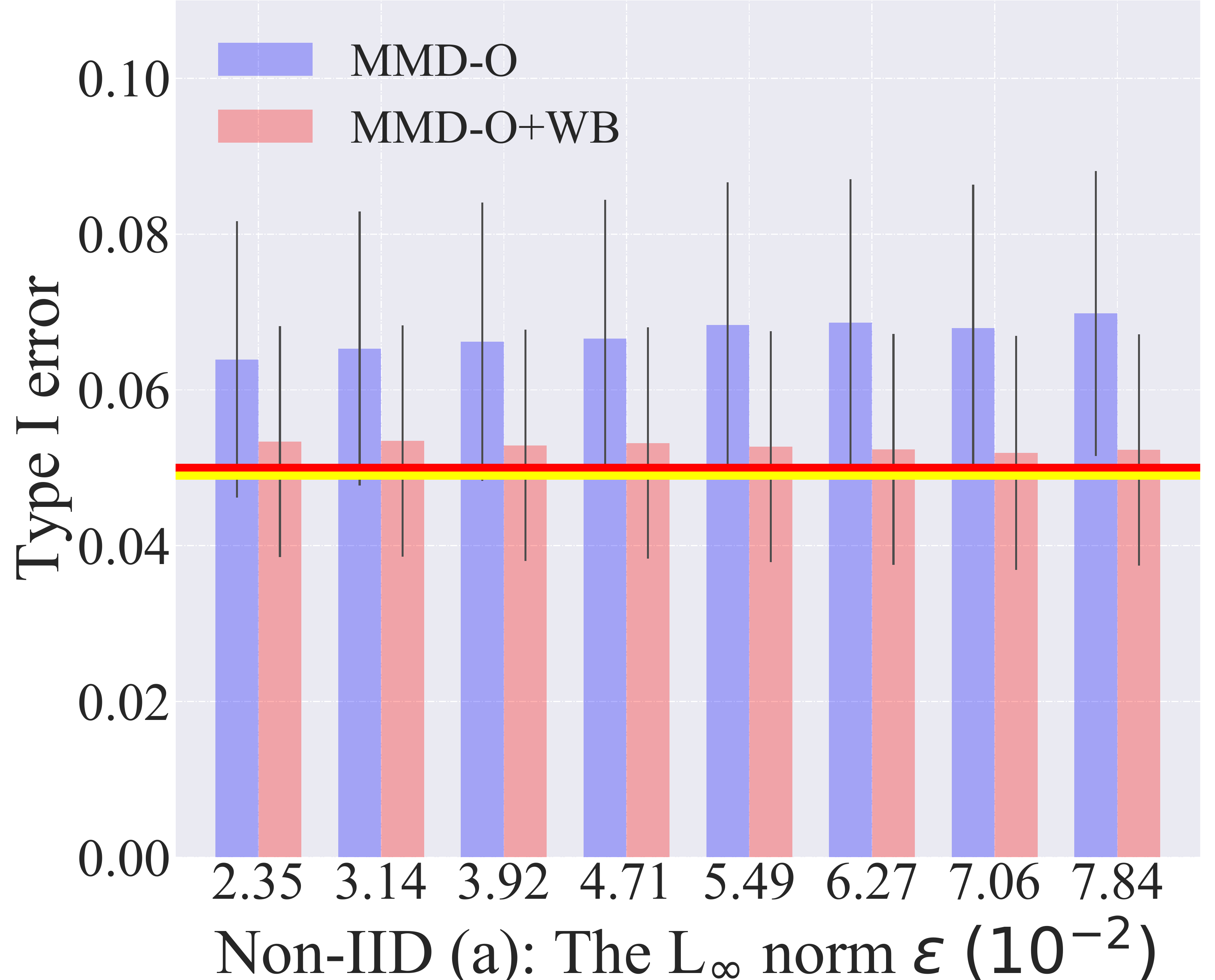}}
        \subfigure[Type I error]
        {\includegraphics[width=0.243\textwidth]{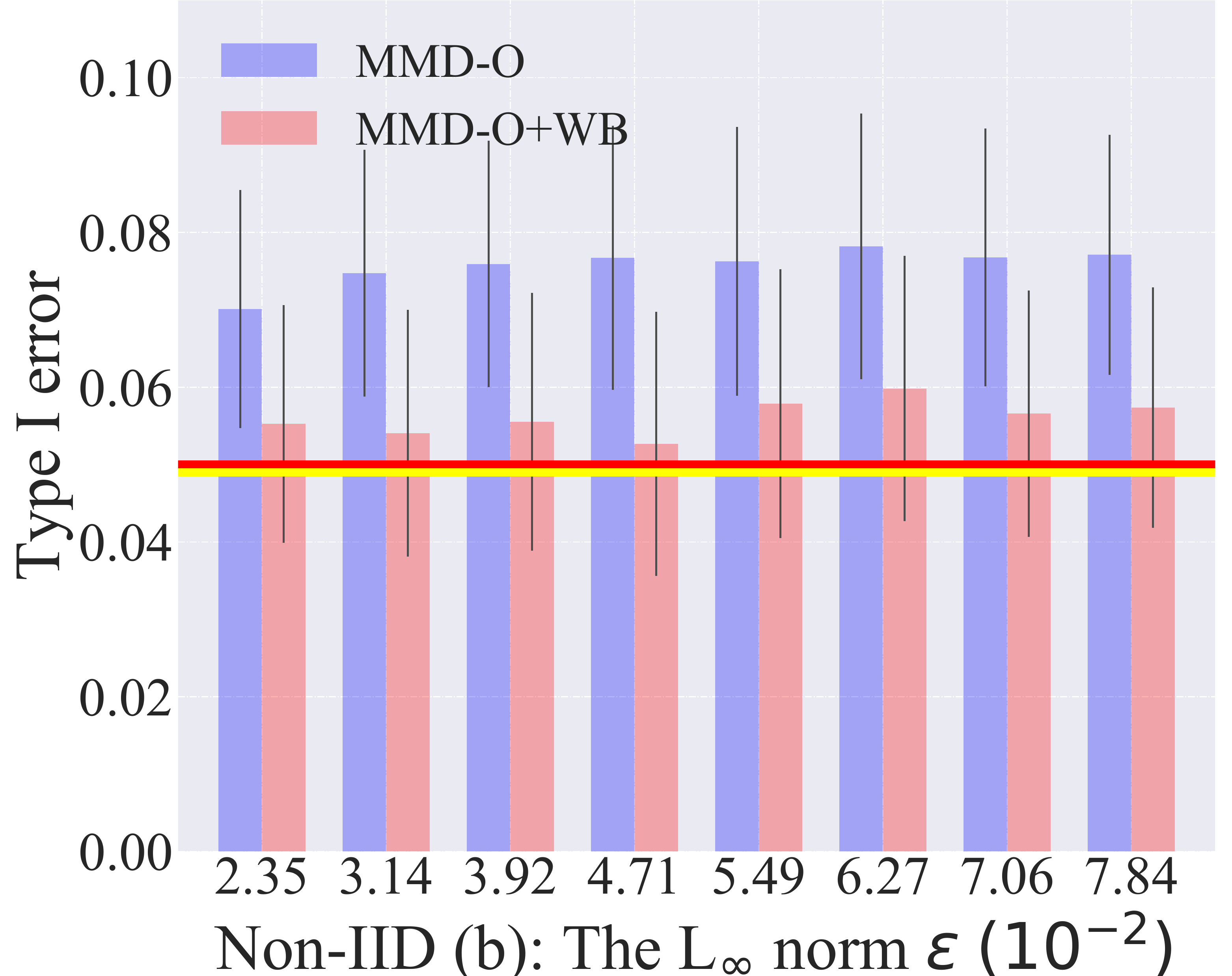}}
        \vspace{-1em}
        \caption{\footnotesize
       Consequences of missing the three key factors when using the MMD on adversarial data detection. The subfigure (a) and (b) illustrate the test power of \emph{the MMD test with deep kernel} (MMD-D test \citep{liu2020learning}), \emph{the MMD test with optimized Gaussian kernel} (MMD-O test \citep{sutherland:mmd-opt}) and \emph{the MMD test with Gaussian kernel} (MMD-G test), respectively. Adversarial data is generated by a white-box attack \emph{fast gradient sign method} (FGSM) \citep{goodfellow2014explaining} and a black-box attack \emph{Square attack} (Square) \citep{andriushchenko2020square} with different $L_\infty$-norm bounded perturbation $\epsilon \in [0.0235,0.0784]$ (following \citep{Madry18PGD,zhang2020attacks}). 
       Clearly, MMD-D and MMD-O tests perform much better than MMD-G test (previously used by \citep{grosse2017statistical} and \citep{carlini2017adversarial}). The failure of MMD-G test takes root in Factors $1$ and $2$ in Section~\ref{Sec:intro}.
       The (c) and (d) show type I error within two typical non-IID adversarial data (see detailed generation in Section~\ref{sec:dependence_within_data}), where type I error of MMD-O test is abnormal (higher than the red line that $\alpha=0.05$, while the type I error within natural data is the yellow line). The main reason is the Factor $3$ in Section~\ref{Sec:intro}. If we apply the \emph{wild bootstrap} (WB) process to MMD-O test, it brings type I error to normality (MMD-O+WB). 
      }
    \label{fig:moti}
    \end{center}
    \vspace{-1em}
\end{figure*}

However, it has been empirically shown that the MMD test, as one of the most powerful two-sample tests, is unaware of \emph{adversarial attacks} \cite{carlini2017adversarial}. Specifically, \citet{carlini2017adversarial} input adversarial and natural data into the MMD test, then the MMD test outputs a $p$-value that is greater than the given threshold $\alpha$ with a high probability. Namely, the MMD test agrees that adversarial and natural data are from the same distribution. Given the success of MMD test in many fields \cite{liu2020learning}, this phenomenon seems a \emph{paradox} regarding the homogeneity between nature and adversarial data.

In this paper, we raise a question regarding this paradox: \emph{are natural data and adversarial data really from different distributions?} The answer is affirmative, and we find the previous use of MMD missed \emph{three} factors. As a result, previous MMD-based adversarial data detection methods not only have a low detection rate when detecting attacks (due to the first two factors), but also are invalid detection methods (due to the third factor). 

\textbf{Factor 1.} The Gaussian kernel (used by previous MMD-based adversarial data detection methods) has \emph{limited} representation power and cannot measure the similarity between two multidimensional samples (e.g., images) well \cite{Kevin_ICML2019}. Although $\MMD(\P,\Q)$ is a perfect statistic to see if $\P$ equals $\Q$, test power (i.e., the detection rate when detecting adversarial attacks) of its empirical estimation (Eq.~\eqref{eq:MMD_U_compute}) depends on the form of used kernels \cite{sutherland:mmd-opt,liu2020learning}. Since a Gaussian kernel only looks at data uniformly rather than focuses on areas where two distributions are different, it requires many observations to distinguish the two distributions~\cite{liu2020learning}. As a result, the test power of \emph{the MMD test with a Gaussian kernel} (MMD-G test used by \citet{grosse2017statistical} and \citet{carlini2017adversarial}) is \emph{limited}, especially when facing complex data~\cite{,sutherland:mmd-opt,liu2020learning}. 

We replace the Gaussian kernel with a simple and effective semantic-aware deep kernel to take care of the first factor. We call this semantic-aware deep kernel based MMD as \emph{semantic-aware MMD} (SAMMD). The SAMMD is motivated by the recent advances in nonparametric two-sample tests, i.e., \emph{the MMD test with deep kernel} (MMD-D). In MMD-D, the kernel is parameterized by deep neural nets \cite{liu2020learning} and measures the distributional discrepancy between two sets of images using raw features (i.e., pixels in images).  Compared to the deep kernel used in MMD-D, semantic-aware deep kernel uses \emph{semantic features} extracted by a well-trained classifier on natural data.
Figure~\ref{fig:moti_SAMMD2} (see Section~\ref{sec:SAMMD}) shows that natural and adversarial data are quite different in the view of semantic features, showing that semantic features can help distinguish between natural and adversarial data, taking care of the first factor.



%

\textbf{Factor 2.} Previous MMD-based adversarial data detection methods overlook the optimization of parameters of the used kernel. In MMD-G test, its test power is related to the choice of the bandwidth of the Gaussian kernel \cite{sutherland:mmd-opt}. Once we overlook the optimization of the kernel bandwidth, the test power of MMD-G test will \emph{drop significantly}~\cite{Gretton2012NeurIPS,sutherland:mmd-opt}. Furthermore, recent studies have shown that Gaussian kernel with an optimized bandwidth still has limited representation power for complex distributions (e.g., multimodal distributions used in \cite{Kevin_ICML2019,liu2020learning}). Namely, it is important to take care of Factor $1$ and Factor $2$ simultaneously, which is verified in Figures~\ref{fig:moti}a-\ref{fig:moti}b.

To take care of the second factor, we analyze the asymptotics of the SAMMD when detecting adversarial attacks. According to the asymptotics of SAMMD, we can compute the approximate test power of SAMMD using two datasets and then optimize the parameters of the deep kernel by maximizing the approximate test power.

\textbf{Factor 3.} The adversarial data are probably not \emph{independent and identically distributed} (IID) due to their unknown generation process, which breaks a basic assumption of the MMD tests used by \cite{grosse2017statistical,carlini2017adversarial}.
Once there exists dependence within the observations, the type I error of ordinary MMD tests will surpass the given threshold $\alpha$. Note that, type I error is the probability of rejecting the null hypothesis ($\P=\Q$) when the null hypothesis is true. If the type I error of a test is much higher than $\alpha$, this test will always reject the null hypothesis. Namely, for two datasets that come from the same distribution, the test will always show that they are different, which means that the test is \emph{meaningless} \cite{Kacper14wildbtp}.

To take care of the third factor, the wild bootstrap is used to resample the value of SAMMD (with the optimized kernel), which ensures that we can get correct $p$-values in non-IID/IID scenarios (Figures~\ref{fig:moti}c-\ref{fig:moti}d). Here, we show two scenarios where the dependence within adversarial data exists: 1) the adversary attacks the data used for training the target model, where the target model depends on the attacked data. Thus, generated adversarial data are highly dependent (the Non-IID (a) in Figure~\ref{fig:moti}c); and 2) the adversary attacks one instance many times to generate many adversarial instances (the Non-IID (b) in Figure~\ref{fig:moti}d).

The above study is not of purely theoretical interest; it has also practical consequences. The considered detection problem is also known as \emph{statistical adversarial data detection} (SADD). In SADD, we care about how to find out a dataset that only contains natural data. That will bring benefits to users who are \emph{only} interested in a model that has high accuracy on the \emph{natural data}. For example, as an artificial-intelligence service provider, we need to acquire a client by modeling his/her task well, such as modeling the risk level of a factory. In this task, the client only cares about the accuracy on the natural data. Thus, we need to use MMD test to ensure that our training data only contain natural data. In Appendix~\ref{Asec:examples}, we have demonstrated SADD in detail.


\section{Preliminary}\label{Preliminary}
This section presents four concepts used in this paper.
\paragraph{Two-sample test.}
Let $\mathcal{X}\subset\R^d$
and $\P$, $\Q$ be Borel probability measures on $\mathcal{X}$.
Given IID samples $S_X=\{\bm{x}_i\}_{i=1}^n \sim \P^n$ and $S_Y=\{\bm{y}_j\}_{j=1}^m \sim \Q^m$, in the two-sample test problem,
we aim to determine if $S_X$ and $S_Y$ come from the same distribution, i.e.,
if $\P=\Q$.

\paragraph{Estimation of MMD.}

We can estimate MMD (Eq.~\eqref{eq:mmd_kernel_form}) using the $U$-statistic estimator,
which is unbiased for $\MMD^2$ and has nearly minimal variance among unbiased estimators \cite{Gretton2012}:
\begin{gather}
\widehat{\MMD}_u^2(S_X, S_Y; k)
= \frac{1}{n (n-1)} \sum_{i \ne j} H_{ij},
\label{eq:MMD_U_compute}
\\
H_{ij} =
    k(\vx_i, \vx_j)
    + k(\vy_i, \vy_j)
    - k(\vx_i, \vy_j)
    - k(\vy_i, \vx_j)
\notag
,\end{gather}
where $\vx_i,\vx_j \in S_X$ and $\vy_i,\vy_j \in S_Y$. 

\begin{table*}[!t]
\caption{Average values of dependence scores (HSIC) within natural data ($\epsilon=0$) and non-IID adversarial data (the $L_\infty$-norm bounded perturbation $\epsilon \in [0.0235,0.0784]$). The adversarial data of the Non-IID (a) are generated by FGSM on the training set of \emph{CIFAR-10}. The Non-IID (b) consists of the adversarial data generated by Square on CIFAR-10's testing set (for each natural image, Square generates four different adversarial images). We can see that the dependence within non-IID adversarial data is stronger than that within IID natural data.}
\label{tab:HSIC}
\vspace{1mm}
\label{tab:comparison}
\centering
\small
\begin{tabular}{p{9em}p{3em}p{3em}p{3em}p{3em}p{3em}p{3em}p{3em}p{3em}p{3em}}
\toprule
Perturbation bound $\epsilon$ & 0.0000 & 0.0235 & 0.0314 & 0.0392 & 0.0471 & 0.0549 & 0.0627 & 0.0706 & 0.0784\\
\midrule[0.6pt]
\midrule[0.6pt]
Non-IID (a) (10e-5) & 2.1948& 2.2214& 2.2409 & 2.2650 & 2.3067 & 2.3320 & 2.3727 & 2.4234 & 2.4805\\
Non-IID (b) (10e-5) & 2.1948 & 2.2146 & 2.2346& 2.2614 & 2.2952 & 2.3359 & 2.3835 & 2.4381& 2.4998\\
\bottomrule
\end{tabular}
\end{table*}

\paragraph{Adversarial data generation.}

\newcommand{\T}{{\hspace{-0.25ex}\top\hspace{-0.25ex}}}
\newcommand{\ST}{\mathrm{s.t.}}
\newcommand{\bE}{\mathbb{E}}
\newcommand{\bR}{\mathbb{R}}
\newcommand{\cB}{\mathcal{B}}
\newcommand{\cF}{\mathcal{F}}
\newcommand{\cX}{\mathcal{X}}
\newcommand{\cY}{\mathcal{Y}}
\newcommand{\bx}{{x}}
\newcommand{\bxprime}{{x}'}
\newcommand{\bxtidle}{\tilde{{x}}}
\newcommand{\epsball}{\mathcal{B}_\epsilon}
\newcommand{\yadv}{\tilde{y}}
Let $(\cX,d_\infty)$ be the input feature space $\cX$ with the infinity distance metric $d_{\inf}(\bm{x},\bm{x}^{\prime})=\|\bm{x}-\bm{x}^{\prime}\|_\infty$, and
\begin{align}
\label{eqn:perturbation_ball}
\epsball[\bm{x}] = \{\bm{x}^{\prime} \in \cX \mid d_{\inf}(\bm{x},\bm{x}')\le\epsilon\}
\end{align}
be the closed ball of radius $\epsilon>0$ centered at $\bm{x}$ in $\cX$.
Let $D = \{ (\bm{x}_i, l_i\}^n_{i=1}$ be a dataset, where $\bm{x}_i \in \cX$, $l_i\in\mathcal{C}$ is ground-truth label of $\bm{x}_i$, and $\mathcal{C}=\{1,\dots,C\}$ is a label set. 
Then, adversarial data regarding $\bm{x}_i$ is
\begin{align}
\label{Eq:madry_inner_maximization}
\mathcal{G}_{\ell,\hat{f}}(\bm{x}_i) = \argmax\nolimits_{\xadv\in\epsball[\bm{x}_i]} \ell(\hat{f}(\xadv),l_i),
\end{align}
where $\xadv$ is a sample within the $\epsilon$-ball centered at $\bm{x}$, $\hat{f}(\cdot):\cX\to\mathcal{C}$ is a well-trained classifier on $D$, and $\ell:\mathcal{C}\times\mathcal{C}\to\bR_{\ge0}$ is a loss function. 





There are many methods to solve Eq.~(\ref{Eq:madry_inner_maximization}) and generate adversarial data, e.g., white-box attacks including \emph{fast gradient sign method} (FGSM) \cite{goodfellow2014explaining}, \emph{basic iterative methods} (BIM) \cite{kurakin2016adversarial}, \emph{project gradient descent} (PGD) \cite{Madry18PGD}, \emph{AutoAttack} (AA) \cite{croce2020reliable}, \emph{Carlini and Wagner attack} (CW) \cite{carlini2017towards} and a score-based black-box attack: \emph{Square attack} (Square \citet{andriushchenko2020square}).

\paragraph{Wild bootstrap process.}

The wild bootstrap process has been proposed to resample observations from a stochastic process $\{{Y}_i\}_{i\in\Z}$ \cite{shao2010dependentwildbtp}, where $\E(Y_i)=0$ for each $i\in\Z$. Through multiplying the given observations with random numbers from the wild bootstrap process, we can obtain new samples that can be regarded as resampled observations from $\{{Y}_i\}_{i\in\Z}$ \cite{leucht2013dependentwildbtp,Kacper14wildbtp}. After resampling observations many times, we can use such resampled observations to estimate the distribution of statistics regarding the random process $\{{Y}_i\}_{i\in\Z}$, such as the null distribution of MMD over two time series \cite{Kacper14wildbtp}. Following \cite{Kacper14wildbtp} and \cite{leucht2013dependentwildbtp}, this paper uses the following wild bootstrap process:
\begin{align}
\label{eq:wb_generator}
    W_t = e^{-1/l}W_{t-1} + \sqrt{1-e^{-2/l}}\epsilon_t,
\end{align}
where $W_0, \epsilon_0,\dots,\epsilon_t$ are independent standard normal random variables. 

\section{Problem Setting}
Following \citet{grosse2017statistical}, we aim to address the following problem (i.e., SADD mentioned in Section~\ref{Sec:intro}).
\begin{problem}[SADD]\label{problem1}
Let $\mathcal{X}$ be a subset of $\R^d$
and $\P$ be a Borel probability measure on $\mathcal{X}$, and $S_X=\{\bm{x}_i\}_{i=1}^n \sim \P^n$ be IID observations from $\P$, and $f(\cdot): \R^d\rightarrow \mathcal{C}$ be the ground-truth labeling function on observations from $\P$, where $\mathcal{C}=\{1,\dots,C\}$ is a label set. Assume that attackers can obtain a well-trained classifier $\hat{f}$ on $S_X$ and IID observations $S_X^\prime$ from $\P$, 
we aim to determine if the upcoming data $S_Y=\{\bm{y}_i\}_{i=1}^m$ come from the distribution $\P$, where $S_X$ and $S_X^\prime$ are independent, and we do not have any prior knowledge regarding the attacking methods. Note that, in SADD, $S_Y$ may be IID data from $\P$ or non-IID data generated by attackers. 
\end{problem}

In Problem~\ref{problem1}, if $S_Y$ are IID observations from $\P$, given a threshold $\alpha$, we aim to accept the null hypothesis $H_0$ (i.e., $S_X$ and $S_Y$ are from the same distribution) with the probability $1-\alpha$. If $S_Y$ contains adversarial data (i.e., $S_X$ and $S_Y$ are from different distributions), we aim to reject the null hypothesis $H_0$ with a probability near to $1$. Please note that, an invalid test method could be ``rejecting all upcoming data'', which can perform very well when $S_X$ and $S_Y$ being from different distributions but fail when $S_Y$ being from $\P$.

\section{Failure of Gaussian-kernel MMD Test for Adversarial Data Detection}
\label{sec:failure_MMDG}

We reimplement the experiment in \cite{carlini2017adversarial} and \cite{grosse2017statistical} to test the performance of MMD-G test on \textit{CIFAR-10} dataset. Adversarial data with different perturbation bound $\epsilon$ are generated by FGSM, BIM, PGD, AA, CW and Square. Figure~\ref{fig:results} shows how test power changes as the $\epsilon$ value increases in each attacking method. Through our implementations, we draw the same conclusion with \citet{carlini2017adversarial}. Namely, MMD-G test (the pink line) fails to detect adversarial data. 

As demonstrated in Section~\ref{Sec:intro}, MMD-G test has the following issues: 1) the limited representation power of the Gaussian kernel \cite{Kevin_ICML2019,liu2020learning}; and 2) the overlook of optimization of the kernel bandwidth \cite{liu2020learning}; and 3) the non-IID property of adversarial data \cite{shao2010dependentwildbtp,leucht2013dependentwildbtp,Kacper14wildbtp}. Since the third issue is crucial, we first analyze whether there exists dependence within adversarial data in the following section and then propose a novel test to address the above issues simultaneously (see Section~\ref{sec:SAMMD}).

\begin{figure*}[tp]
\centering
    \begin{center}
        \subfigure[RN18-Natural]
        {\includegraphics[width=0.134\textwidth]{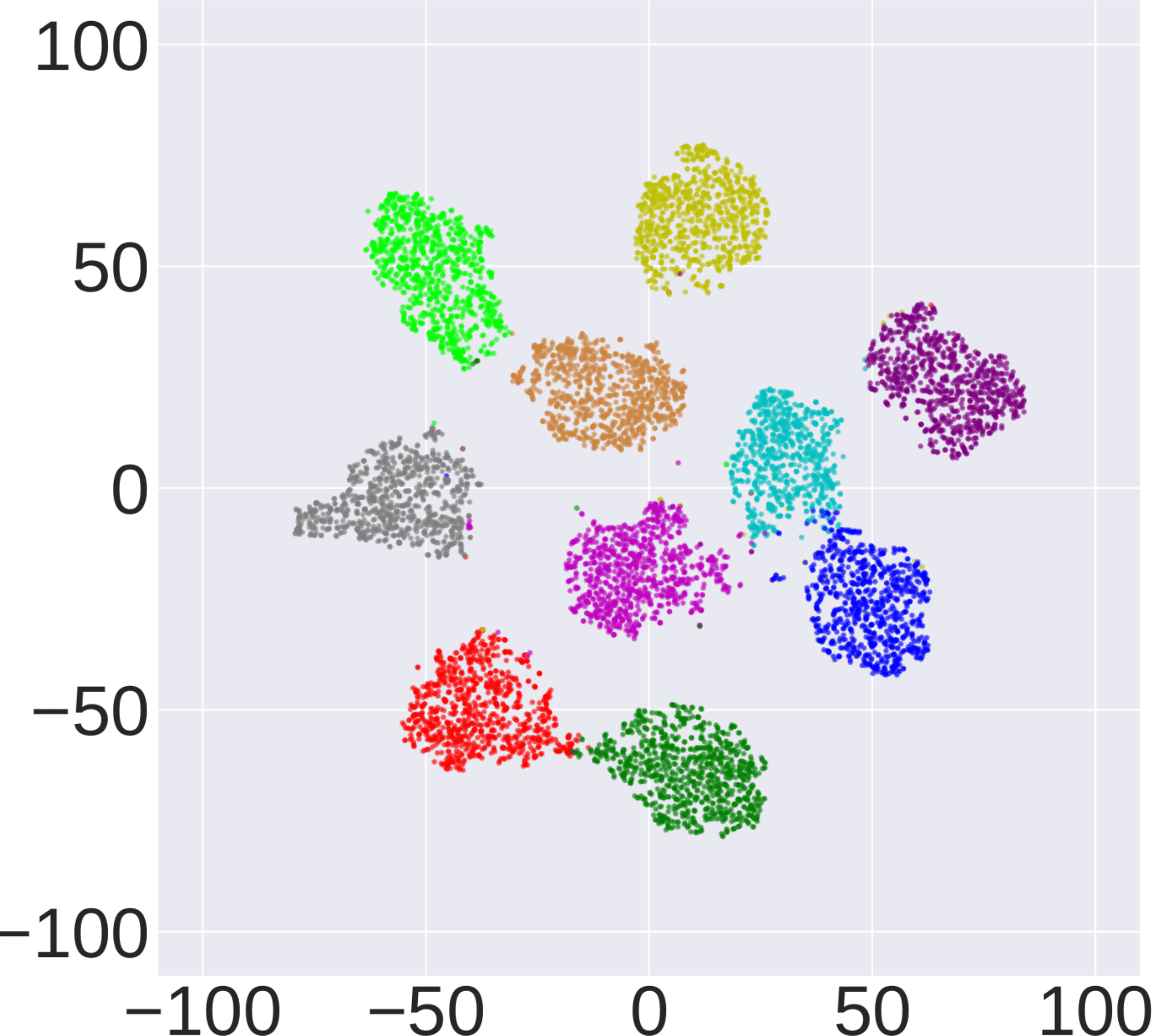}}
        \subfigure[RN18-FGSM]
        {\includegraphics[width=0.134\textwidth]{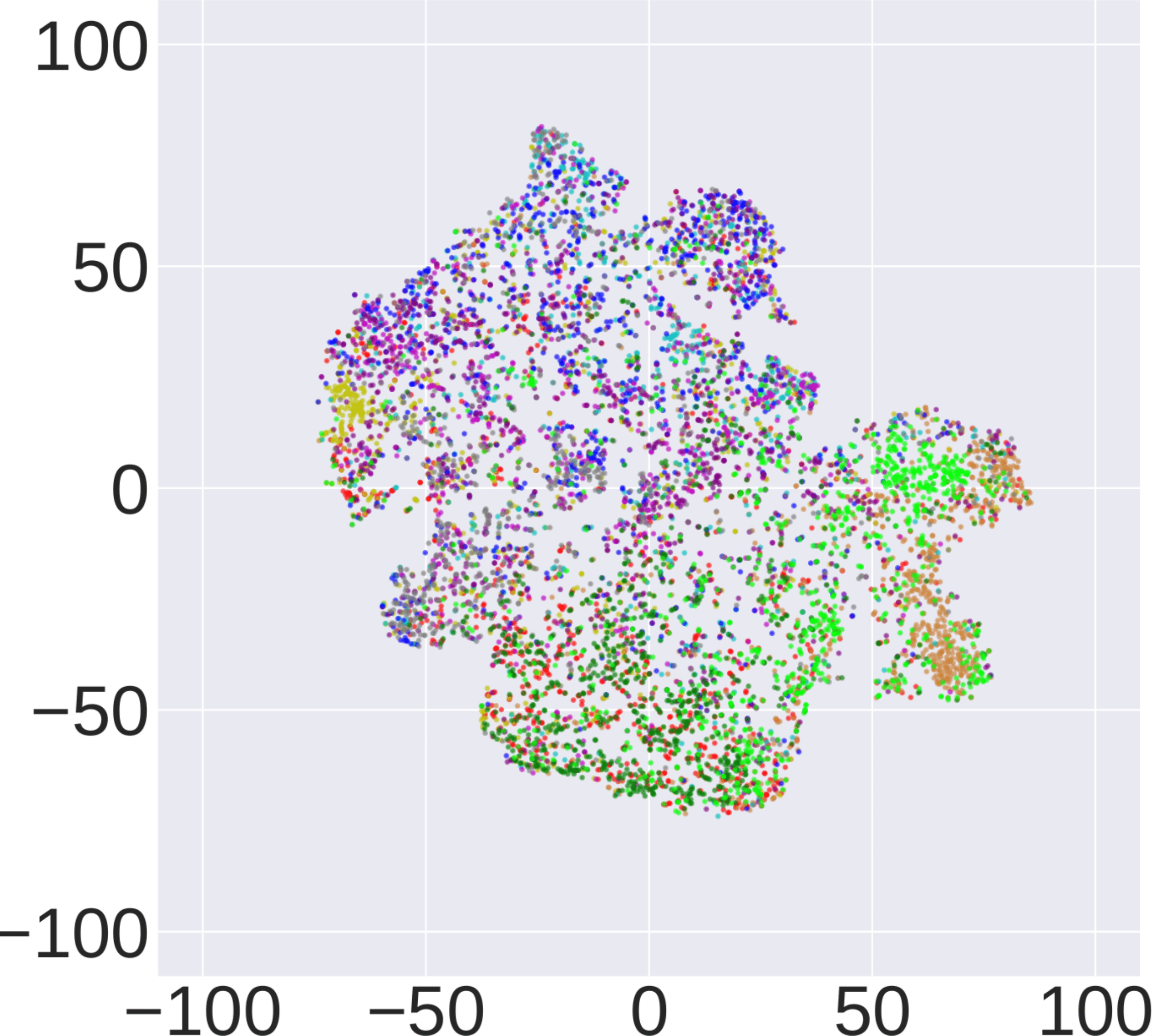}}
        \subfigure[RN18-BIM]
        {\includegraphics[width=0.134\textwidth]{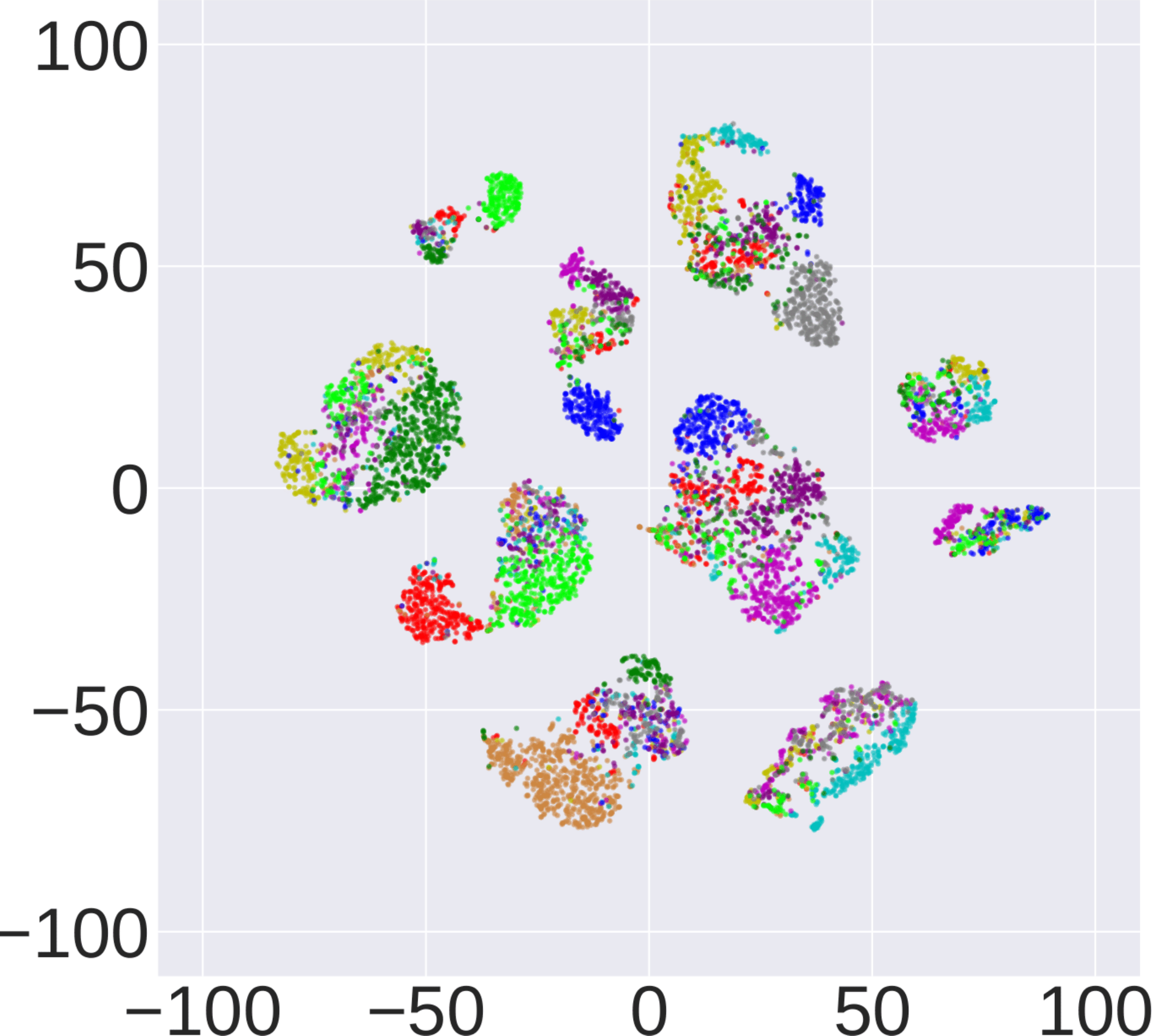}}
        \subfigure[RN18-PGD]
        {\includegraphics[width=0.134\textwidth]{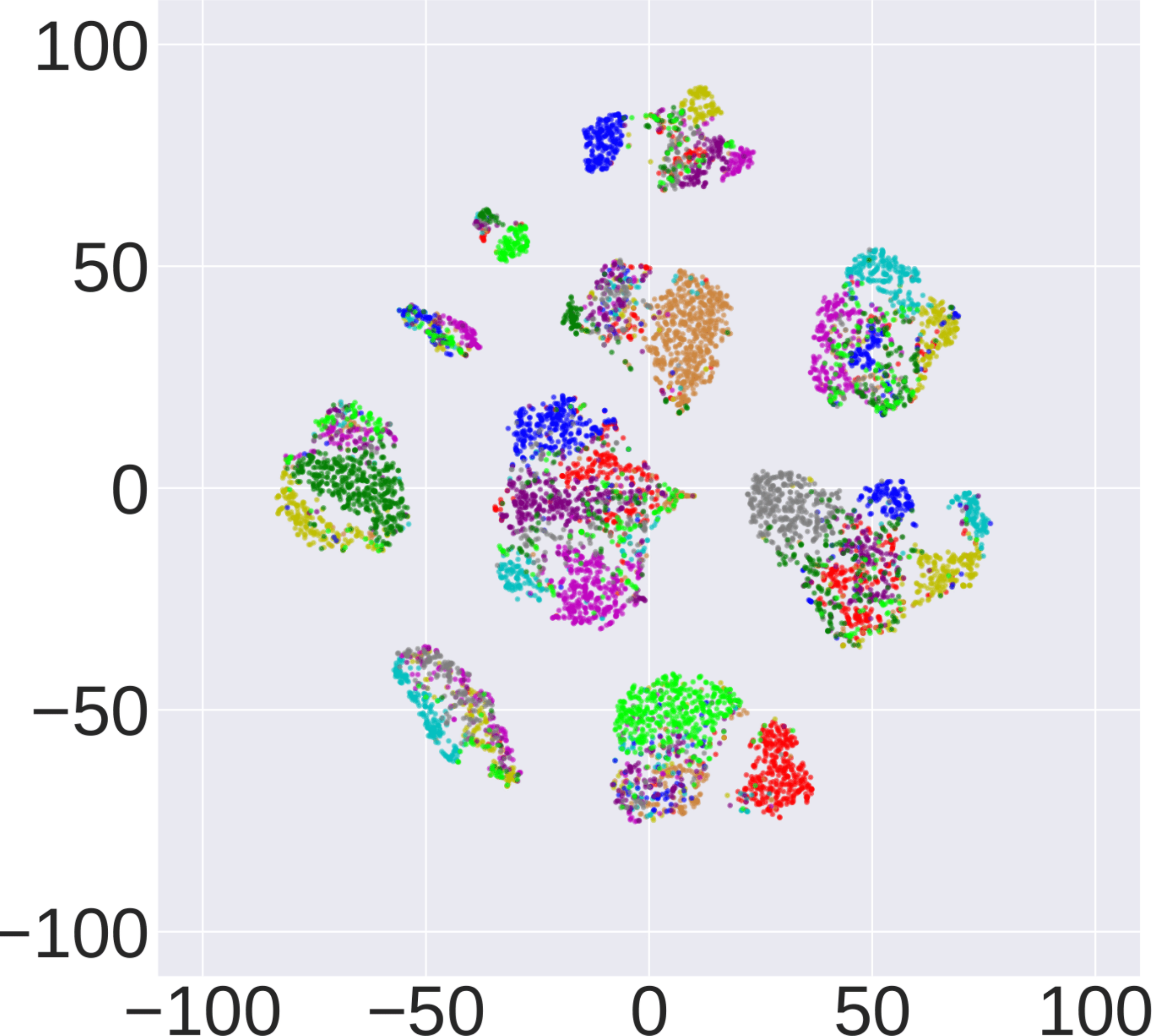}}
        \subfigure[RN18-CW]
        {\includegraphics[width=0.134\textwidth]{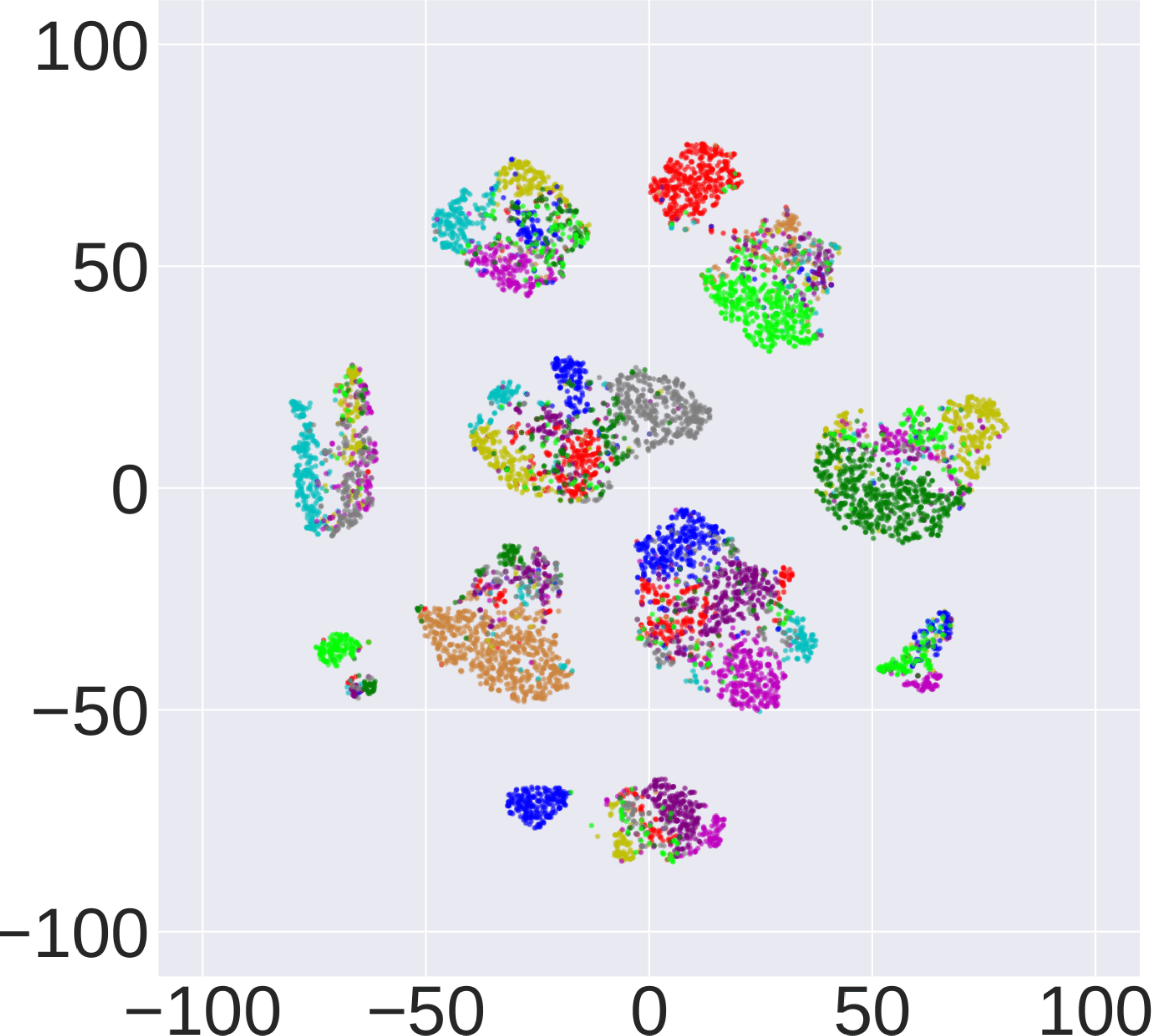}}
        \subfigure[RN18-AA]
        {\includegraphics[width=0.134\textwidth]{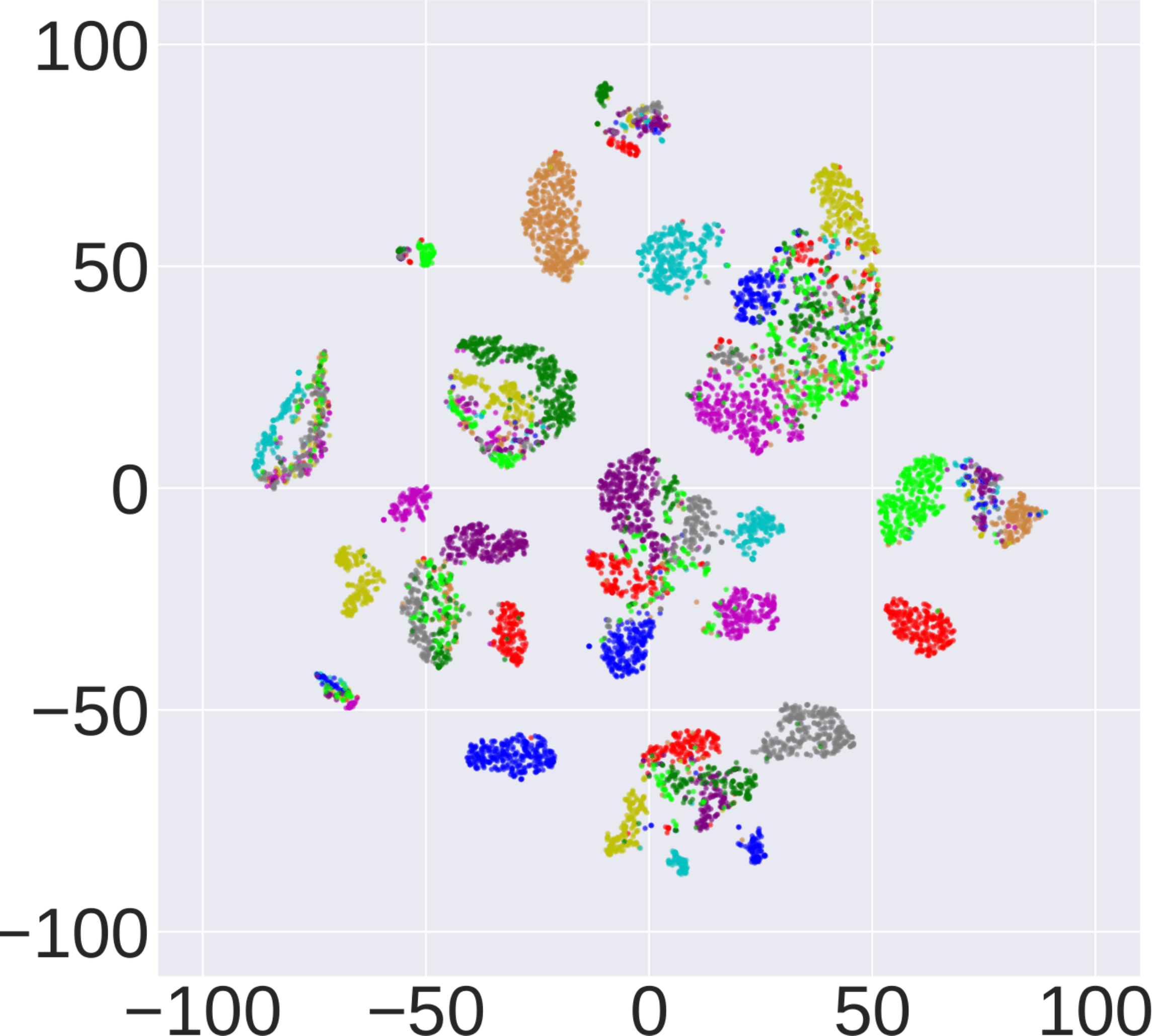}}
        \subfigure[RN18-Square]
        {\includegraphics[width=0.134\textwidth]{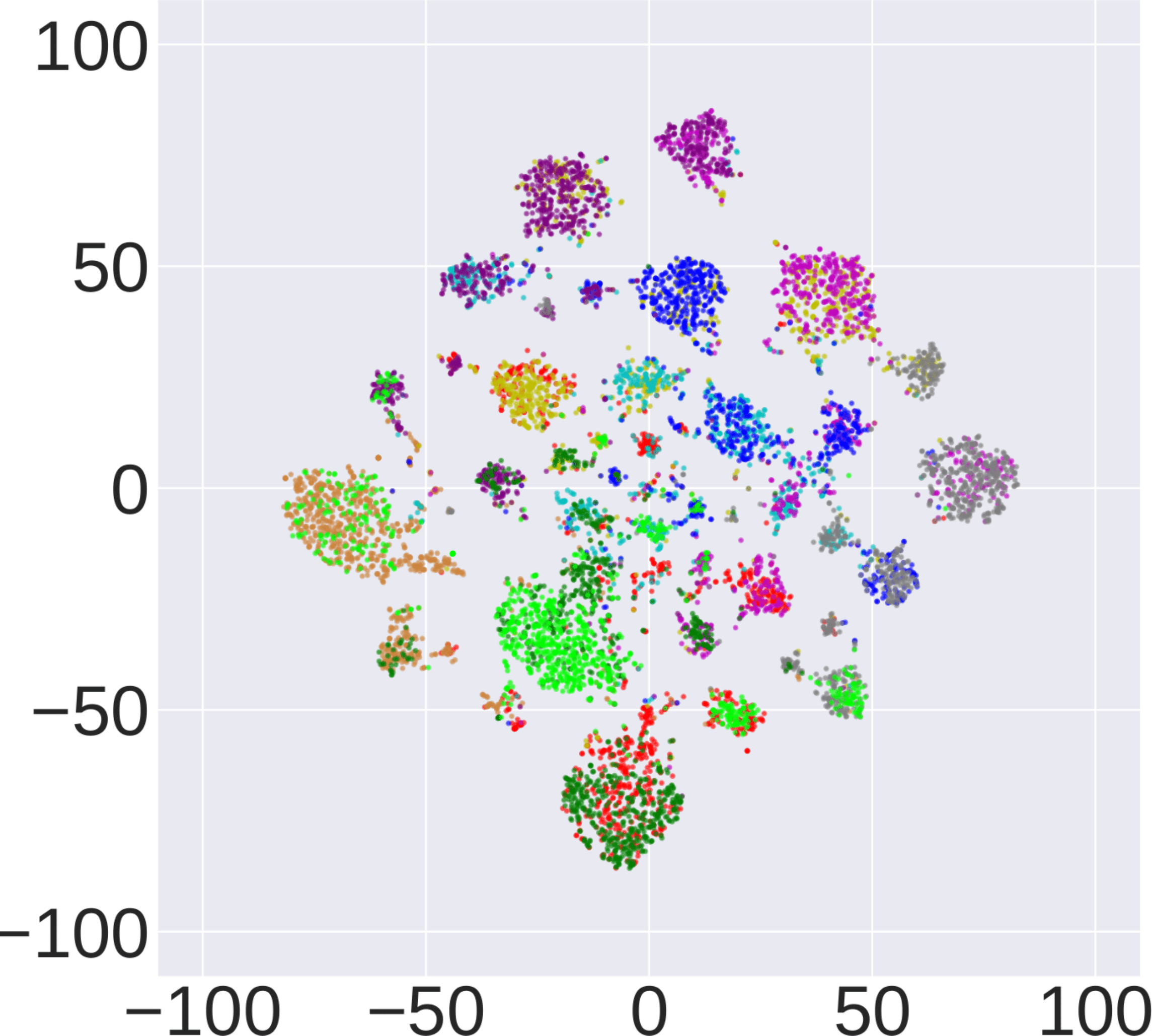}}
        \subfigure[RN34-Natural]
        {\includegraphics[width=0.134\textwidth]{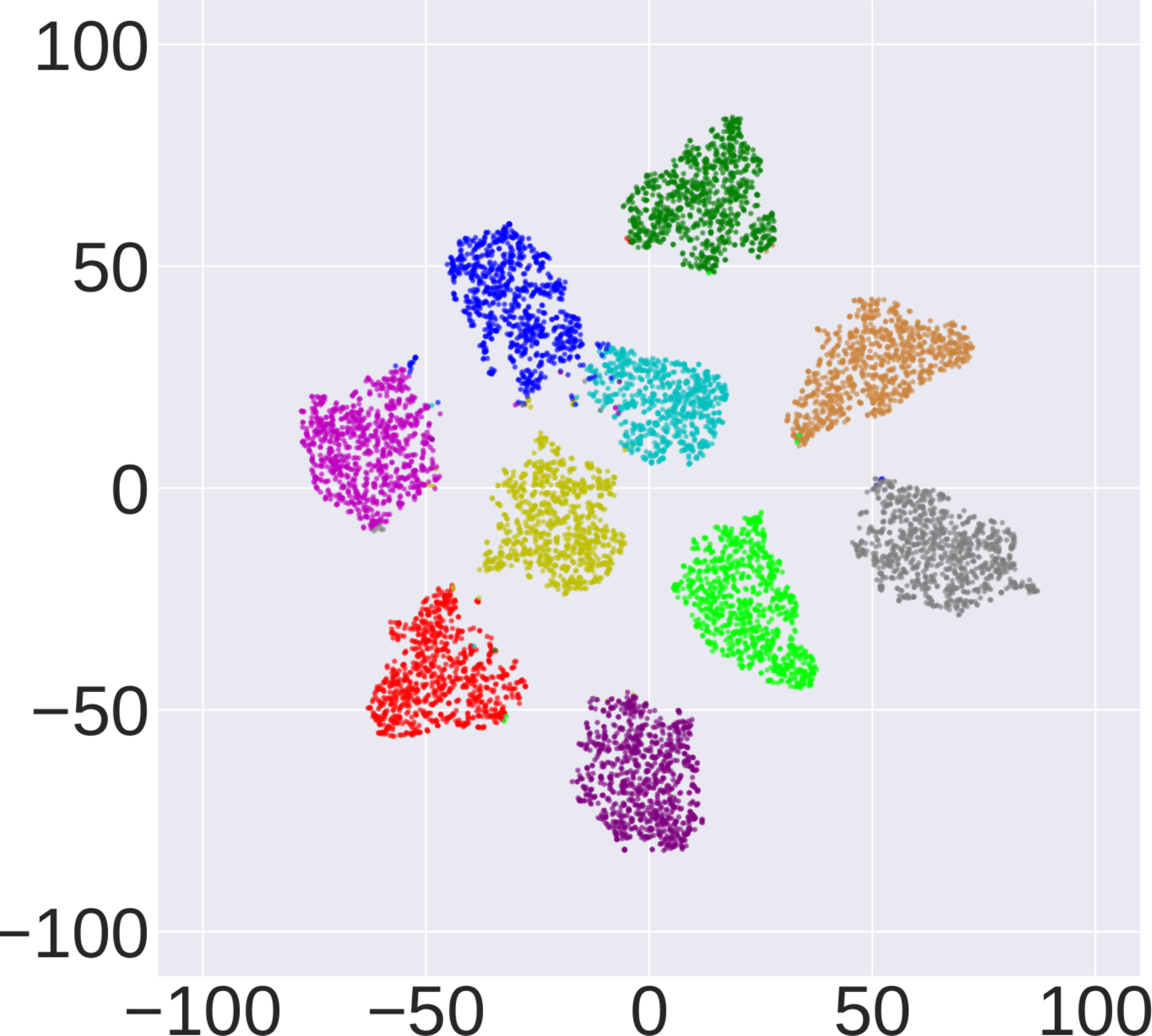}}
        \subfigure[RN34-FGSM]
        {\includegraphics[width=0.134\textwidth]{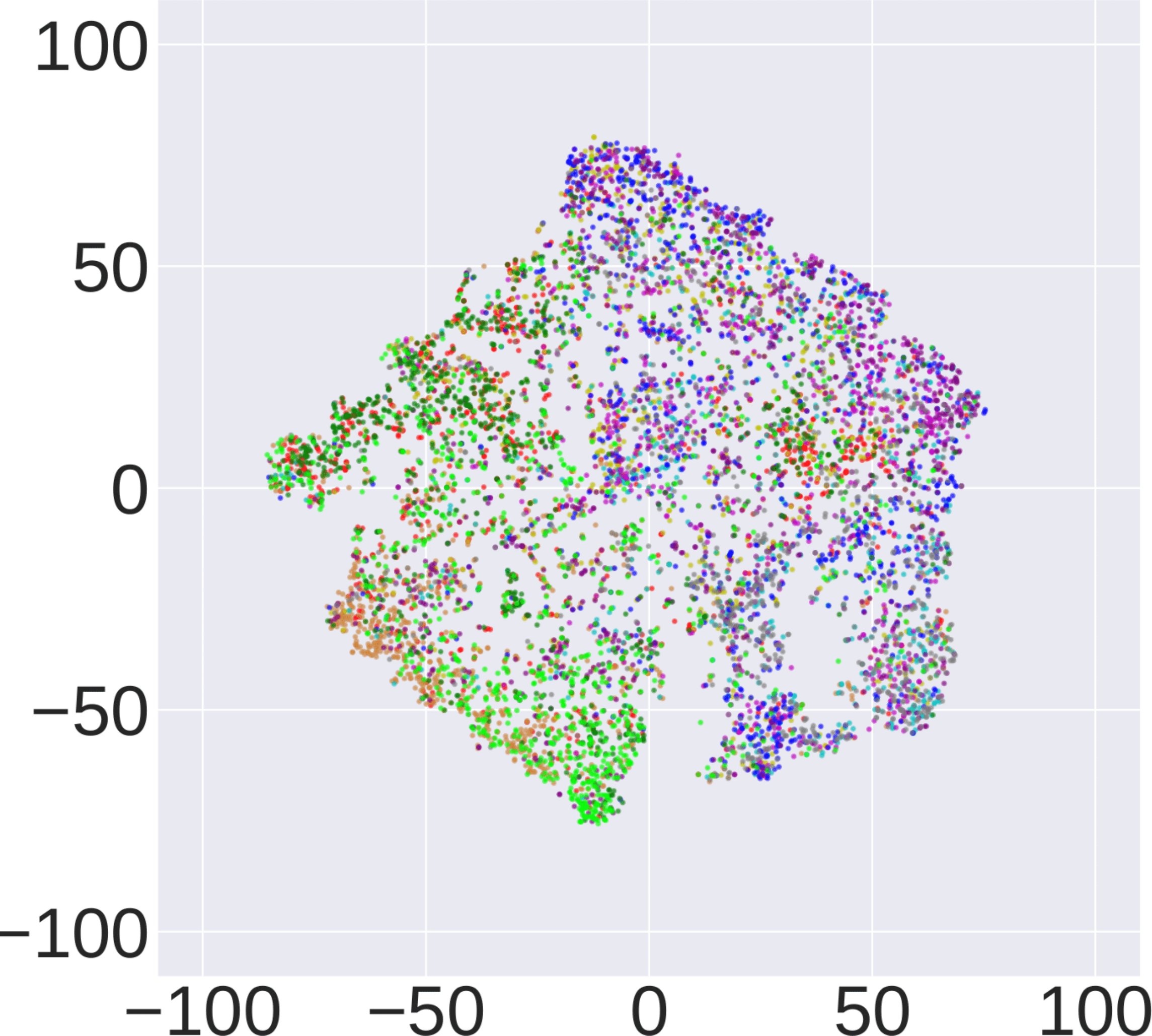}}
        \subfigure[RN34-BIM]
        {\includegraphics[width=0.134\textwidth]{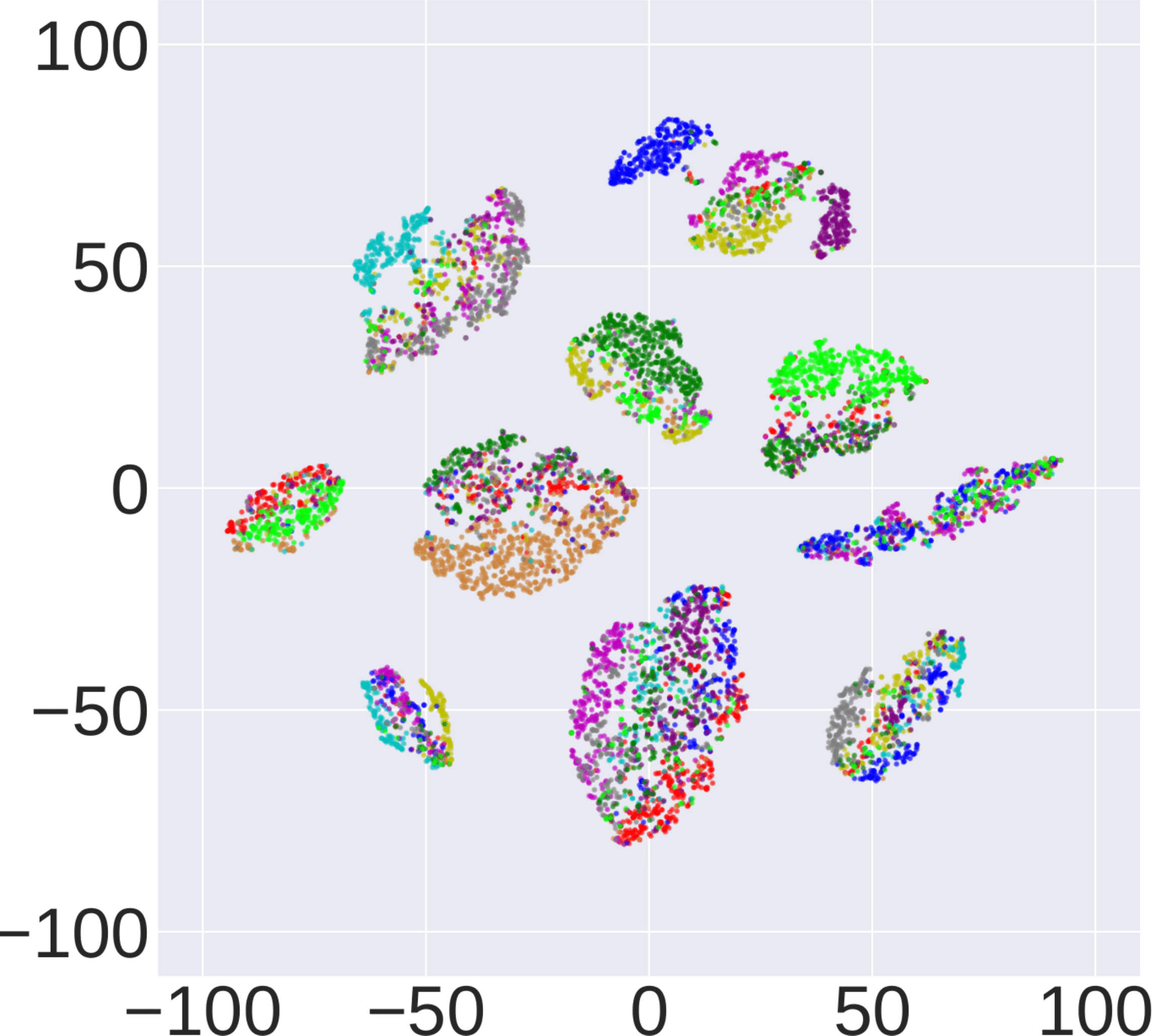}}
        \subfigure[RN34-PGD]
        {\includegraphics[width=0.134\textwidth]{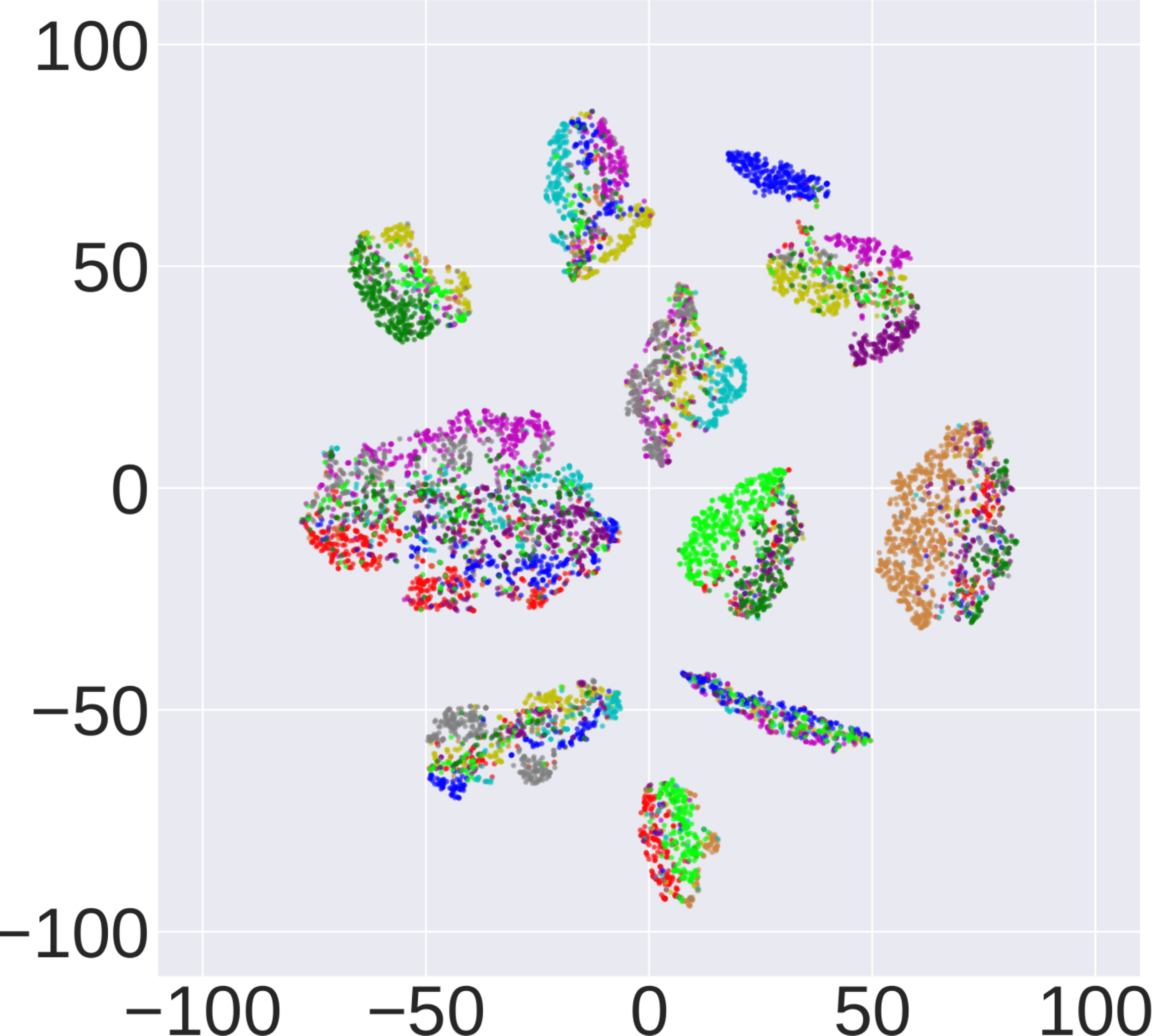}}
        \subfigure[RN34-CW]
        {\includegraphics[width=0.134\textwidth]{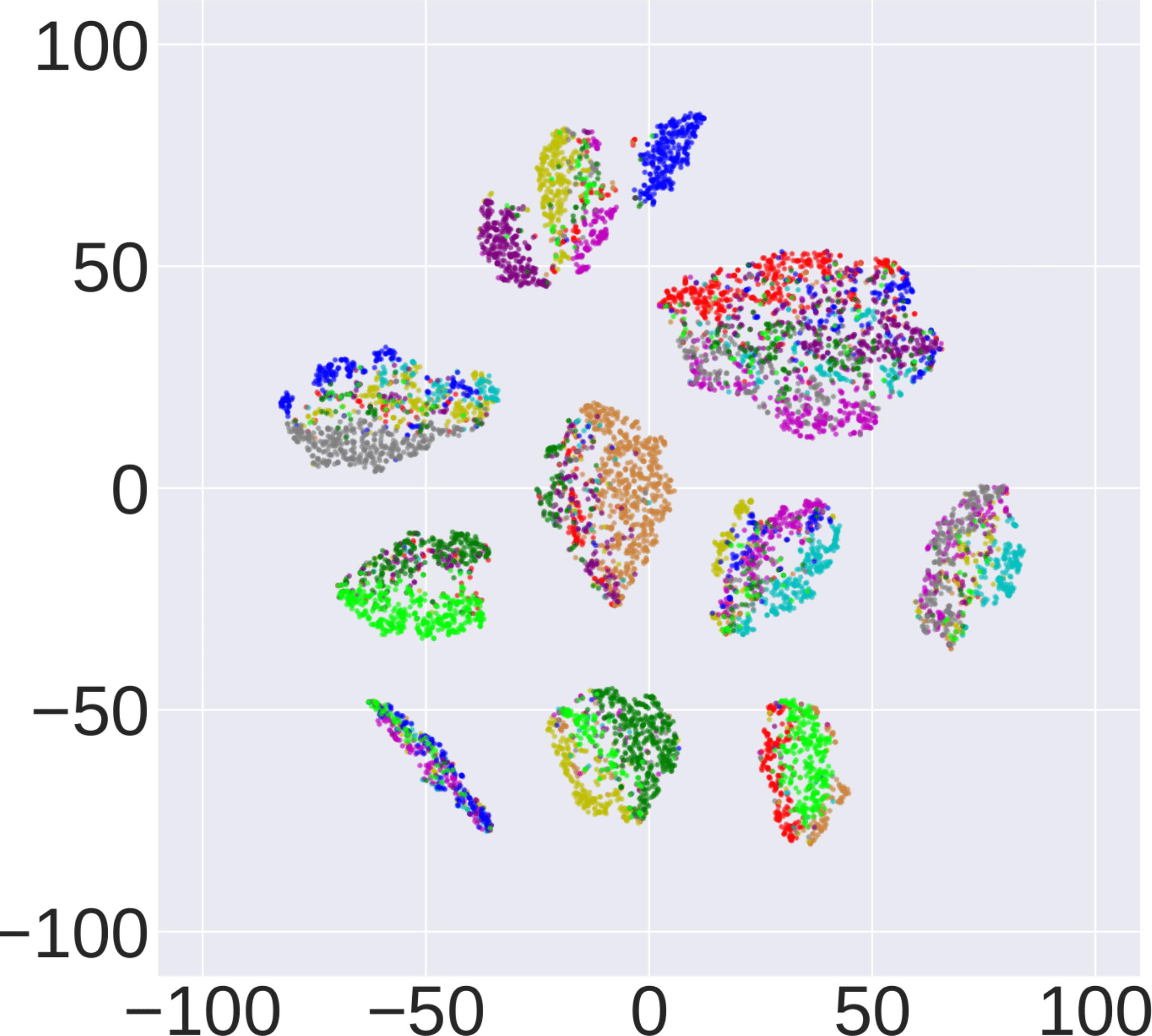}}
        \subfigure[RN34-AA]
        {\includegraphics[width=0.134\textwidth]{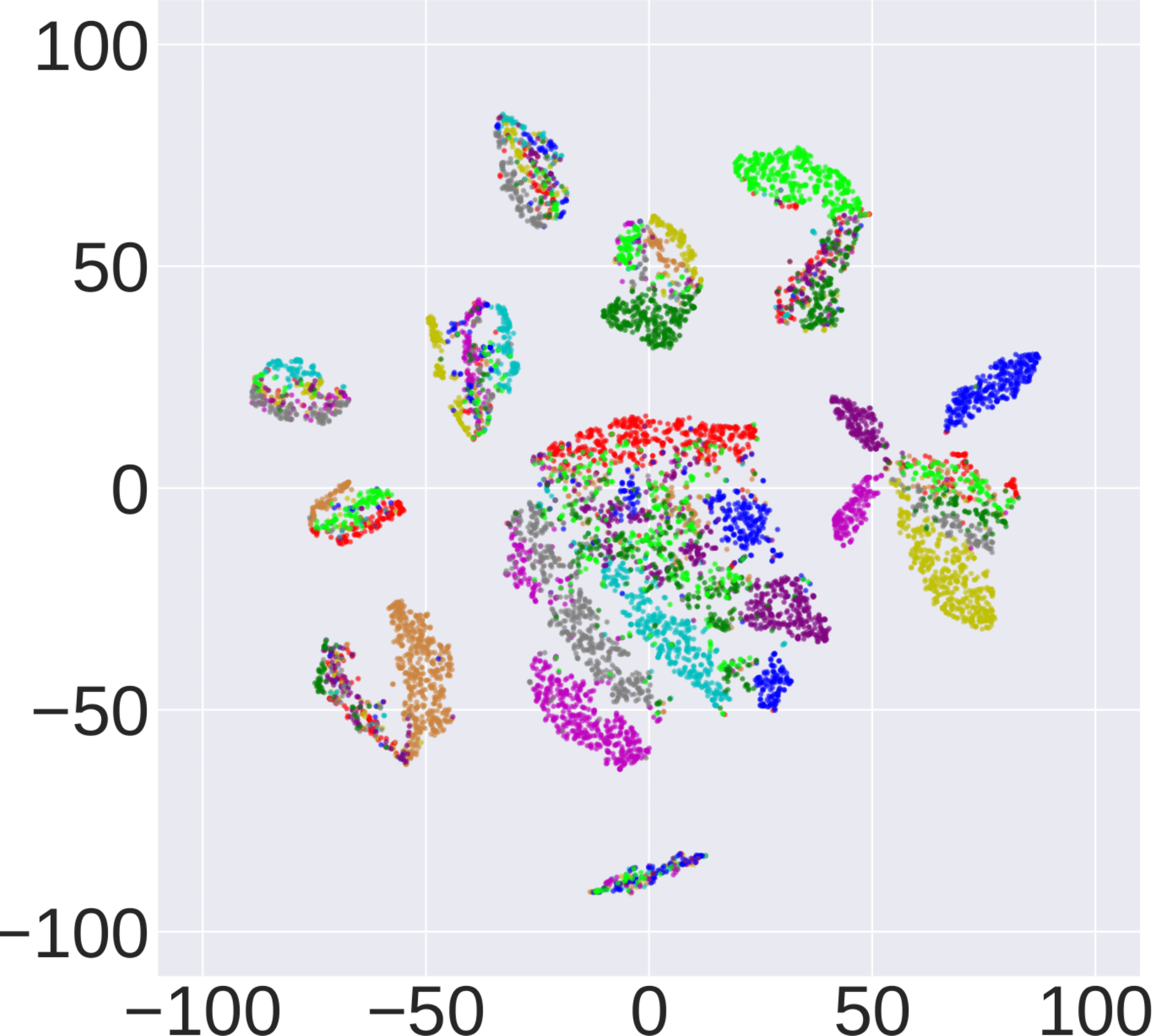}}
        \subfigure[RN34-Square]
        {\includegraphics[width=0.134\textwidth]{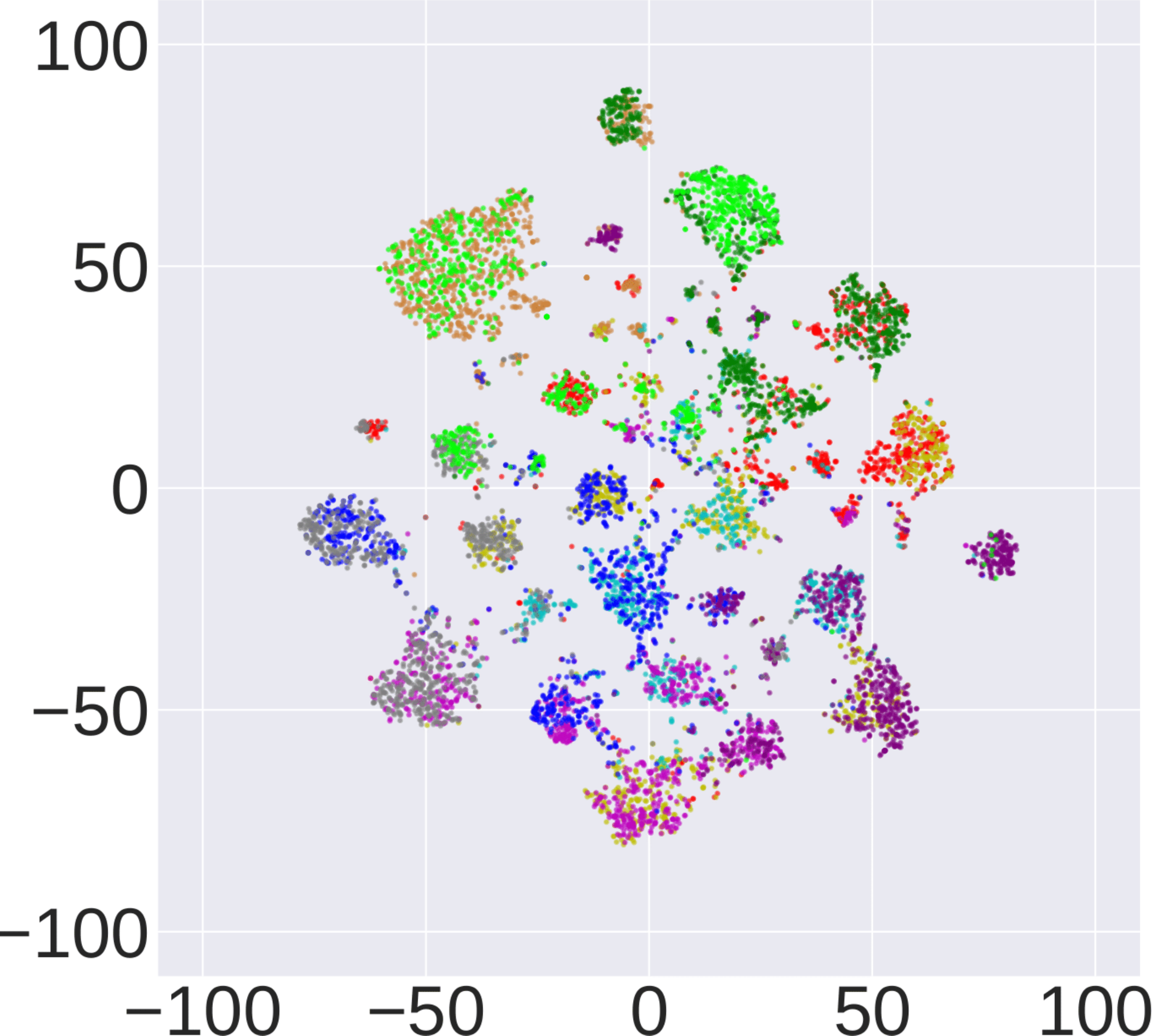}}

        \caption{\footnotesize
       Visualization of outputs using t-SNE. This figure visualizes outputs of the second to last layers in ResNet-18 and ResNet-34. Different colors represent different semantic meanings (i.e., different classes in the testing set of the \textit{CIFAR-10}). Apparently, semantic information contained in natural data is lost in adversarial data. This phenomenon can help us distinguish adversarial data and natural data.}
        \label{fig:moti_SAMMD2}
    \end{center}
\end{figure*}


\section{Dependence in Adversarial Data}
\label{sec:dependence_within_data}

As discussed in Section~\ref{sec:failure_MMDG}, this section investigates whether there is dependence within adversarial data. 

\paragraph{Dependence within adversarial data.}
In the real world, since we do not know the attacking strategies of attackers, the dependence within adversarial data probably exists. For example, if attackers use one natural image to generate many adversarial images, the adversarial data is obviously not independent (the Non-IID (b) in Table~\ref{tab:HSIC}). 
To empirically show the dependence within adversarial data, we use HSIC \citep{gretton2005measuring} as the statistic to represent the dependence score within adversarial data (Appendix~\ref{Asec: HSIC} presents detailed procedures to compute the HSIC values between two datasets). The larger value of HSIC represents the stronger dependence.

We generated two typical non-IID adversarial datasets that the Non-IID (a) and the Non-IID (b). Given natural images from the \textit{CIFAR-10} training set, we generated the Non-IID (a) using FGSM with the $L_\infty$-norm bounded perturbation $\epsilon \in [0.0235,0.0784]$. Given natural images from the \textit{CIFAR-10} testing set, we used Square with the $L_\infty$-norm bounded perturbation $\epsilon \in [0.0235,0.0784]$ to generate the adversarial data four times and mixed them into the Non-IID (b). For each dataset, we randomly selected two disjoint subsets (containing $500$ images) and compute the HSIC value over the two subsets. Repeating the above process $100$ times, we obtained the average value of the $100$ HSIC values in Table~\ref{tab:HSIC}. Since the HSIC value of adversarial data is higher than that of natural data (i.e., $\epsilon = 0$), the dependence within adversarial data is stronger than that within IID natural data. 

\paragraph{Dependence meets MMD tests.}
\citet{grosse2017statistical} and \citet{carlini2017adversarial} used the permutation based bootstrap \cite{oden1975arguments} to implement MMD-G test (i.e., the ordinary MMD test). Specifically, they initialize $a$ by $\MMD(S_X,S_Y)$. Then, they shuffle the elements of $S_X$ and $S_Y$ into two new sets $G_X$ and $G_Y$, and let $b = \MMD(G_X,G_Y)$. Repeating the shuffling process $K$ times, they can obtain a sequence $\{b_k\}_{k=1}^K$. If $a$ is greater than the $1-\alpha$ quantile of $\{b_k\}_{k=1}^K$, then the null hypothesis is rejected \cite{grosse2017statistical,carlini2017adversarial}. Namely, the adversarial attacks are detected. 

However, the permutation-based MMD-G test only works when facing IID data \cite{Kacper14wildbtp}. Since adversarial data may not be IID, according to \citet{Kacper14wildbtp}, \textit{the permutation based MMD-G} (previously used by \citet{grosse2017statistical} and \citet{carlini2017adversarial}) could be invalid to detect adversarial data.


\section{A Semantic-aware MMD Test}
\label{sec:SAMMD}

To take care of three factors missed by previous studies \cite{grosse2017statistical,carlini2017adversarial}, we design a simple and effective test motivated by 
the most important characteristic of adversarial data. Namely, semantic meaning of adversarial data (in the view of a well-trained classifier on natural data) is very different from that of natural data. Based on this characteristic, \emph{semantic-aware MMD} (SAMMD) is proposed to measure the discrepancy between natural and adversarial data in this section. 



\paragraph{Semantic features.}As mentioned above, the semantic meaning of data plays an important role to distinguish between natural and adversarial data. Thus, we will first introduce how to represent the semantic meaning of each image, i.e., to construct \emph{semantic features} of images, in this part. Since the success of deep learning mainly takes roots in its ability to extract features that can be used to classify images well, outputs of the layers of a well-trained deep neural network have already contained semantic meaning. Namely, we can construct semantic features of images using outputs of the layers of the well-trained network.

Figure~\ref{fig:moti_SAMMD2} visualizes outputs of the second to last full connected layers of a well-trained ResNet-18 and ResNet-34 using t-SNE \cite{maaten2008visualizing}, showing that these outputs indeed contain clear semantic meanings (in the view of natural data). Thus, we use these outputs as semantic features in this paper. This figure also shows that natural data and adversarial data are quite different in the view of semantic features. In addition, we also show the MMD values between semantic features of natural and adversarial data in Figure~\ref{fig:moti_SAMMD1}. Results show that, in the second to last full connected layer of ResNet-18, outputs of natural and adversarial data have the largest distributional discrepancy. Thus, the semantic features we constructed can help us distinguish adversarial data and natural data well.


\textbf{Semantic-aware MMD.}~~
Based on the semantic features, we consider the following semantic-aware deep kernel $k_{\omega}(\vx,\vy)$ to measure the similarity between two images:
\begin{align}
\label{eq:deepkernel_SAMMD}
 k_{\omega}(\vx,\vy) = \Big[(1-\epsilon_0)s_{\hat{f}}(\vx,\vy)+\epsilon_0\Big]q(\vx,\vy),
\end{align}
where $s_{\hat{f}}(\vx,\vy) = \kappa(\phi_p(\vx),\phi_p(\vy))$
is a deep kernel function that measures the similarity between $\vx$ and $\vy$ using semantic features extracted by $\hat{f}$; 
we use $\phi_p$, the second to the last fully connected layer in $\hat{f}$, to extract semantic features (according to Figure~\ref{fig:moti_SAMMD1});
the $\kappa$ is the Gaussian kernel (with bandwidth $\sigma_{\phi_p}$);  
$\epsilon_0\in(0,1)$ and $q(\vx,\vy)$ (the Gaussian kernel with bandwidth $\sigma_q$) are key components to ensure that $k_{\omega}(\vx,\vy)$ is a characteristic kernel \cite{liu2020learning} (ensuring that, SAMMD equals zero if and only if two distributions are the same \cite{liu2020learning}). Since $\hat{f}$ is fixed, the set of parameters of $k_{\omega}$ is $\omega=\{\epsilon_0,\sigma_{\phi_p},\sigma_q\}$. Based on $k_\omega(\vx,\vy)$ in Eq.~(\ref{eq:deepkernel_SAMMD}), $\textnormal{SAMMD}(\P,\Q)$ is
\begin{align*}
\sqrt{\E\left[ k_\omega(X, X') + k_\omega(Y, Y') - 2 k_\omega(X, Y) \right]},
\end{align*}
where $X,X'\sim\P$, $Y,Y'\sim\Q$. We can estimate $\textnormal{SAMMD}(\P,\Q)$ using the $U$-statistic estimator, which is unbiased for $\textnormal{SAMMD}^2(\P,\Q)$: 
\begin{align}\label{eq:sammd_estimator}
\widehat{\textnormal{SAMMD}}_u^2(S_X, S_Y; k_\omega)
= \frac{1}{n (n-1)} \sum_{i \ne j} H_{ij},
\end{align}
where $
H_{ij} = k_\omega(\vx_i, \vx_j) + k_\omega(\vy_i, \vy_j) - k_\omega(\vx_i, \vy_j) - k_\omega(\vy_i, \vx_j)$. 

\begin{figure}[t]
    \begin{center}
        {\includegraphics[width=0.4\textwidth]{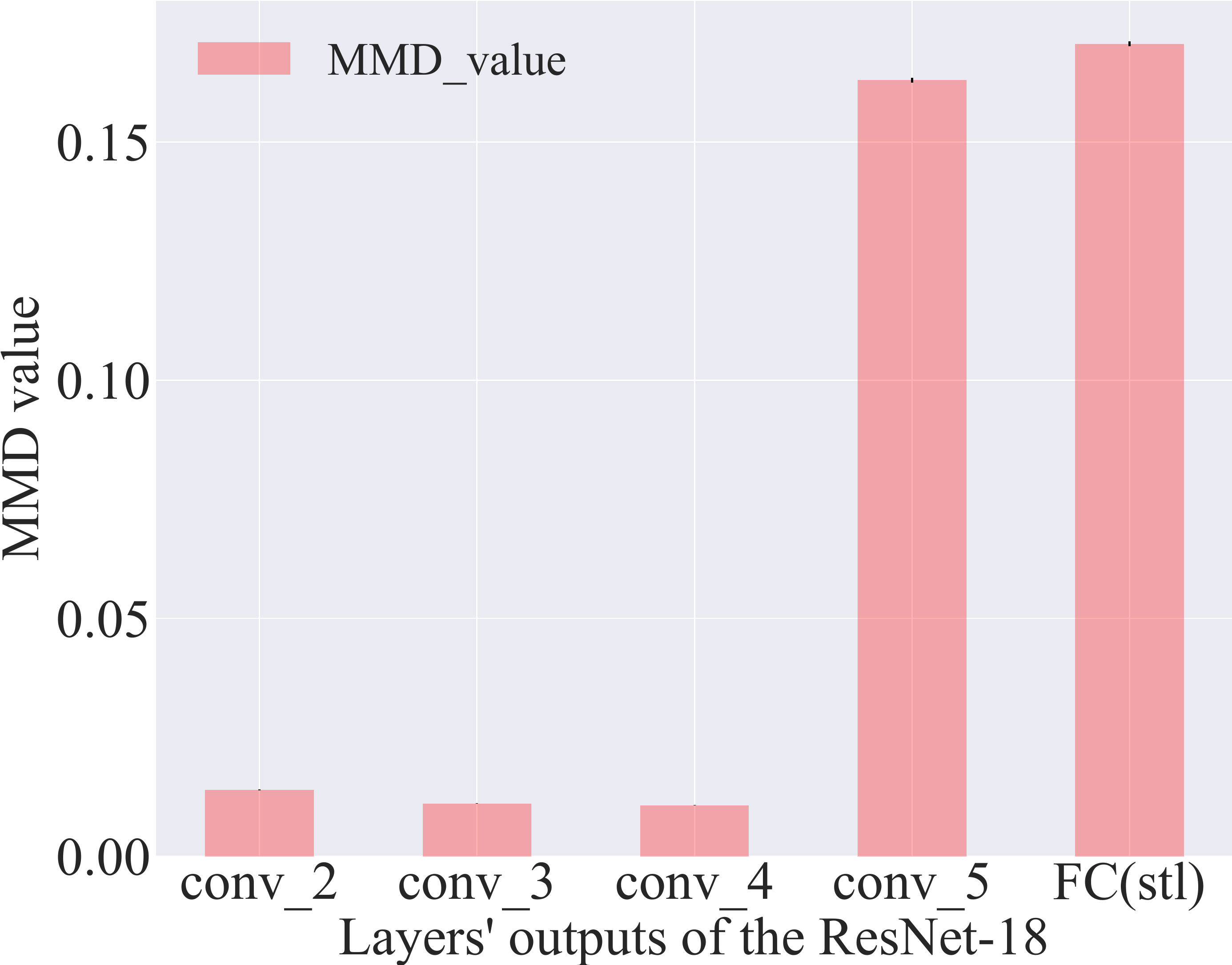}}
        \caption{\footnotesize
        Discrepancy of MMD value between different layers' outputs in $\hat{f}$. The figure shows MMD value between outputs of $5$ different layers of ResNet-18. It is clear that, in the conv$\_5$ layers, outputs of natural and adversarial data have larger distributional discrepancy compared to outputs of other $3$ convolutional layers. FC(stl) is the second to the last fully-connected layer and also an average pooling layer. Compared to the conv$\_5$ layer, the FC(stl) layer has fewer dimensions and its outputs can help measure the discrepancy between natural and adversarial data well.}
        \label{fig:moti_SAMMD1}
    \end{center}
\end{figure}

\begin{figure*}[htb]
    \begin{center}
        \subfigure[Type I error]
        {\includegraphics[width=0.245\textwidth]{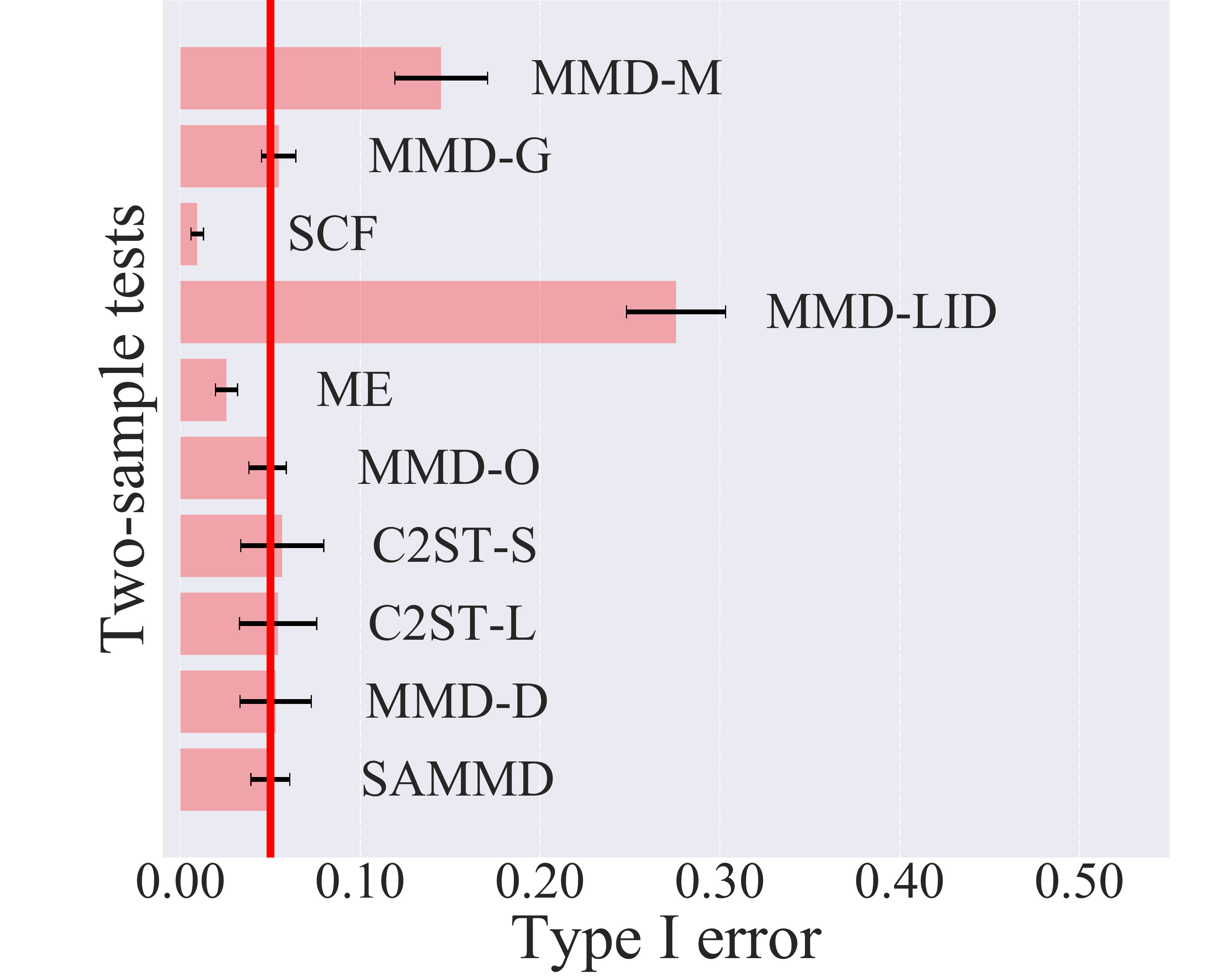}}
        \subfigure[Six different attacks]
        {\includegraphics[width=0.245\textwidth]{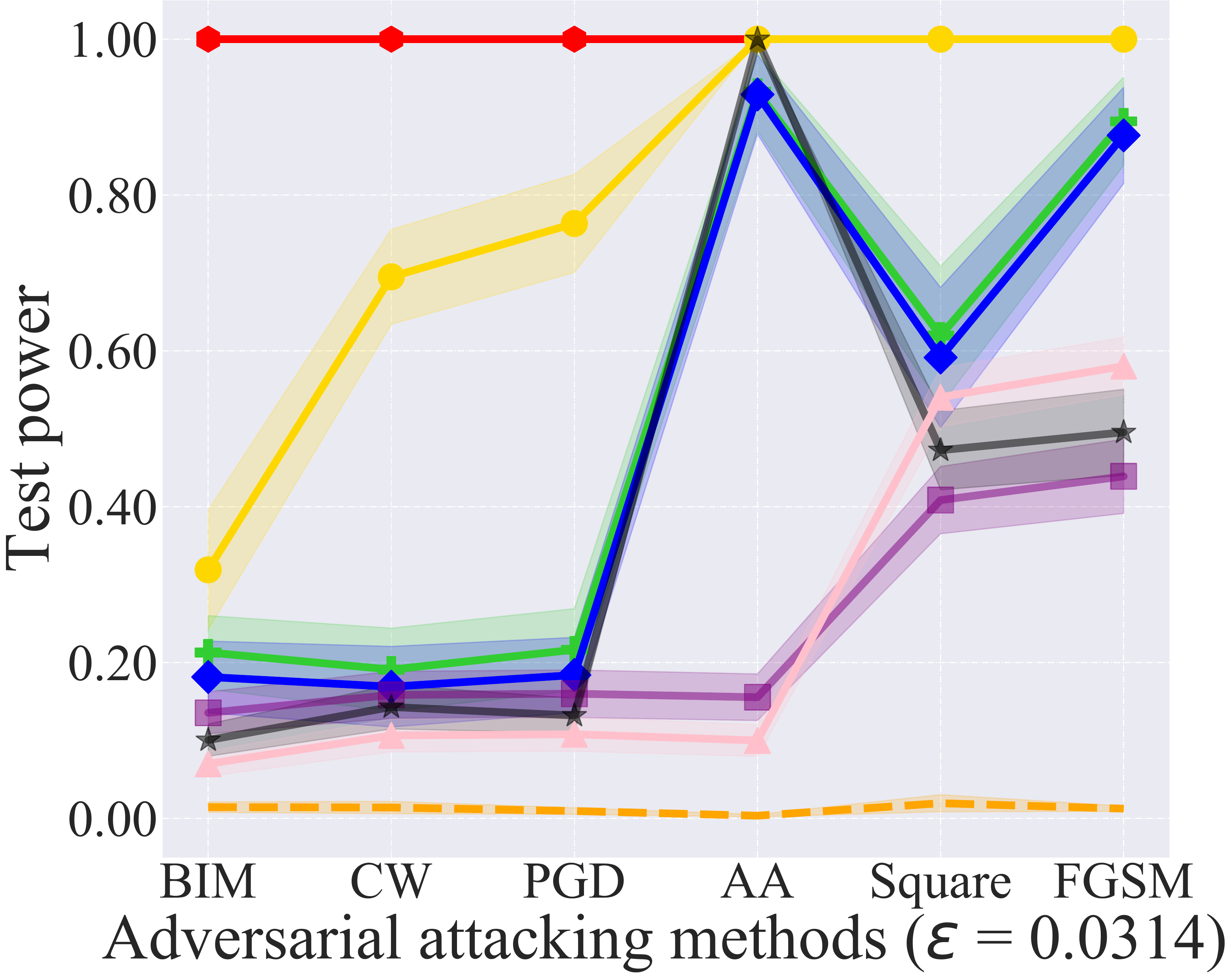}}
        \subfigure[Different $\epsilon$ of FGSM]
        {\includegraphics[width=0.245\textwidth]{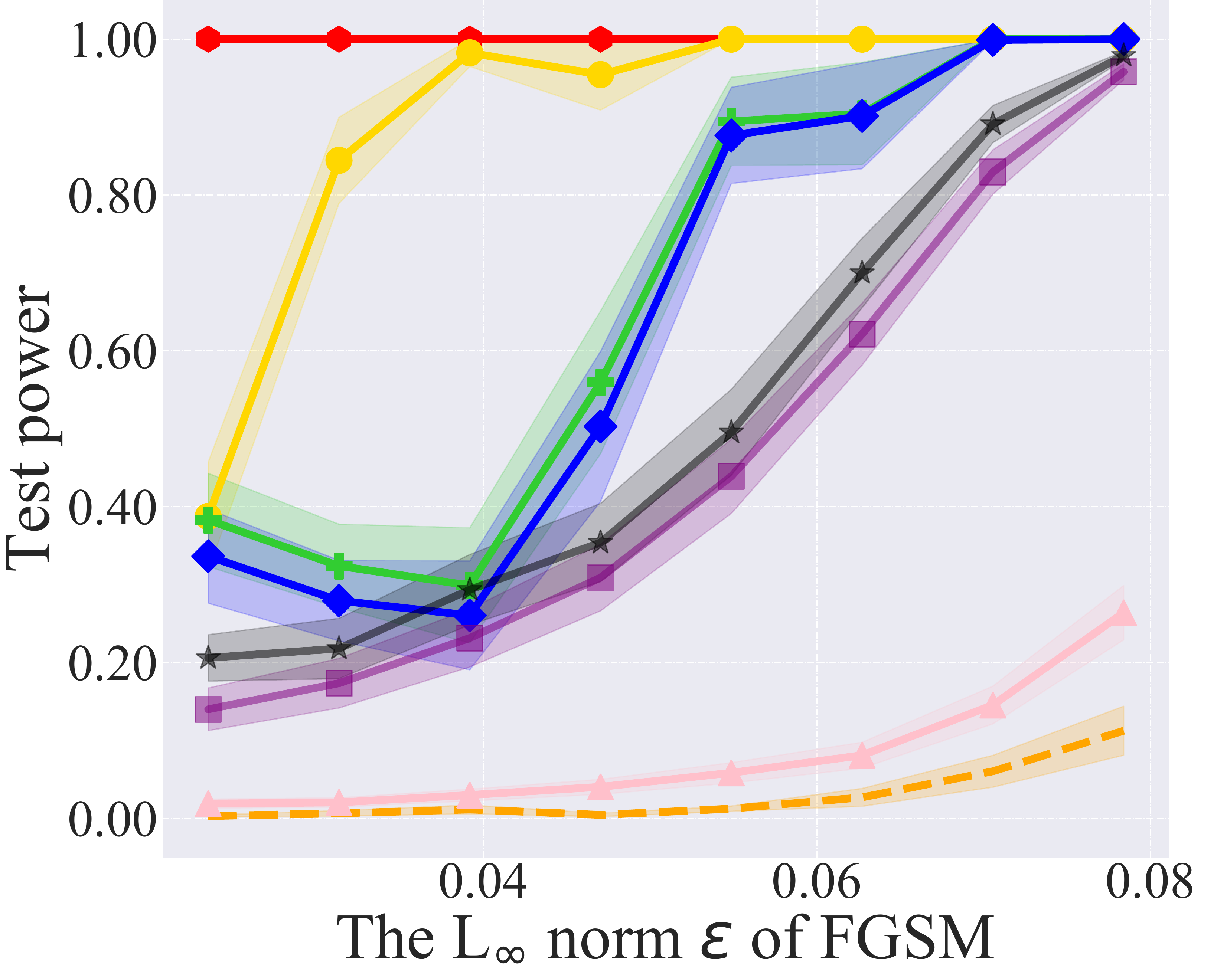}}
        \subfigure[Different $\epsilon$ of BIM]
        {\includegraphics[width=0.245\textwidth]{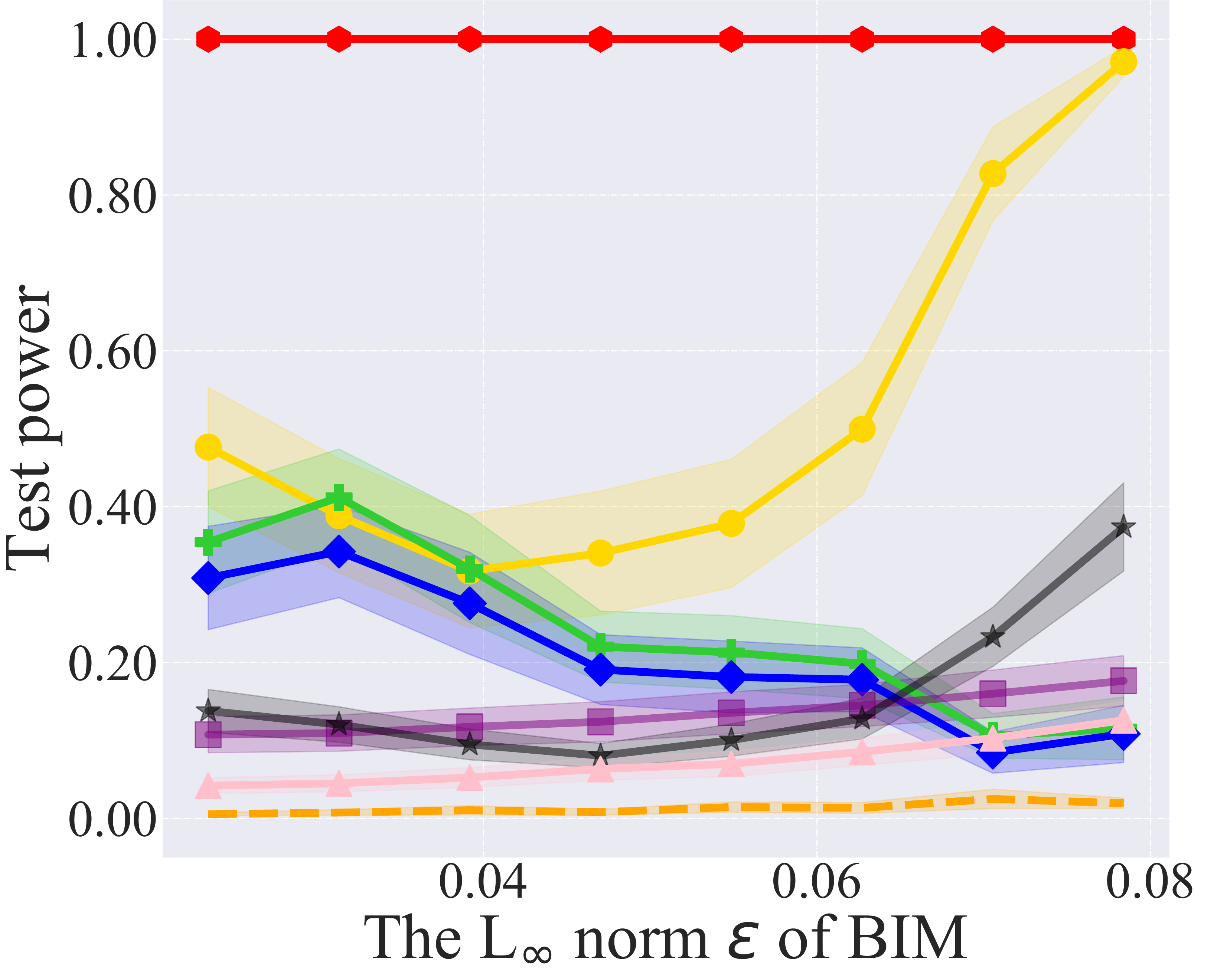}}
        \subfigure[Different $\epsilon$ of CW]
        {\includegraphics[width=0.245\textwidth]{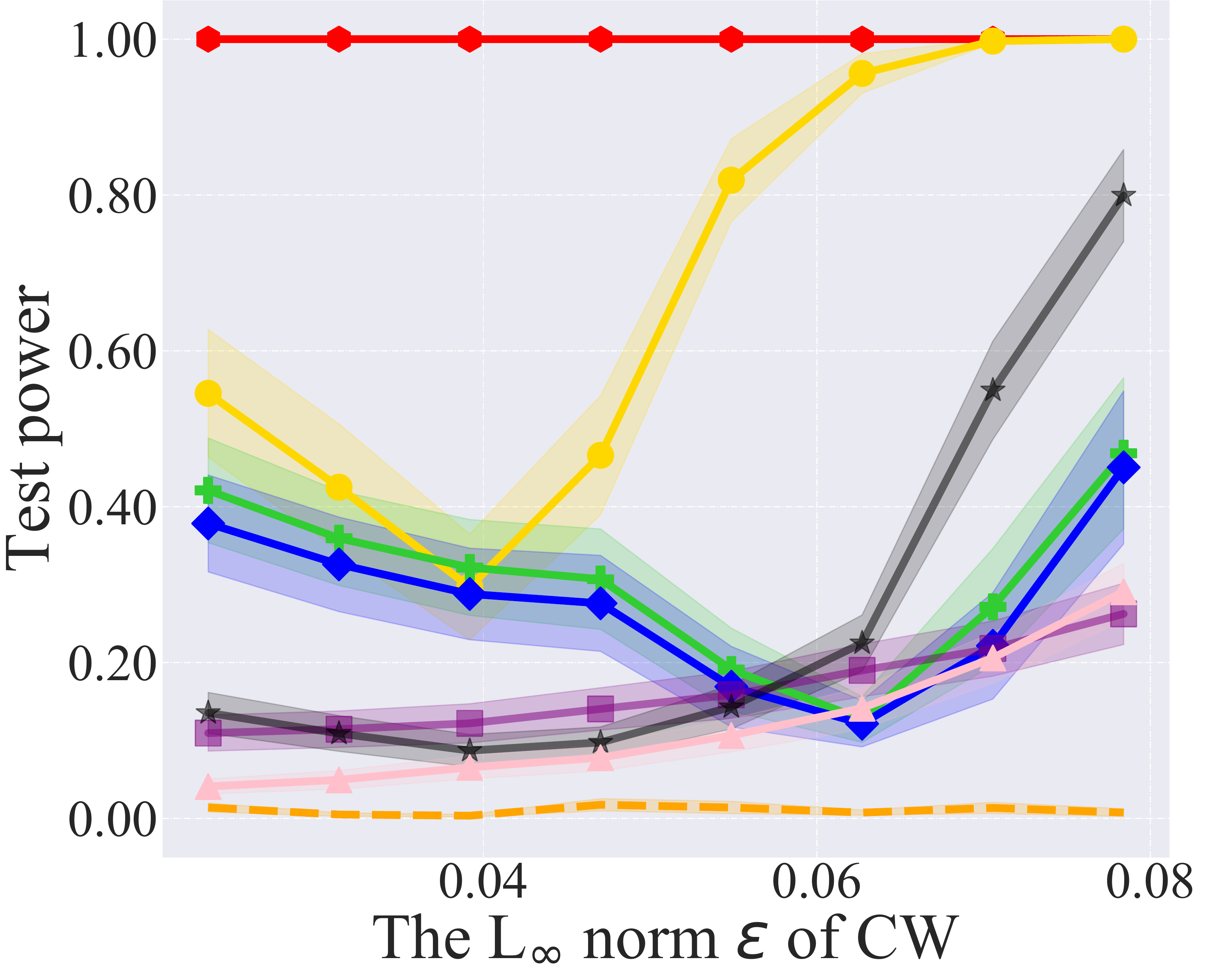}}
        \subfigure[Different $\epsilon$ of AA]
        {\includegraphics[width=0.245\textwidth]{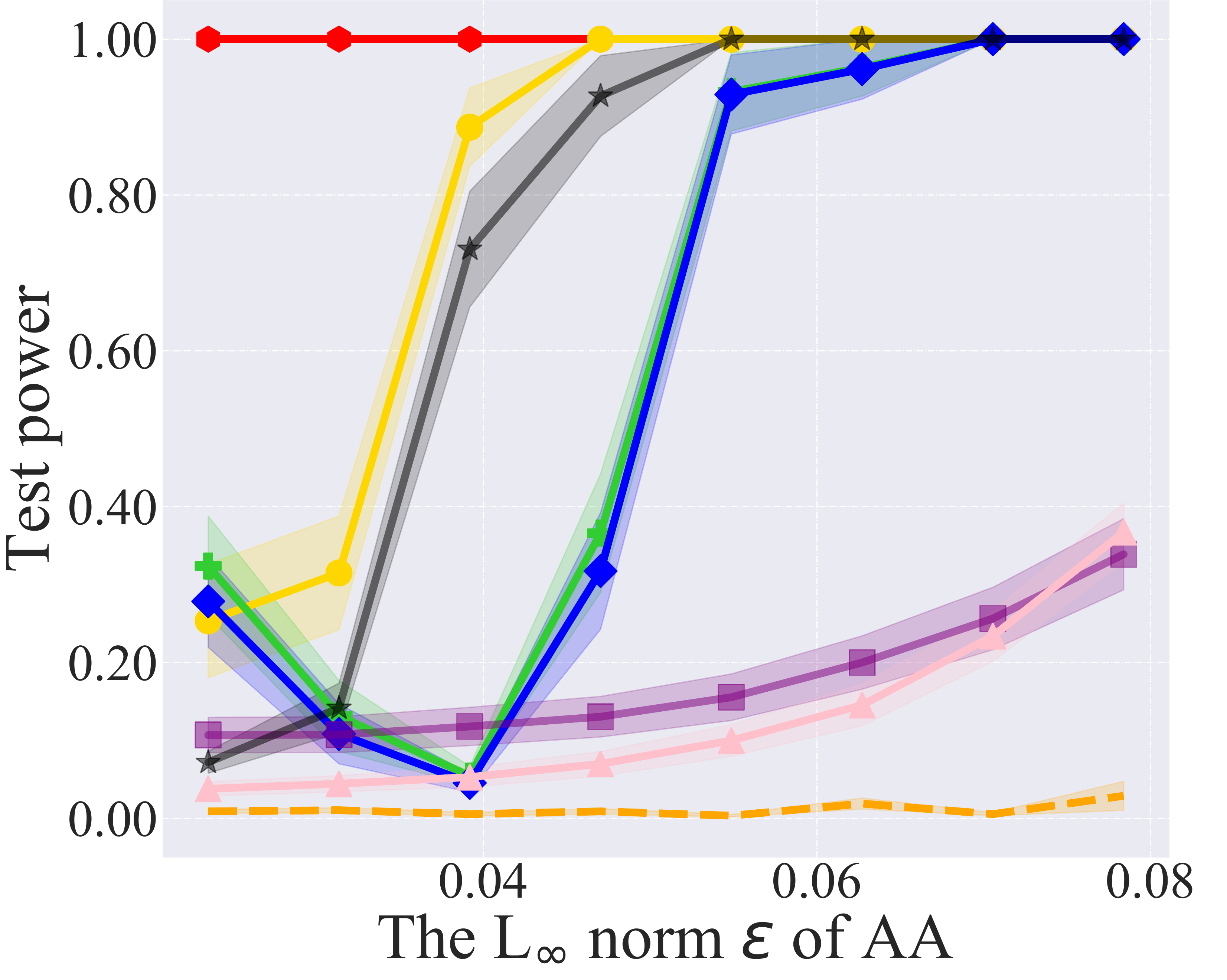}}
        \subfigure[Different $\epsilon$ of PGD]
        {\includegraphics[width=0.245\textwidth]{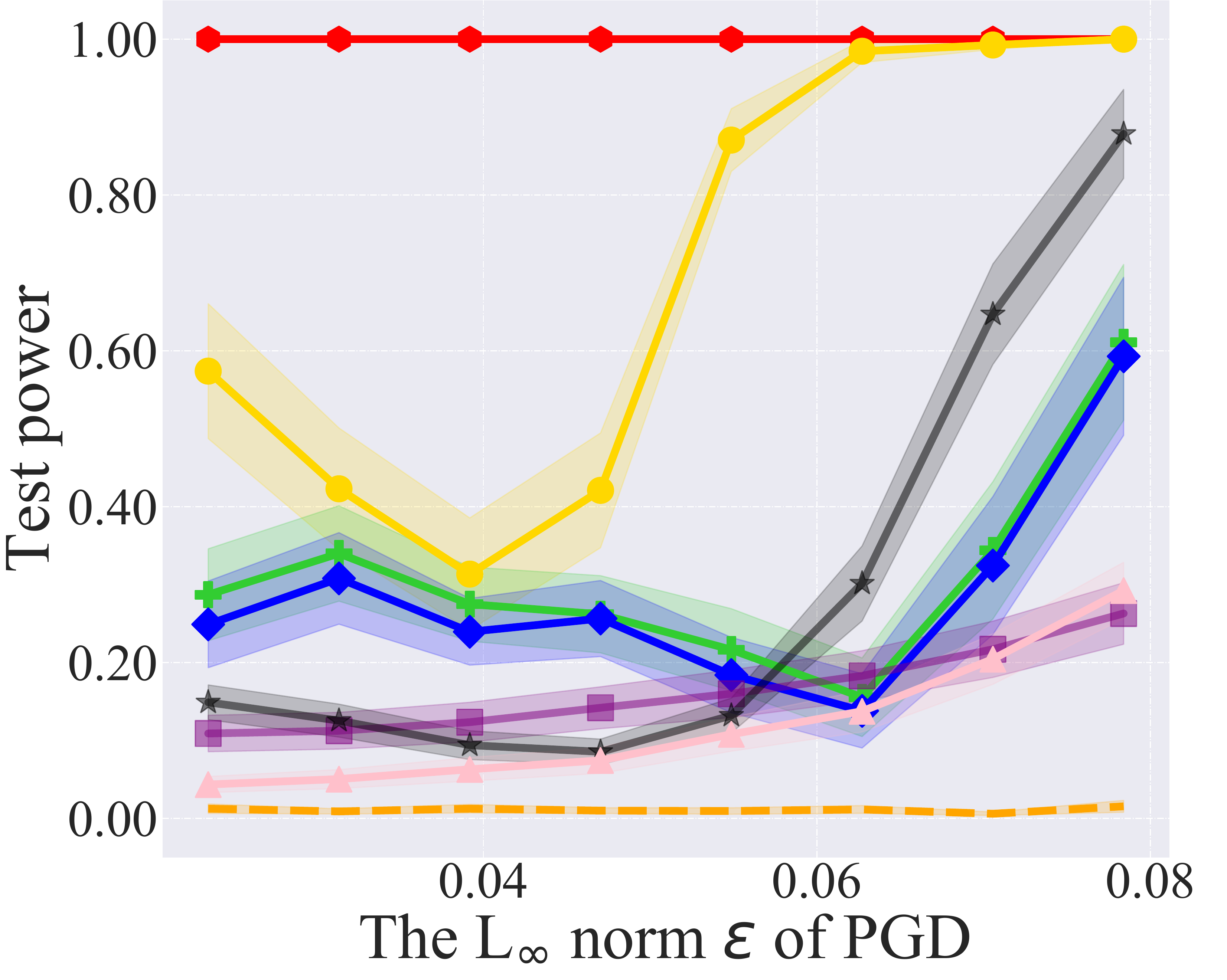}}
        \subfigure[Different set sizes]
        {\includegraphics[width=0.245\textwidth]{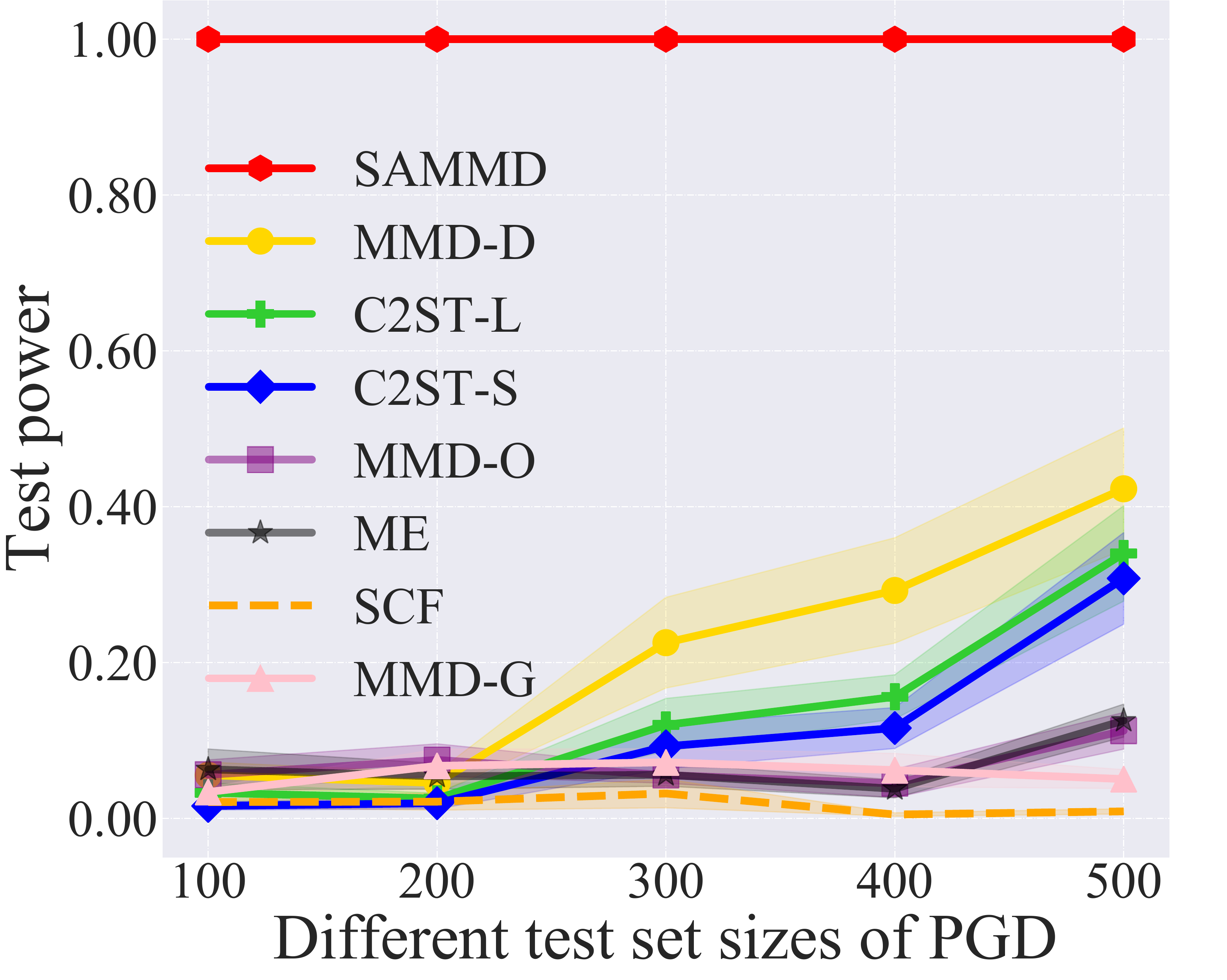}}
        \subfigure[Mixture of adv and natural]
        {\includegraphics[width=0.245\textwidth]{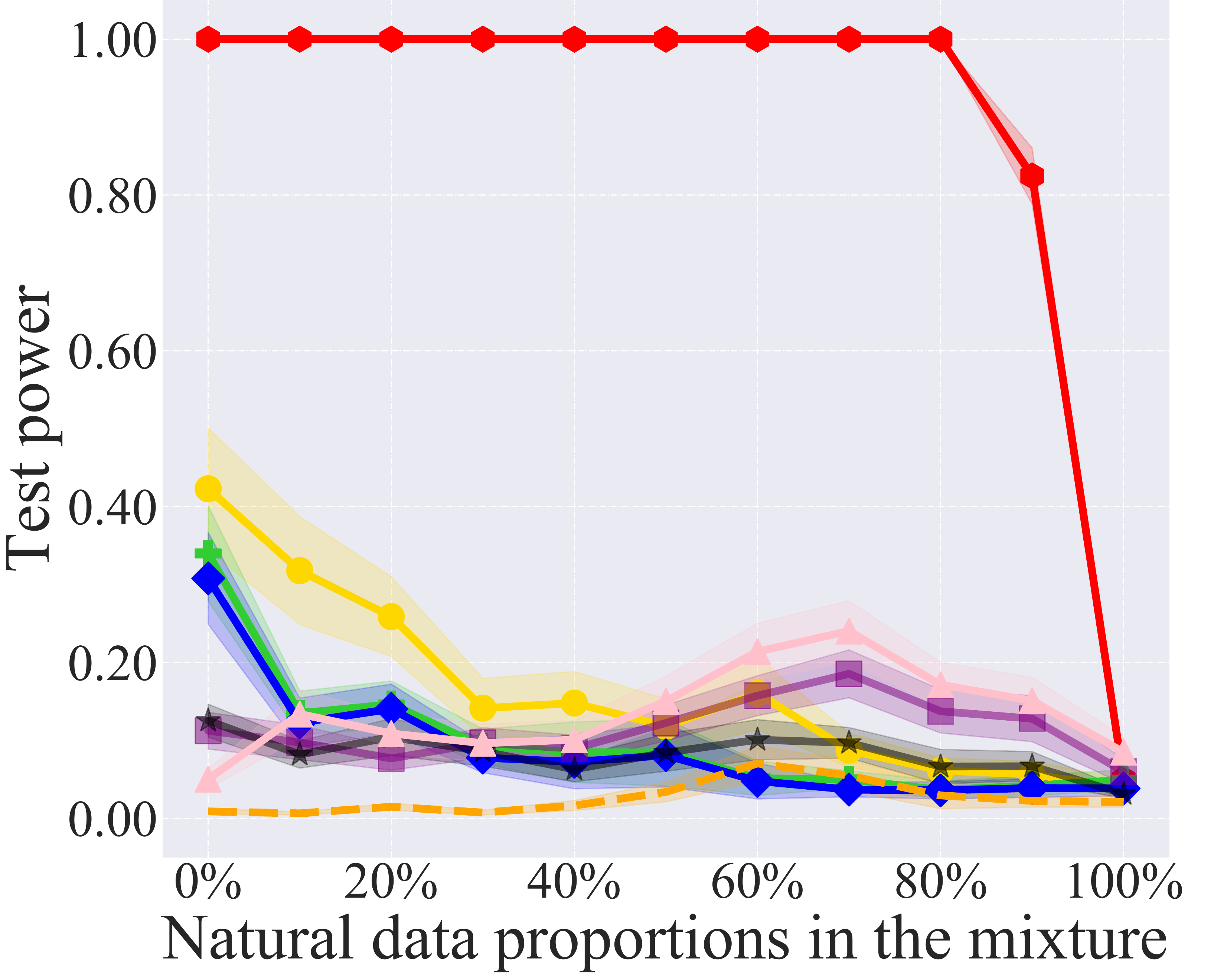}}
        \subfigure[SAMMD: Details of set sizes]
        {\includegraphics[width=0.245\textwidth]{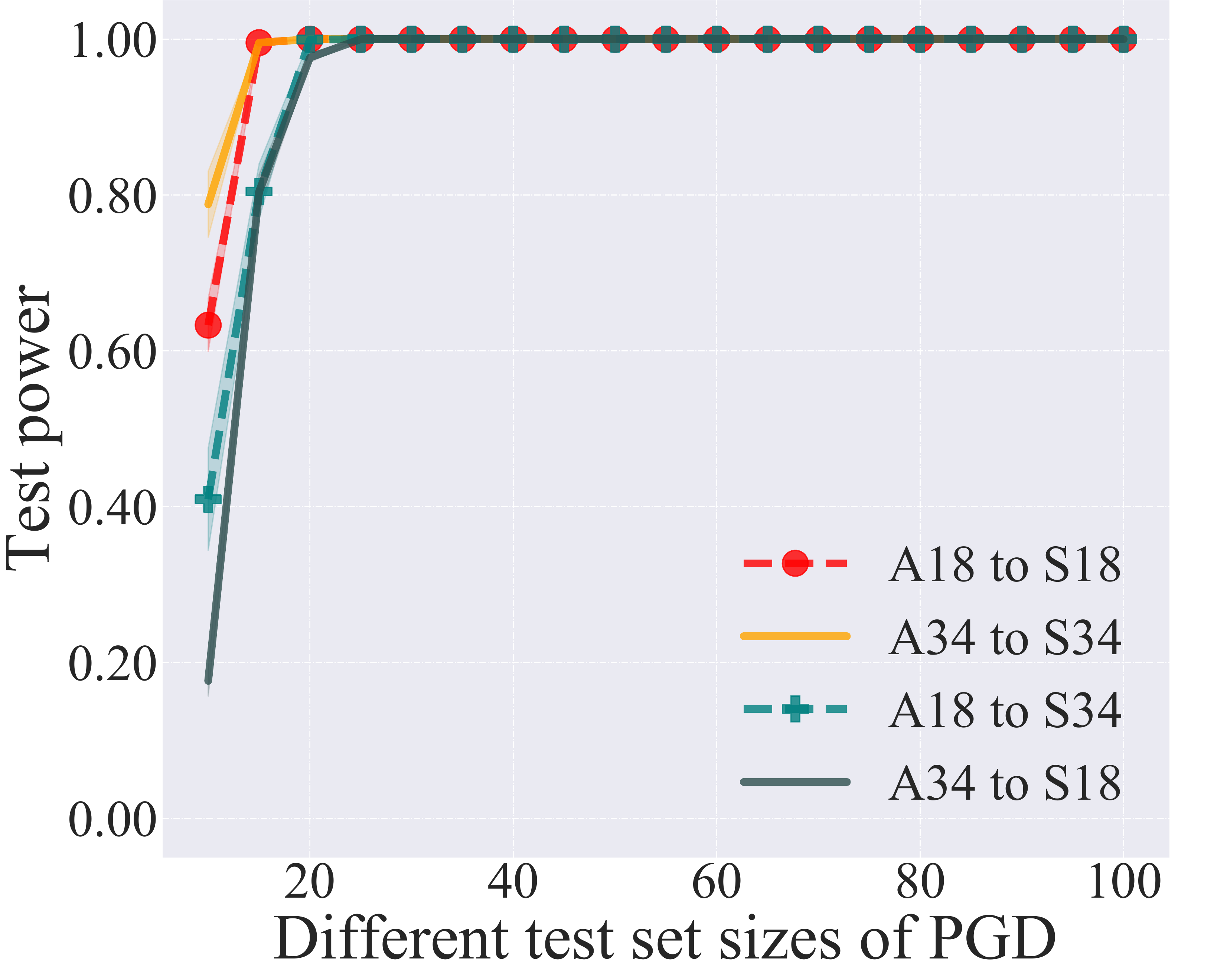}}
        \subfigure[SAMMD: Details of mixture]
        {\includegraphics[width=0.245\textwidth]{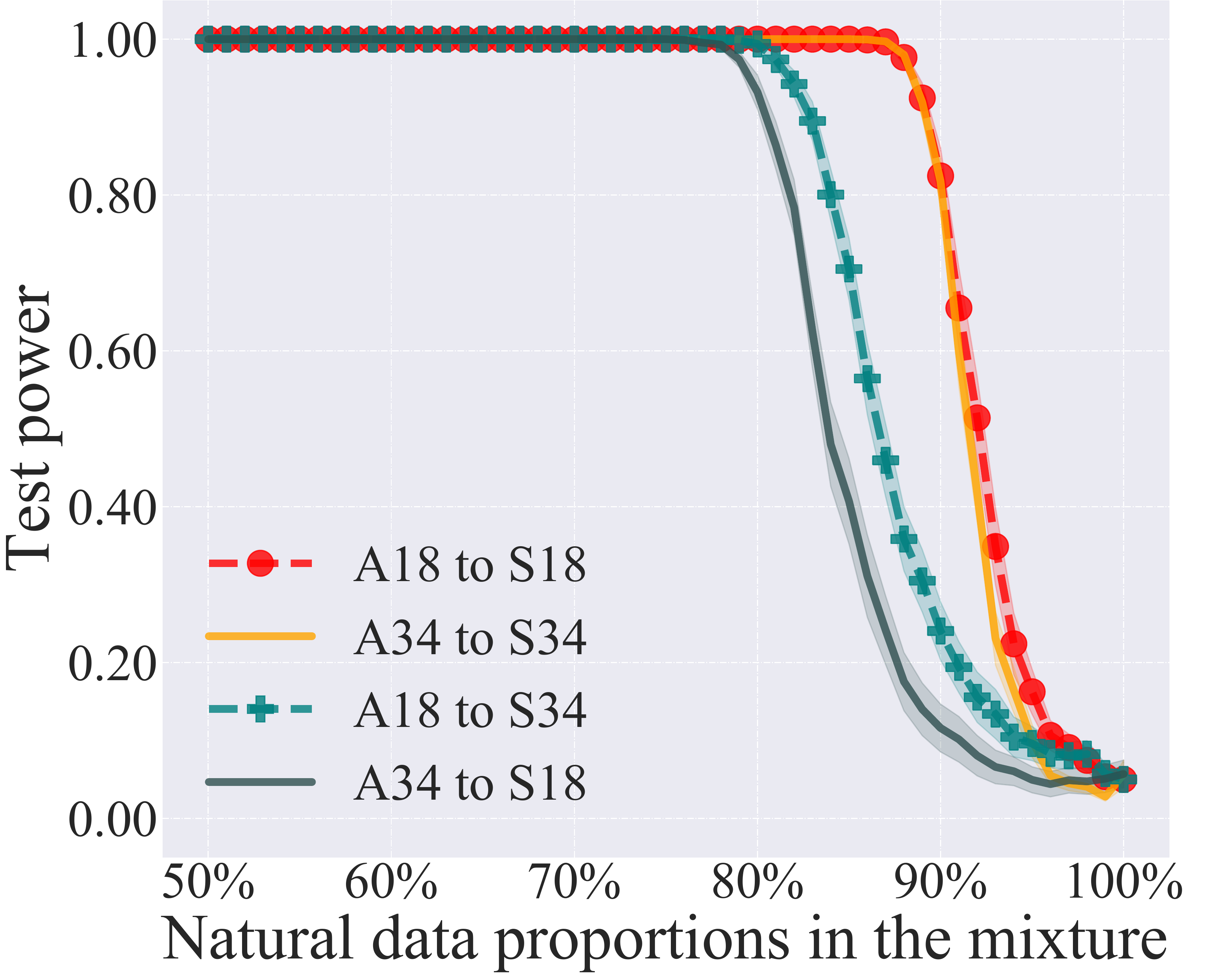}}
        \subfigure[Semantic features' study]
        {\includegraphics[width=0.245\textwidth]{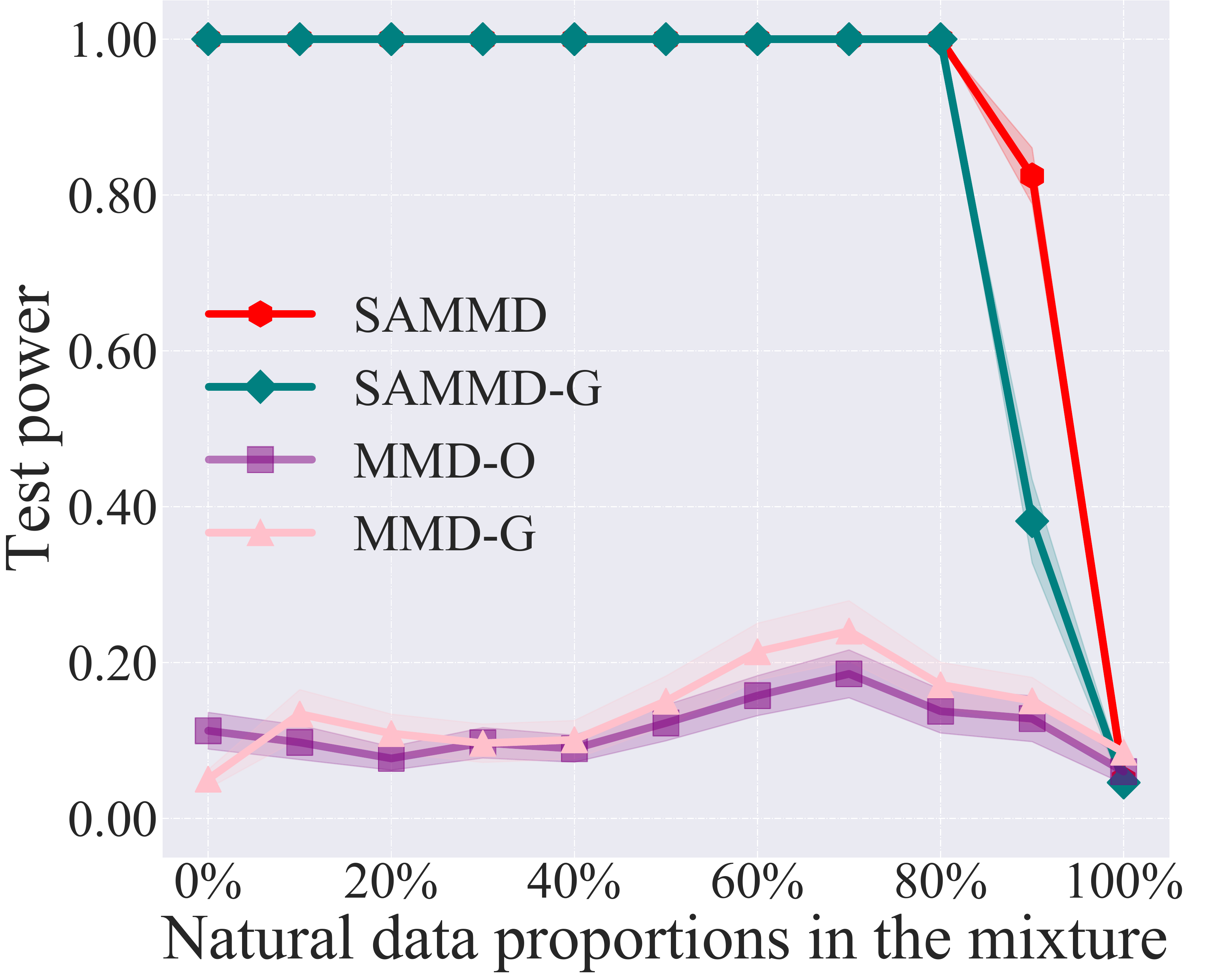}}
        \caption{\footnotesize Results of adversarial data detection. Subfigure (a) reports the type I error when $S_Y$ are natural data. 
        The ideal type I error should be around $\alpha$ (red line, $\alpha=0.05$ in this paper). Subfigures (b)-(l) report the test power (i.e., the detection rate) when $S_Y$ are adversarial data (or the mixture of adversarial and natural data). The ideal test power is $1$ (i.e., $100\%$ detection rate). Subfigures (b) - (i) share the same legend presented in subfigure (h). Details of subfigures are explained in Section~\ref{experiments}.
      }
        \label{fig:results}
    \end{center}
    \vspace{-1em}
\end{figure*}

\begin{figure*}[htb]

    \begin{center}
        \subfigure[Non-IID (a): FGSM]
        {\includegraphics[width=0.245\textwidth]{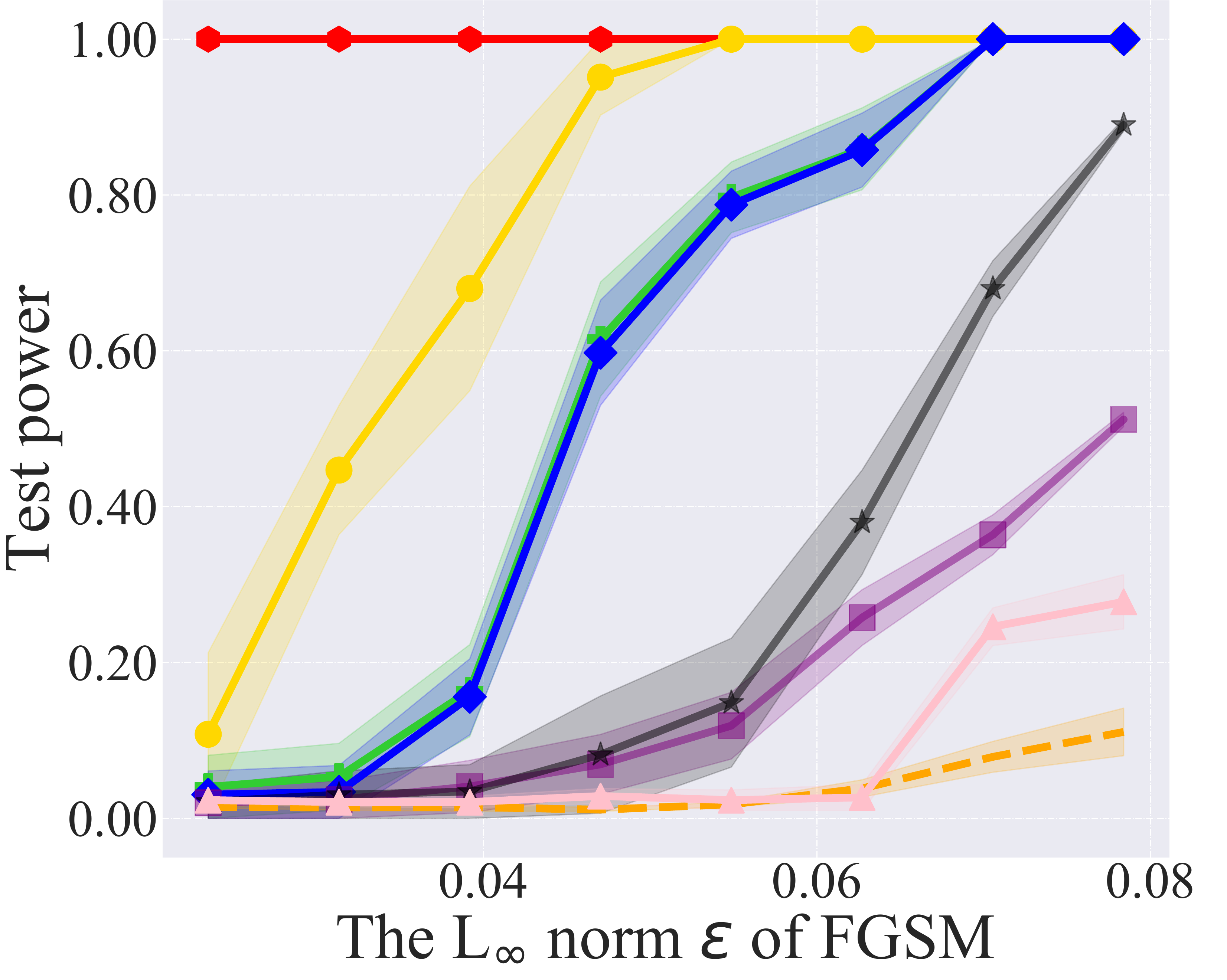}}
        \subfigure[Non-IID (a): BIM]
        {\includegraphics[width=0.245\textwidth]{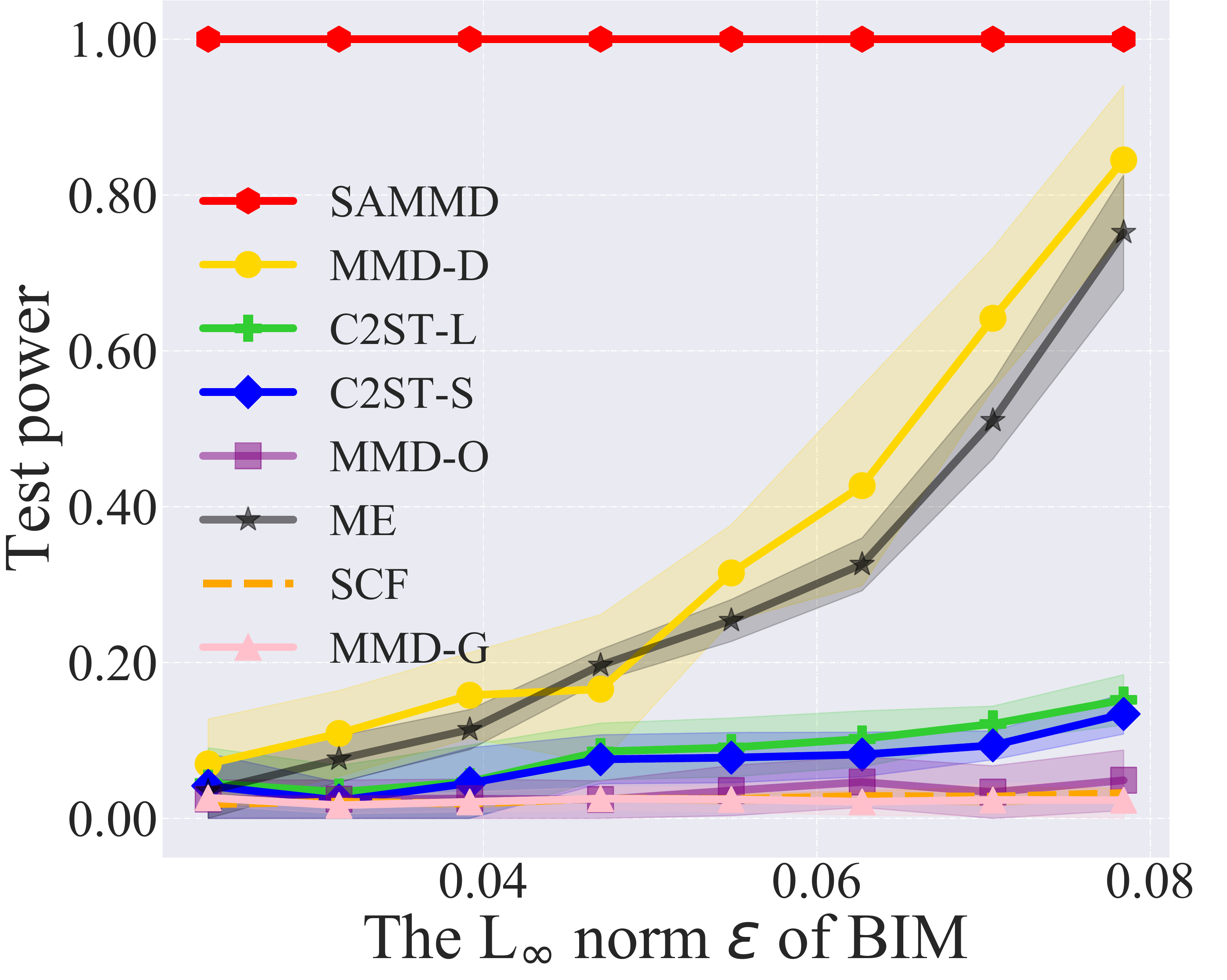}}
        \subfigure[Non-IID (a): PGD]
        {\includegraphics[width=0.245\textwidth]{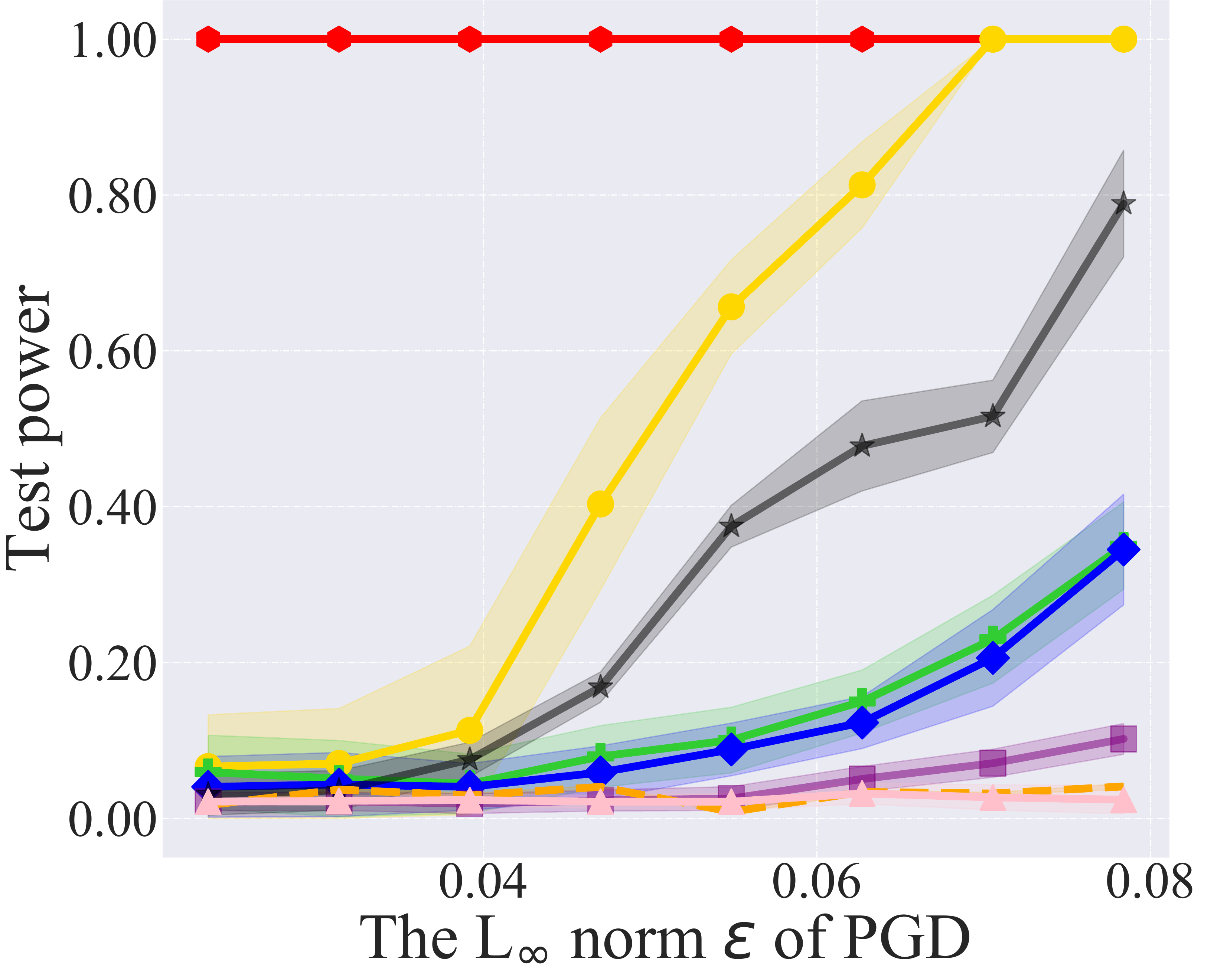}}
        \subfigure[Non-IID (a): CW]
        {\includegraphics[width=0.245\textwidth]{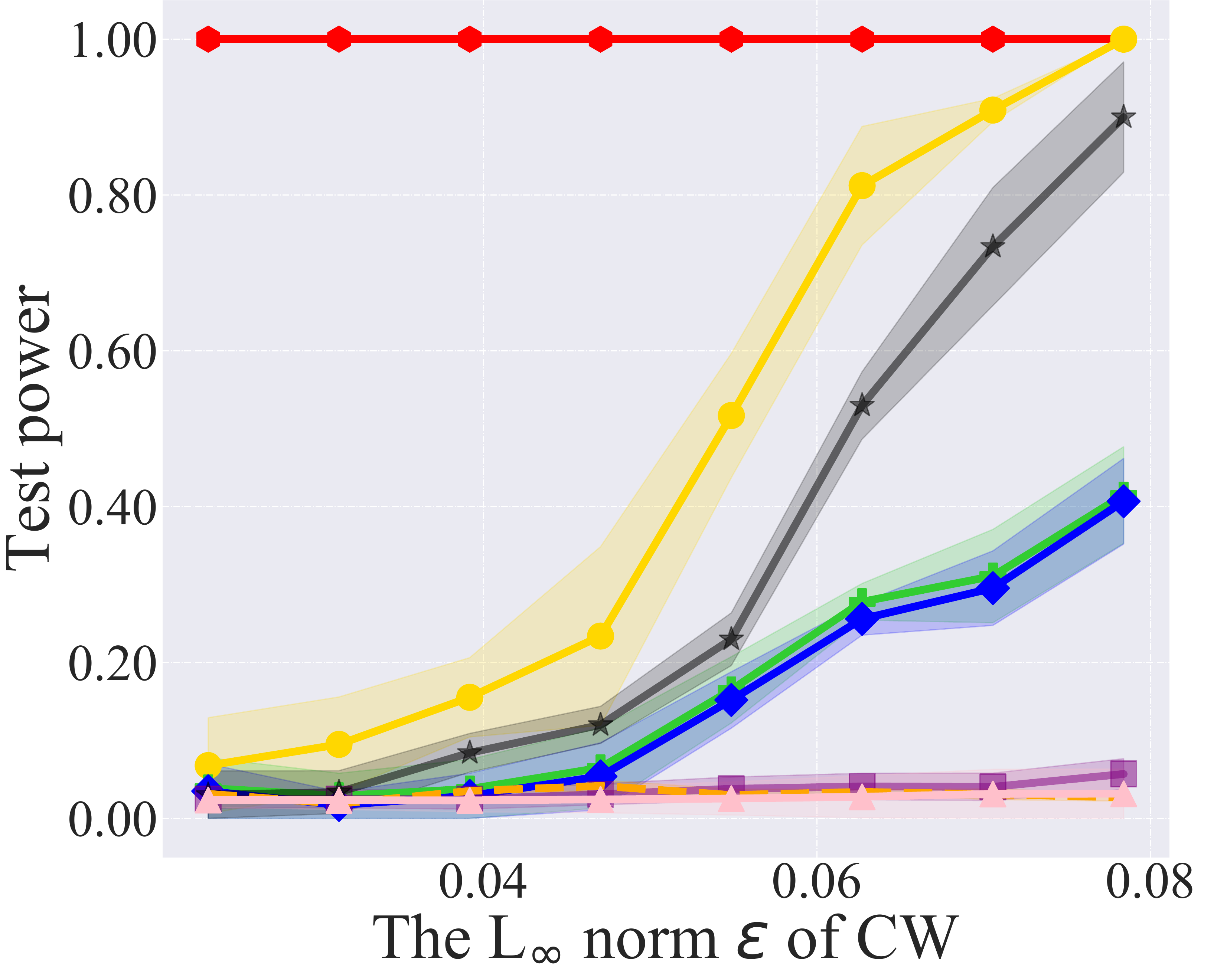}}
        \subfigure[Non-IID (a): AA]
        {\includegraphics[width=0.245\textwidth]{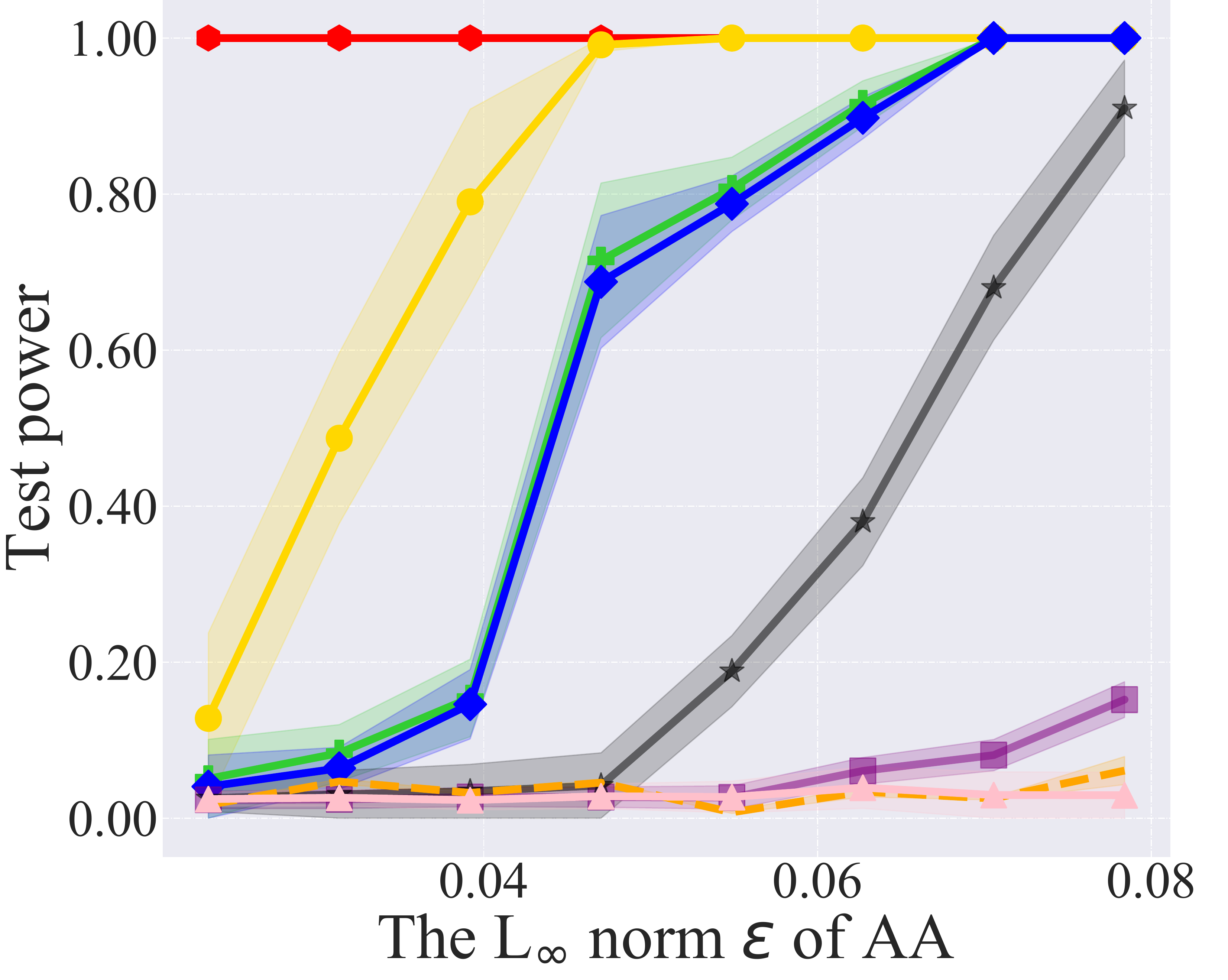}}
        \subfigure[Non-IID (b): Square]
        {\includegraphics[width=0.245\textwidth]{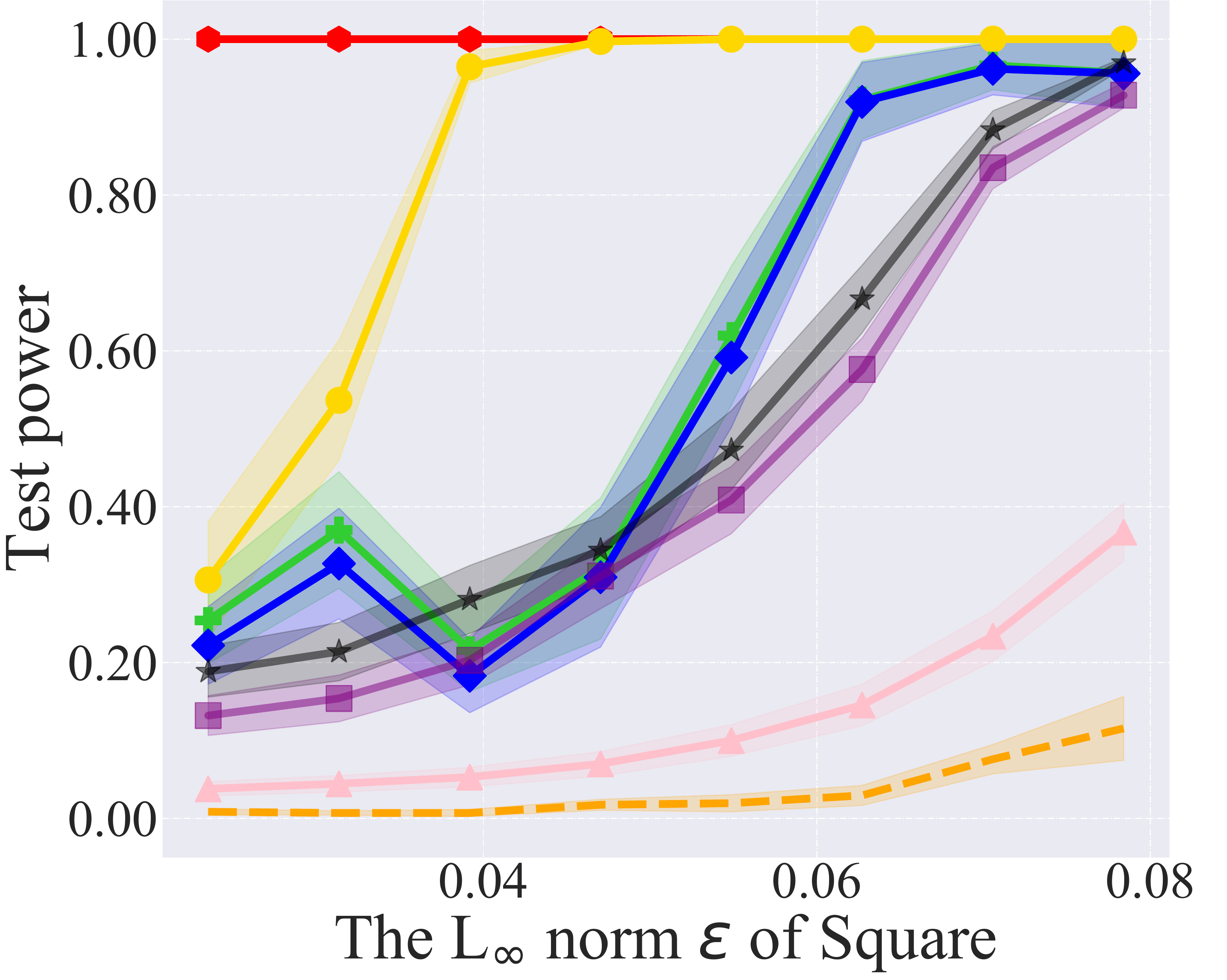}}
        \subfigure[Adaptive attacks]
        {\includegraphics[width=0.245\textwidth]{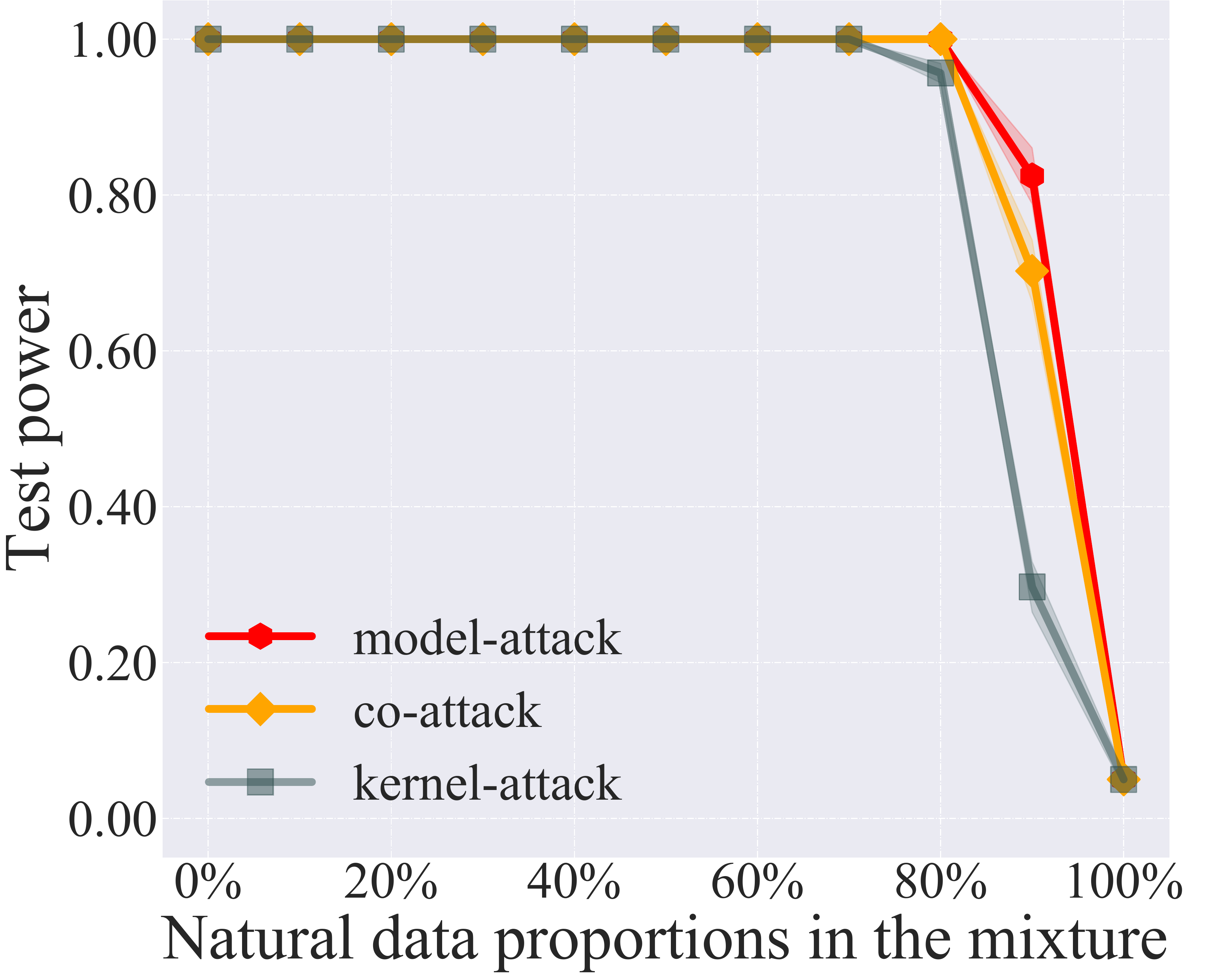}}
        \subfigure[Ablation study]
        {\includegraphics[width=0.245\textwidth]{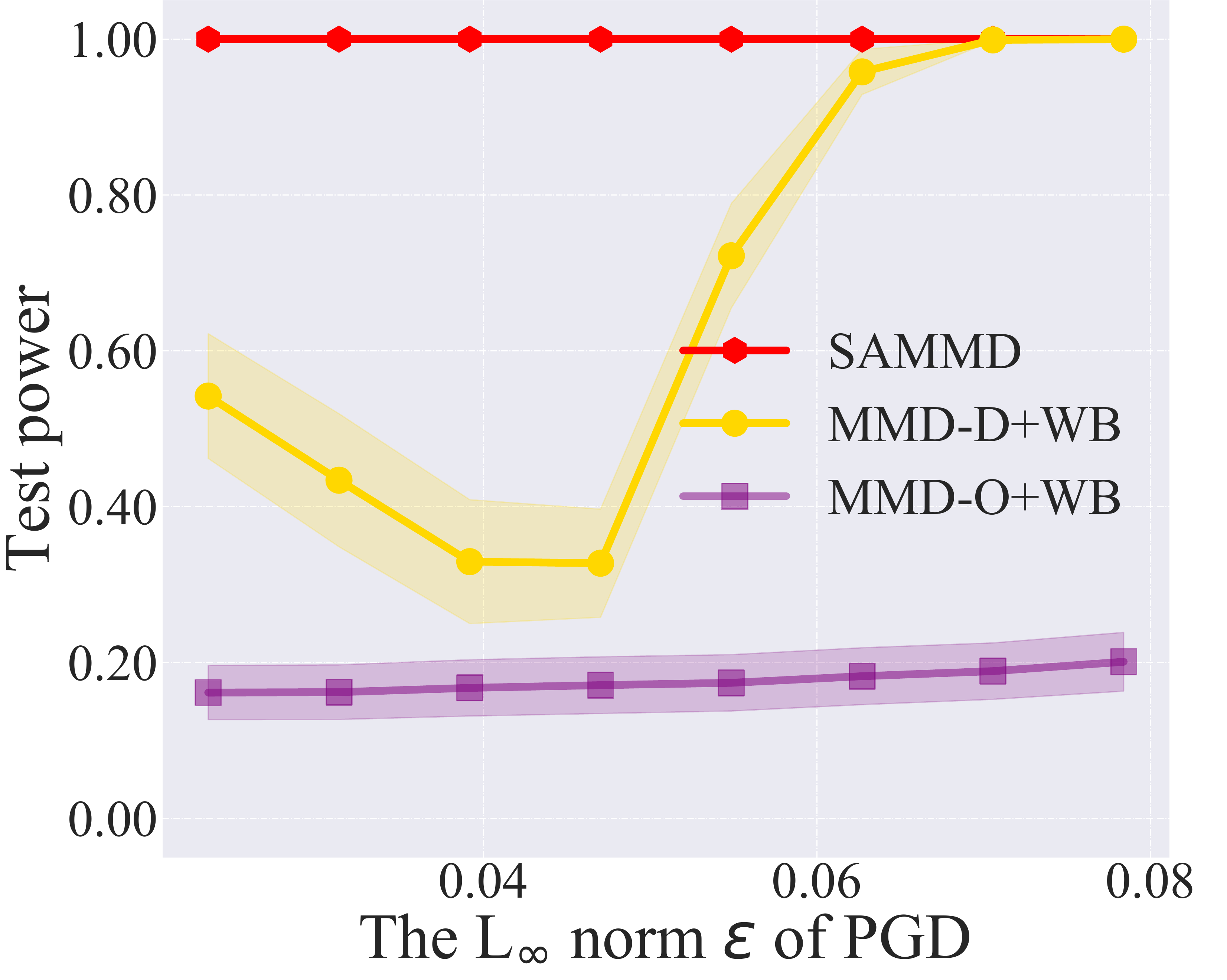}}
        \caption{\footnotesize Results of adversarial data detection. Subfigures (a)-(f) report the test power (i.e., the detection rate) when $S_Y$ are non-IID adversarial data. Subfigure (g) reports the test power when $S_Y$ are adaptive adversarial data. Subfigure (h) reports an ablation study. Subfigures (a)-(f) share the same legend presented in subfigure (b). Details of subfigures are explained in Section~\ref{experiments}.
      }
        \label{fig:nonIID}
    \end{center}
\end{figure*}

\paragraph{Asymtotics and test power of SAMMD.}
In this part, we analyze the asymtotics of SAMMD when $S_Y$ are adversarial data. Based on the asymtotics of SAMMD, we can estimate its test power that can be used to optimize the SAMMD (i.e., optimizing parameters in $k_\omega(\vx,\vy)$).
\begin{theorem}[Asymptotics under $H_1$] \label{prop:asymptotics}
Under the alternative $H_1 :$ $S_Y$ are from a stochastic process $\{Y_i\}_{i=1}^{+\infty}$, under mild assumptions, 
we have
\begin{gather*}
\sqrt{n}(\widehat{\textnormal{SAMMD}}_u^2 - \textnormal{SAMMD}^2) \overset{\textbf{d}}{\to} \N(0, C_1^2\sigma^2_{H_1}),
\end{gather*}
where $Y_i=\mathcal{G}_{\ell,\hat{f}}(\epsball[X_i])\sim\Q$, $X_i\sim \P$, $\sigma_{H_1}^2 = 4(\E_{Z}[(\E_{Z'}h(Z,Z'))^2]-[(\E_{Z,Z'}h(Z,Z'))^2])$, $h(Z,Z')$ $=k_{\omega}(X,X')+k_{\omega}(Y,Y')-k_{\omega}(X,Y')-k_{\omega}(X',Y)$, $Z:=(X,Y)$ and $C_1<+\infty$ is a constant for a given ${\omega}$.
\end{theorem}

The detailed version of Theorem~\ref{prop:asymptotics} can be found in Appendix~\ref{Asec:asym_MMD}.
Using Theorem~\ref{prop:asymptotics}, we have
\begin{align}
\label{eq:test_power}
    {\Pr}_{H_1,r}^{\textnormal{SAMMD}}\to \Phi\Big( \frac{\sqrt{n} \textnormal{SAMMD}^2}{C_1\sigma_{H_1}} - \frac{r}{\sqrt{n} \, C_1\sigma_{H_1}} \Big),
\end{align}
where ${\Pr}_{H_1,r}^{\textnormal{SAMMD}}={\Pr}_{H_1}\Big( n \widehat{\textnormal{SAMMD}}_u^2> r \Big)$ is the test power of SAMMD, $\Phi$ is the standard normal CDF and $r$ is the rejection threshold related to $\P$ and $\Q$. Via Theorem~\ref{prop:asymptotics},
we know that $r$, $\textnormal{SAMMD}(\P,\Q)$, and $\sigma_{H_1}$ are constants.
Thus, for reasonably large $n$,
the test power of SAMMD is dominated by the first term (inside $\Phi$),
and we can optimize $k_{\omega}$ by maximizing 
\begin{align*}
    J(\P, \Q; k_{\omega})= \textnormal{SAMMD}^2(\P, \Q; k_{\omega}) / \sigma_{H_1}(\P, \Q; k_{\omega}).
\end{align*}
Note that, we omit $C_1$ in $J(\P, \Q; k_{\omega})$, since $C_1$ can be upper bounded by a constant $C_0$ (see Appendix~\ref{Asec:asym_MMD}).

\begin{algorithm}[!t]
\footnotesize
\caption{The SAMMD Test}
\label{alg:learn_deep_kernel}
\begin{algorithmic}
\STATE \textbf{Input:} $S_X$, $S_Y$, $\hat{f}$,  various hyperparameters used below;

\vspace{1mm}
\STATE $\omega \gets \omega_0$; $\lambda \gets 10^{-8}$;  

\STATE Split the data as $S_X = S^{tr}_X\cup S^{te}_X$ and $S_Y = S^{tr}_Y\cup S^{te}_Y$;

\vspace{1mm}
\STATE \textit{\# Phase 1: train the kernel parameters $\omega$ and $\beta$; on $S^{tr}_X$ and $S^{tr}_X$\hfill}

\FOR{$T = 1,2,\dots,T_\mathit{max}$}

\STATE $S_X^{\prime} \gets$ minibatch from $S^{tr}_X$;

\STATE $S_Y^{\prime} \gets$ minibatch from $S^{tr}_Y$;

\STATE $k_\omega \gets$ kernel function with parameters $\omega$ using Eq. (\ref{eq:deepkernel_SAMMD});

\STATE $M(\omega) \gets \widehat{\textnormal{SAMMD}}_u^2(S_X^{\prime}, S_Y^{\prime}; k_\omega)$ using Eq.~(\ref{eq:sammd_estimator});

\STATE $V_\lambda(\omega) \gets \hat{\sigma}^2_{H_1,\lambda}(S_X^{\prime}, S_Y^{\prime}; k_\omega)$ using Eq. (\ref{eq:estimate_sigma_H1});

\STATE $\hat J_\lambda(\omega) \gets {M(\omega)}/\sqrt{V_\lambda(\omega)}$ using Eq.~(\ref{eq:tpp-hat});

\STATE $\omega \gets \omega + \eta\nabla_{\textnormal{Adam}} \hat J_\lambda(\omega)$;
       \hfill \textit{\# maximize $\hat J_\lambda(\omega)$}

\ENDFOR

\vspace{1mm}
\STATE \textit{\# Phase 2: testing with $k_\omega$ on $S^{te}_X$ and $S^{te}_Y$}

\STATE $\mathit{est} \gets \widehat{\textnormal{SAMMD}}_b^2(S^{te}_X, S^{te}_Y; k_\omega)$

\FOR{$i = 1, 2, \dots, n_\mathit{perm}$}
\STATE Generate $\{W^X_i\}_{i=1}^n$ and $\{W^Y_i\}_{i=1}^m$ using Eq.~(\ref{eq:wb_generator});

\STATE $\{\tilde{W}^X_i\}_{i=1}^n \gets \{{W}^X_i\}_{i=1}^n - \frac{1}{n}\sum_{i=1}^n{W}^X_i$;

\STATE $\{\tilde{W}^Y_i\}_{i=1}^m \gets \{{W}^Y_i\}_{i=1}^m - \frac{1}{m}\sum_{i=1}^m{W}^Y_i$;

\STATE $\mathit{perm}_i \gets \frac{1}{n (n-1)} \sum_{i,j} H_{ij}\tilde{W}^X_i\tilde{W}^Y_j$; \hfill \textit{\# resample}
\ENDFOR

\STATE \textbf{Output:} $k_\omega$, $\mathit{est}$, $p$-value: $\frac{1}{n_\mathit{perm}} \sum_{i=1}^{n_\mathit{perm}} \bm{1}(\mathit{perm}_i \ge \mathit{est})$
\end{algorithmic}
\end{algorithm}

\paragraph{Optimization of SAMMD.}
Although the higher value of criterion $J(\P, \Q; k_{\omega})$ means higher test power of SAMMD, we cannot directly maximize $J(\P, \Q; k_{\omega})$, since $\textnormal{SAMMD}^2(\P, \Q; k_{\omega})$ and $ \sigma_{H_1}(\P, \Q; k_{\omega})$ depend on the particular $\P$ and $\Q$ that are unknown. However, we can estimate it with
\begin{equation} \label{eq:tpp-hat}
  \hat J_\lambda(S_X, S_Y; k_{\omega}) := \frac{\widehat{\textnormal{SAMMD}}_u^2(S_X, S_Y; k_{\omega})}{\hat\sigma_{H_1,\lambda}(S_X, S_Y; k_{\omega})}
,\end{equation}
where $\hat\sigma_{H_1,\lambda}^2$ is a regularized estimator of $\sigma_{H_1}^2$ \cite{liu2020learning}: 
\begin{align} 
\label{eq:estimate_sigma_H1}
    \frac{4}{n^3} \sum_{i=1}^n \left( \sum_{j=1}^n H_{ij} \right)^2
    - \frac{4}{n^4}\left( \sum_{i=1}^n \sum_{j=1}^n H_{ij} \right)^2
    + \lambda
.\end{align}
Then we can optimize SAMMD by maximizing $\hat J_\lambda(S_X, S_Y; k_{\omega})$ on the training set (see Algorithm~\ref{alg:learn_deep_kernel}). Note that, although
\citet{sutherland:mmd-opt} and \citet{unbiased-var-ests} have given an unbiased estimator for $\sigma_{H_1}^2$,
it is much more complicated to implement.

\paragraph{The SAMMD test.}
Since adversarial data are probably not IID, we cannot simply use a permutation-based bootstrap method to simulate the null distribution of the SAMMD test \cite{Kacper14wildbtp}. To address this issue, wild bootstrap \cite{shao2010dependentwildbtp} is used to help simulate the null distribution of SAMMD, then we can test if $S_Y$ are from $\P$. To the end, the Algorithm~\ref{alg:learn_deep_kernel} shows the whole procedure of the SAMMD test. In Appendix~\ref{Asec:asym_MMD}, it has been shown that, under mild assumptions, the proposed SAMMD test is a provably consistent test to detect adversarial attacks.

\section{Experiments}
\label{experiments}


We verify detection methods on the ResNet-18 and ResNet-34 trained on the \textit{CIFAR-10} and the \textit{SVHN}. We also validate performance of SAMMD on the large network Wide ResNet (WRN-32-10) \cite{zagoruyko2016wide} and the large dataset \textit{Tiny-Imagenet}. Configuration of all experiments is in Appendix~\ref{Asec:exp_set}. Detailed experimental results are presented in Appendix~\ref{Asec:add_exp}. The code of our SAMMD test is available at \httpsurl{github.com/Sjtubrian/SAMMD}.
\paragraph{Baselines.}
We compare SAMMD test with $6$ existing two-sample tests: 1) MMD-G test used by \cite{grosse2017statistical}; 2) MMD-O test \cite{sutherland:mmd-opt}; 3) Mean embedding (ME) test \cite{Jitkrittum2016}; 4) Smooth characteristic functions (SCF) test \cite{Chwialkowski2015}; 5) Classifier two-sample test (C2ST) \cite{liu2020learning,Lopez:C2ST}; 6) MMD-D test \cite{liu2020learning}. 

Besides, we also try to construct two new MMD tests based on features commonly used by adversarial data classification methods: 1) MMD-LID: the \underline{MMD} with a Gaussian kernel whose inputs are \emph{\underline{l}ocal \underline{i}ntrinsic \underline{d}imensionality} (LID) features \cite{ma2018characterizing} of two samples. Then we optimize the Gaussian kernel by maximizing its test power; and 2) MMD-M: the \underline{MMD} with a Gaussian kernel whose inputs are \underline{mahalanobis distance} based features \cite{lee2018simple} of two samples. Then, we optimize the bandwidth of the Gaussian kernel by maximizing the test power.

\paragraph{Test power on different attacks.}
We first report the type I error of our SAMMD test and $9$ baselines when $S_Y$ are natural data in Figure~\ref{fig:results}{a}. 
It is clear that MMD-LID test and MMD-M test have much higher type I error than the given threshold $\alpha=0.05$ (the red line in Figure~\ref{fig:results}{a}). That is, both baselines are invalid two-sample tests.
The main reason is that LID features and mahalanobis-distance features are sensitive to any perturbation. The sensitivity leads to that MMD-LID test and MMD-M test will recognize natural data as adversarial data. Other methods except for SCF maintain reasonable type I errors. Since MMD-LID test and MMD-M test are invalid two-sample tests, we do not validate the test power of them in the remaining experiments.

For $6$ different attacks, FGSM, BIM, PGD, AA, CW and Square (Non-IID(b)), we report the test power of all tests when $S_Y$ are adversarial data ($L_\infty$ norm $\epsilon = 0.0314$; set size $= 500$) in Figure~\ref{fig:results}{b}.
Results show that SAMMD test performs the best and achieves the highest test power.

\paragraph{Test power on different $\epsilon$.}
In addition to different adversarial attacks, different perturbation bound $\epsilon$ can also affect the adversarial data generation process. If the adversarial attack is within a small perturbation bound, the generated adversarial data is not sufficient to fool the well-trained natural-data classifier \cite{tramer2020fundamental}. However, if the adversarial attack is within a big perturbation bound, natural information contained in images will be completely lost \cite{tramer2020fundamental,zhang2020attacks}. 

Following previous studies \cite{carlini2017towards,Madry18PGD,wang2019convergence,tramer2020fundamental,zhang2020attacks,Chen_2020_CVPR,wu2020adversarial,zhang2020dual}, we set the $L_\infty$-norm bounded perturbation $\epsilon \in [0.0235,0.0784]$. The lower bound of $0.0235$ is calculated by $6$/$255$, i.e., the maximum variation of each pixel value is $6$ intensities, and the upper bound of $0.0784$ is calculated by $20$/$255$. This range covers all possible $\epsilon$ used in the literature \cite{Madry18PGD,zhang2020attacks}. 

We report the average test power (with its standard error) on different $\epsilon$ of FGSM, BIM, CW, AA and PGD (set size $=500$) in Figure~\ref{fig:results}(c)-(g). For the non-IID adversarial data mentioned in Section~\ref{sec:dependence_within_data}, we also report the average test power on different $\epsilon$ of the Non-IID (a) and the Non-IID (b) in Figure~\ref{fig:nonIID}(a)-(f). Given the training set of \textit{CIFAR-10}, we use FGSM, BIM, CW, AA and PGD to generate the Non-IID (a). Given the testing set of \textit{CIFAR-10}, we use Square to generate the adversarial data four times and mix them into the Non-IID (b). Results show that our SAMMD test also achieves the highest test power all the $\epsilon$.

\paragraph{Test power on different set sizes.}
The effective of previous kernel non-parametric two-sample tests like C2ST and MMD-D test \cite{Lopez:C2ST,liu2020learning} depends on a large size of data. Namely, they can only measure the discrepancy well when there are a large batch of data. Hence, we evaluate the performance of our SAMMD test and baselines with different set sizes. Experiments results are reported in Figure~\ref{fig:results}{h}, which shows that our SAMMD test is suitable to different data sizes.

\paragraph{Test power on the mixture of adversarial data and natural data.}
For practical concerns, it is often that only part of data is adversarial. We analyze test power of the SAMMD test and baselines in this case, with natural data mixture proportion ranging from $0\%$ to $100\%$. The experimental results of PGD ($L_\infty$ norm $\epsilon = 0.0314$; set size $= 500$) are presented in Figure~\ref{fig:results}{i}. Results show that the performance of our SAMMD test is much better than all baselines.  


\paragraph{Semantic featurizers meet unknown adversarial data.}
In the above setting, the semantic featurizers $\phi_p(\cdot)$ are also the classifiers subjected to adversarial attacks. In this part, we also consider that a dataset to be tested contains the adversarial data acquired by unknown classifiers. Hence, we analyze the performance when adversarial data and semantic features are acquired by different classifiers. We train two classifiers (ResNet-18 and ResNet-34) on natural data. One is used to acquire adversarial data (A$18$/A$34$), and the other is used to acquire semantic features (S$18$/S$34$).

Experiments results of our SAMMD test with different set sizes (from $10$ to $100$) are presented in Figure~\ref{fig:results}{j}. Experiments results of our SAMMD test with mixture proportion (from $0\%$ to $100\%$) are presented in Figure~\ref{fig:results}{k}. In Figures~\ref{fig:results}{j} and \ref{fig:results}{k}, the attack method is PGD ($L_\infty$ norm $\epsilon = 0.0314$; set size $= 500$). Results clearly show that our SAMMD test can also work well in this case. It is the existence of adversarial transferability that can help our SAMMD test defend against such attacks. 

\paragraph{Study of Semantic features.}
\label{Semantic_results}
In order to verify that semantic features are better to help measure the distribution discrepancy between adversarial data and natural data than raw features, we also test the semantic features (the same with the SAMMD test) and the Gaussian kernel of the fixed bandwidth (SAMMD-G in Figure~\ref{fig:results}{l}). Experiments in Figure~\ref{fig:results}{l} confirms the importance of semantic features.

\paragraph{The SAMMD test meets adaptive attacks.}
In the case where the attacker is aware of our SAMMD kernel, we evaluate our SAMMD test from the security standpoint. Compared to other defenses, the advantage of our method in security is adaptive defense. In our detection mechanism, the semantic-aware deep kernel is trained on part of unknown data (to be tested), that is to say, for each input of data in the test, parameters of our semantic-aware deep kernel can be adaptively trained to be powerful. Therefore, the target of the adaptive attack can only be the SAMMD-G mentioned above which has fixed parameters. 

First, we use the PGD white-box attack to minimize the ${M(\omega)}/\sqrt{V_\lambda(\omega)}$ in Eq.~(\ref{eq:tpp-hat}) and obtain examples (kernel-attack in Figure~\ref{fig:nonIID}{g}), and $89.08\%$ of them can fool the pre-trained ResNet-18. Then, we obtain examples using the PGD white-box attack to minimize the ${M(\omega)}/\sqrt{V_\lambda(\omega)}$ in Eq.~(\ref{eq:tpp-hat}) and maximize the cross entropy loss in Eq.~(\ref{Eq:madry_inner_maximization}) (co-attack in Figure~\ref{fig:nonIID}{g}), and $61.34\%$ of them can fool the pre-trained ResNet-18. The examples acquired by model attack is the adversarial examples, $100.00\%$ of them can fool the pre-trained ResNet-18. Experiments in Figure~\ref{fig:nonIID}{g} show that these examples fail to fool our SAMMD test. And attacking such a statistic test will also reduce the ability of adversarial data to mislead a well-trained classifier. 
\paragraph{Ablation study.}
To illustrate the effectiveness of semantic features, we compare 
our SAMMD test with MMD-O test and MMD-D test after wild bootstrap process (MMD-O+WB, MMD-D+WB).  Experiments results are reported in Figure~\ref{fig:nonIID}{h}, which verifies that semantic features are better to help measure the distribution discrepancy between adversarial data and natural data than raw features (MMD-O+WB) and the learned features (MMD-D+WB) \cite{liu2020learning}.

\section{Conclusion}

Two-sample tests could in principle detect any distributional discrepancy between two datasets. However, previous studies have shown that the MMD test, as the most powerful two-sample test, is unaware of adversarial attacks. 
In this paper, we find that previous use of MMD on adversarial data detection missed three key factors, which \emph{significantly} limits its power. To this end, we propose a simple and effective test that is cooperated with a new semantic-aware kernel---\emph{semantic-aware MMD} (SAMMD) test, to take care of the three factors simultaneously. 
Experiments show that our SAMMD test can successfully detect adversarial attacks. 

Thus, we argue that \emph{MMD is aware of adversarial attacks}, which lights up a novel road for adversarial attack detection based on two-sample tests. 
We also recommend practitioners to use our SAMMD test when they wish to check whether the dataset they acquired contains adversarial data.


\section*{Acknowledgements}
RZG and BH were supported by HKBU Tier-1 Start-up Grant and HKBU CSD Start-up Grant. BH was also supported by the RGC Early Career Scheme No. 22200720, NSFC Young Scientists Fund No. 62006202 and HKBU CSD Departmental Incentive Grant. TLL was supported by Australian Research Council Project DE-190101473. JZ, GN and MS were supported by JST AIP Acceleration Research Grant Number JPMJCR20U3, Japan, and MS was also supported by Institute for AI and Beyond, UTokyo. FL would like to thank Dr. Yiliao Song and Dr. Wenkai Xu for productive discussions.


\bibliography{example_paper}
\bibliographystyle{icml2021}

\appendix
\onecolumn

\section{Related Works}
\label{sec:relationwork}

In this section, we briefly review related works used in our paper.

\subsection{Adversarial attacks}
A growing body of research shows that neural networks are vulnerable to adversarial attacks, i.e., test inputs that are modified slightly yet strategically to cause misclassification \citep{carlini2017adversarial,kurakin2016adversarial,wang2019convergence,zhang2020dual}, which seriously threaten the security-critical computer vision systems, such as autonomous driving and medical diagnostics \citep{chen2015deepdriving,ma2021understanding,nguyen2015deep,szegedy2013intriguing}. Thus, it is crucial to defend against adversarial attacks \citep{Chen_2020_CVPR,wang2020once,zhu2021understanding}, for example, by injecting adversarial examples into training data, adversarial training methods have been proposed in recent years \citep{Madry18PGD,bai2019hilbert,wang2019improving,zhang2020attacks}. However, these defenses can generally be evaded by \emph{optimization-based} (Opt) attacks, either wholly or partially \citep{carlini2017adversarial,he2018decision,li2014feature}

\subsection{Adversarial data detection}
For the adversarial defense, in addition to improving models' robustness by more effective adversarial training \cite{Chen_2020_CVPR,wang2019convergence,wu2020adversarial,zhang2020geometry},
recent studies have instead focused on detecting adversarial data. Based on features extracted from DNNs, most works train classifiers to discriminate adversarial 
data from both natural and adversarial data. Recent studies include, a cascade detector based on the PCA projection of activations \cite{li2017adversarial}, 
detection subnetworks based on activations \cite{metzen2017detecting}, a logistic regression detector based on Kernel Density KD, and Bayesian Uncertainty (BU) 
features \cite{grosse2017statistical}, an augmented neural network detector based on statistical measures, a learning framework that covers unexplored space 
invulnerable models \cite{rouhani2017curtail}, a \emph{local intrinsic dimensionality} (LID) based characterization of adversarial data \cite{ma2018characterizing}, 
a generative classifier based on Mahalanobis distance-based score \cite{lee2018simple}.

\subsection{Statistical adversarial data detection}
In the safety-critical system, it is important to find reliable data (i.e., natural data) and eliminate adversarial data that is statistically different from natural data distribution. Thus, statistical detection methods are also proposed to detect if the upcoming data contains adversarial data (or saying that if upcoming data is from natural data distribution in the view of statistics). 
A number of these methods have been introduced, including the use of the Maximum Mean Discrepancy (MMD) \cite{Gretton2012,borgwardt2006integrating} with a simple polynomial-time approximation to test whether the upcoming data are all adversarial data, or all natural images \cite{grosse2017statistical}, and a kernel density estimation defense used a Gaussian Mixture Model to model outputs from the final hidden layer of a neural network, to test whether the upcoming data belongs to a different distribution than that of natural data \cite{feinman2017detecting}. However, recent studies have shown these statistical detection failed to work under attack evaluations \cite{carlini2017adversarial}. 

\subsection{Two-sample Tests}
Two-sample tests aim to check whether two datasets come from the same distribution. Traditional tests such as $t$-test and Kolmogorov-Smirnov test are the mainstream of statistical applications, but require strong assumptions on the distributions being studied. Researchers in statistics and machine learning have been focusing on relaxing these assumptions,
with methods specific to various real-world domains
\cite{Sugiyama11NN,Yamada11NeurIPS,Kanamori12TIT,Gretton2012,Jitkrittum2016,sutherland:mmd-opt,Chen2017,Ghoshdastidar2017,Lopez:C2ST,LiW18TIT,Matthias:deep-test,liu2020learning}. In order to involve distributions with complex structure such as images, deep kernel approaches has been proposed \cite{sutherland:mmd-opt,Kevin_ICML2019,Jean2018}, the foremost study has shown that kernels parameterized by deep neural nets, can be trained to maximize test power in high-dimensional distribution such as images \cite{liu2020learning}.
They propose statistical tests of the null hypothesis that the two distributions are equal against the alternative hypothesis that the two distributions are different. Such tests have applications in a variety of machine learning problems such as domain adaptation, covariate shift, label-noise learning, generative modeling, fairness and causal discovery \cite{MMD_GAN,zhang2020a,fang2020rethinking,Gong2016,fang2020open,liu2019butterfly,zhang2020clarinet,zhang2020domain,liu2020TFS,zhong2021how,yu2020label,DA_app_Stojanov,Lopez:C2ST,ODLCMP20:fair-reps}.

\section{Real-world Scenarios regarding SADD}\label{Asec:examples}

\textbf{Scenario 1.} As an artificial-intelligence service provider, we need to acquire a client by modeling his/her task well, such as modeling the risk level of manufacturing factory. To finish this task, we need to hire distributed annotators to obtain labeled natural data regarding the risk level in the factory. However, our competitors may \textit{conspire} with several annotators against us, poisoning this training data by injecting malicious adversarial data ~\cite{barreno2010security,kloft2012security}. If training data contains adversarial ones, the test accuracy will drop~\cite{ZhangYJXGJ19TRADES}, which makes us lose the client unexpectedly. To beware of such adversarial attacks, we can use the MMD test to find reliable annotators providing natural training data. 

\textbf{Scenario 2.} As a client, we need to purchase artificial-intelligence services to model our task well, such as modeling the risk level of manufacturing factory mentioned above. Given a variety of models offered by many providers, we should select the optimal one and need to hire distributed annotators to obtain labeled natural data regarding the risk level in our factory. However, some artificial-intelligence service providers may \textit{conspire} with several our annotators, poisoning our testing data by injecting malicious adversarial data~\cite{barreno2010security,kloft2012security}. If the testing data contains adversarial ones which are only in the training set of those conspired providers, the test accuracy of conspired providers' models will surpass that of their competitors~\cite{Madry18PGD}, which makes us fail to select the optimal provider. To beware of such adversarial attacks, we can use the MMD test to find reliable annotators providing natural test data.

\begin{figure*}[tp]
    \begin{center}
        {\includegraphics[width=0.8\textwidth]{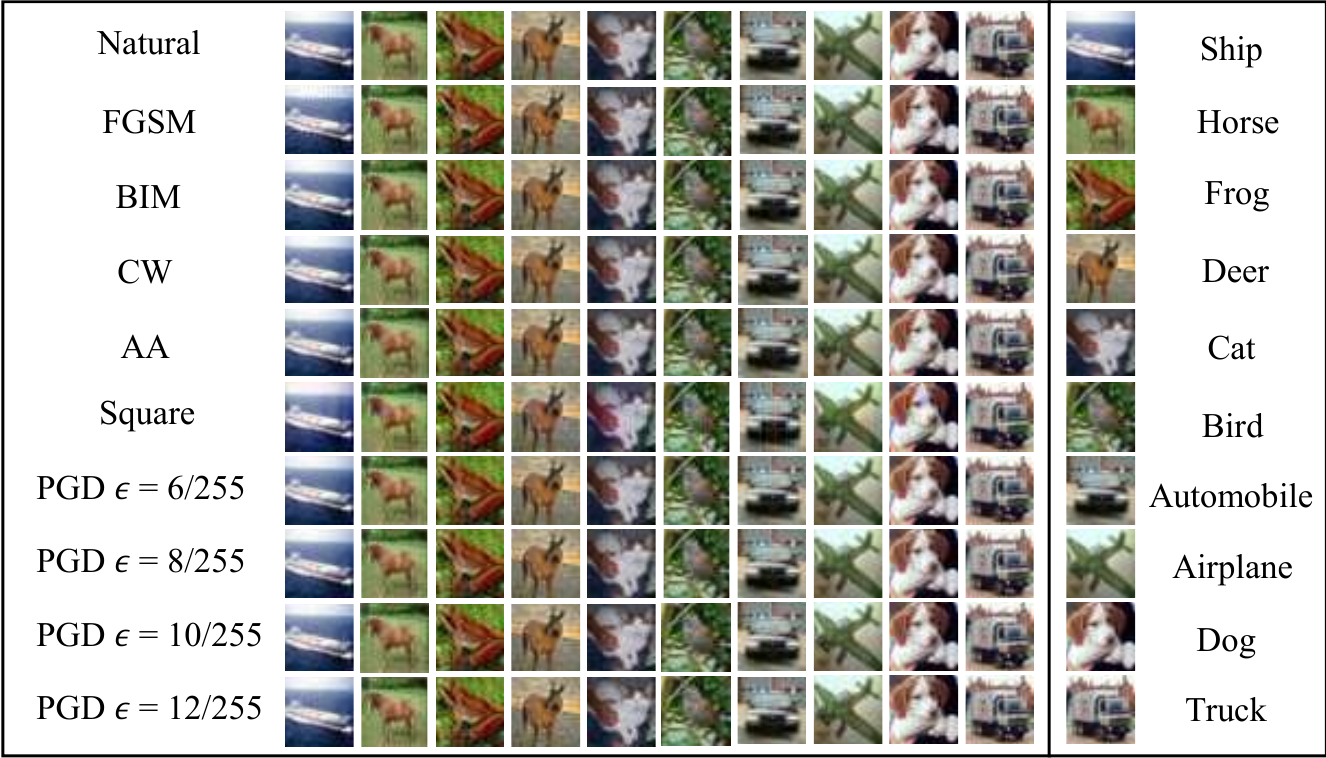}}\label{fig:moti:a}
        \caption{Adversarial data on \textit{CIFAR-10}. We output adversarial examples on a pre-trained ResNet-18, which are attacked by different methods or under different bounded perturbation $\epsilon$.
      }
    \label{fig:adv_data}
    \end{center}
\end{figure*}

\section{Hilbert-Schmidt Independence Criteria}\label{Asec: HSIC}
The HSIC \cite{gretton2008kernel,gretton2005measuring} is a 
test statistic to work on independence testing \cite{gretton2005measuring}. HSIC can be interpreted as the distance between embeddings of the joint distribution and the product of the marginals in a
RKHS. More importantly, HSIC between two random variables is zero if and only if the two variables are independent \cite{sriperumbudur2010hilbert}. Under the null hypothesis of independence, $P_{XY} = P_XP_Y$, the minimum variance estimate of HSIC is a degenerate U-statistic. The formulation of the HSIC is as follows (more details can be found in \cite{gretton2005measuring}). Given two sets of data $S_X$ and $S_Y$ (with size $n$), the HSIC can be computed using
\begin{equation}
	\begin{split}
	\textnormal{HSIC}(S_X,S_Y)
	=& ~\E_{(x_1,y_1)\sim p(x,y),(x_2,y_2)\sim p(x,y)}[\kappa_X(x_1,x_2)\kappa_Y(y_1,y_2)] \\
	& +\E_{x_1\sim p(x),x_2\sim p(x),y_1\sim p(y),y_2\sim p(y)}[\kappa_X(x_1,x_2)\kappa_Y(y_1,y_2)]\\
	& -2\E_{(x_1,y_1)\sim p(x,y),x_2\sim p(x),y_2\sim p(y)}[\kappa_X(x_1,x_2)\kappa_Y(y_1,y_2)],
	\end{split}
\end{equation}
where $\kappa_X$ and $\kappa_Y$ are two Gaussian kernel functions whose bandwidths are set to two constants, and
\begin{align}
\E_{(x_1,y_1)\sim p(x,y),(x_2,y_2)\sim p(x,y)}[\kappa_X(x_1,x_2)\kappa_Y(y_1,y_2)] = \frac{1}{n^2}\sum_{i=1}^n\sum_{j=1}^n[\kappa_X(x_i,x_j)\kappa_Y(y_i,y_j)].
\end{align}

\section{Asymptotics of the SAMMD}
\label{Asec:asym_MMD}

In this section, we will first prove the asymptotics of the SAMMD by assuming that the adversarial data $\{Y_i\}_{i\in\Z^+}$ are an absolutely regular
process with mixing coefficients
$\{\beta_k\}_{k>0}$ defined in the following.
\begin{definition}[Absolutely regular
process]
(i) Let $(\Omega,\mathcal{A},\Q)$ be a probability space, and let $\mathcal{A}_1,\mathcal{A}_2$ be sub-$\sigma$-field of $\mathcal{A}$. We define
\begin{align}
\label{eq:def_beta}
    \beta(\mathcal{A}_1,\mathcal{A}_2) = \sup_{A_1,\dots,A_n,B_1,\dots,B_m}\sum_{i=1}^n\sum_{j=1}^m|\Q(A_i\cap B_j)-\Q(A_i)\Q(B_j)|,
\end{align}
where the supremum is taken over all partitions $A_1,\dots,A_n$ and $B_1,\dots,B_m$ of $\Omega$ into elements of $\mathcal{A}_1$ and $\mathcal{A}_2$, respectively.
\newline
(ii) Given a stochastic process $\{Y_i\}_{i\in\Z^+}$ and integers $1\leq a\leq b$, we denote by $\mathcal{A}_a^b$ the $\sigma$-field generated by the random variables $Y_{a+1},\dots,Y_{b}$. We define the mixing coefficients of absolute regularity by 
\begin{align}
    \beta_k = \sup_{n\in{\Z^+}}\beta(\mathcal{A}_{1}^n,\mathcal{A}_{n+k}^{\infty}).
\end{align}
The process $\{Y_i\}_{i\in\Z^+}$ is called absolutely regular if $\lim_{k\rightarrow \infty}\beta_k=0$.
\end{definition}
Then, we can obtain the main theorem in the following.

\begin{theorem}[Asymptotics under $H_1$] \label{prop:asymptotics_full}
Under the alternative, $H_1 :$ $S_Y$ are from a stochastic progress $\{Y_i\}_{i=1}^{+\infty}$, if $\{Y_i\}_{i=1}^{+\infty}$ is an absolutely regular
process with mixing coefficients
$\{\beta_k\}_{k>0}$ satisfying $\sum_{k=1}^{+\infty}\beta_k^{\delta/(2+\delta)}<+\infty$ for some $\delta>0$, then
$\widehat{\textnormal{SAMMD}}_u^2$ is $\bigO_P(1 / n)$,
and in particular
\begin{gather*}
\sqrt{n}(\widehat{\textnormal{SAMMD}}_u^2 - \textnormal{SAMMD}^2) \overset{\textbf{d}}{\to} \N(0, C_1^2\sigma^2_{H_1}),
\end{gather*}
where $Y_i=\mathcal{G}_{\ell,\hat{f}}(\epsball[X_i^{\prime\prime}])\sim\Q$, $X_i^{\prime\prime}\sim \P$, $\sigma_{H_1}^2 = 4(\E_{Z}[(\E_{Z'}h(Z,Z'))^2]-[(\E_{Z,Z'}h(Z,Z'))^2])$, $h(Z,Z')$ $=k_{\omega}(X,X')+k_{\omega}(Y,Y')-k_{\omega}(X,Y')-k_{\omega}(X',Y)$, $Z:=(X,Y)$, $X\sim\P$ and $X^{\prime\prime}$ are independent and $C_1<+\infty$ is a constant for a given ${\omega}$.
\end{theorem}
\begin{proof}
Without loss of generality, let $Z$ be a random variable on a probability space $(\Omega^Z,\mathcal{A}^Z,\Q^Z)$. We will first prove that $\{Z\}_{i=1}^{+\infty}$ is an absolutely regular process. According to Eq.~(\ref{eq:def_beta}), we have
\begin{align}
    \beta^Z(\mathcal{A}_1^Z,\mathcal{A}_2^Z) = \sup_{A_1^Z,\dots,A_n^Z,B_1^Z,\dots,B_m^Z}\sum_{i=1}^n\sum_{j=1}^m|\Q^Z(A_i^Z\cap B_j^Z)-\Q^Z(A_i^Z)\Q^Z(B_j^Z)|,
\end{align}
where $\mathcal{A}_1^Z,\mathcal{A}_2^Z$ are sub-$\sigma$-field of $\mathcal{A}^Z$ generated by $\{Z\}_{i=1}^{+\infty}$ and the supremum is taken over all partitions $A_1^Z,\dots,A_n^Z$ and $B_1^Z,\dots,B_m^Z$ of $\Omega$ into elements of $\mathcal{A}_1^Z$ and $\mathcal{A}_2^Z$, respectively. Since $X$ and $Y$ are independent and $Z = (X,Y)$, $\Q^Z(Z\in A^Z) = \Q^Z(X\in A^X,Y\in A) = \P(A^X)\Q(A)$. Thus, we have $\Q^Z(A_i^Z\cap B_j^Z) = \P(A_i^X\cap B_i^X)\Q(A_i\cap B_i)$, $\Q^Z(A_i^Z) = \P(A_i^X)\Q(A_i)$, and $\Q^Z(B_i^Z) = \P(B_i^X)\Q(B_i)$. Since $X$ and $X^\prime$ are independent, we have $\P(A_i^X\cap B_i^X) = \P(A_i^X)\P(B_i^X)$, meaning that
\begin{align}
    \beta^Z(\mathcal{A}_1^Z,\mathcal{A}_2^Z) = \sup_{A_1^Z,\dots,A_n^Z,B_1^Z,\dots,B_m^Z}\sum_{i=1}^n\sum_{j=1}^m\P(A_i^X\cap B_i^X)|\Q(A_i\cap B_j)-\Q(A_i)\Q(B_j)|.
\end{align}
Due to the supremum, we can safely make $\P(A_i^X\cap B_i^X)$ be $1$. Thus, we have $\beta^Z(\mathcal{A}_1^Z,\mathcal{A}_2^Z)=\beta(\mathcal{A}_1,\mathcal{A}_2)$. Namely, $\{Z\}_{i=1}^{+\infty}$ is an absolutely regular process with mixing coefficients
$\{\beta_k\}_{k>0}$ satisfying $\sum_{k=1}^{+\infty}\beta_k^{\delta/(2+\delta)}<+\infty$. Based on Theorem $1$ in \cite{denker1983u}, since $h(\cdot,\cdot)\leq 2$, we know that
\begin{align}
    \sqrt{n}(\widehat{\textnormal{SAMMD}}_u^2 - \textnormal{SAMMD}^2) \overset{\textbf{d}}{\to} \N(0, 4\sigma^2),
\end{align}
where
\begin{align}
\label{eq:variance}
     \sigma^2 = \underbrace{\E[h_1(Z_1)]^2}_{\sigma^2_{H_1}} + 2\sum_{j=1}^{+\infty}\textnormal{cov}(h_1(Z_1),h_1(Z_j)),
\end{align}
$h_1(Z_j) = \E_{Z_i}h(Z_i,Z_j)-\theta$, and $\theta = \E_{Z_i,Z_j}h(Z_i,Z_j)$. Note that, due to $\P\neq\Q$, we know $\sigma>0$; due to the absolute regularity, $\sigma<+\infty$. Since the possible dependence between $Z_1$ and $Z_j$ are caused by $Y_1$ and $Y_j$, we will calculate the second term in the right side of Eq.~(\ref{eq:variance}) in the following. First, we introduce two notations for the convenience. 
\begin{align}
    \E_X^{(i)} = \E_X[k_\omega(X_i,X)-k_\omega(Y_i,X)], 
\end{align}
\begin{align}
\label{eq:EYi}
    \E_Y^{(i)} = \E_Y[k_\omega(X_i,Y)-k_\omega(Y_i,Y)].
\end{align}
Thus, we know
\begin{align}
    h_1(Z_1) = \underbrace{\E_X^{(1)}+\E_Y^{(1)}}_{\tilde{h}_1(Z_1)}-\theta,~h_1(Z_j) = \underbrace{\E_X^{(j)}+\E_Y^{(j)}}_{\tilde{h}_1(Z_j)}-\theta,~\theta = \E_{Z_1}[\E_X^{(1)}+\E_Y^{(1)}],
\end{align}
and 
\begin{align}
\label{eq:theta2}
    \theta^2 = \E_{Z_1}[\E_X^{(1)}+\E_Y^{(1)}]\E_{Z_j}[\E_X^{(j)}+\E_Y^{(j)}]=\Big(\E_{Z_1}[\E_X^{(1)}]+\E_{Z_1}[\E_Y^{(1)}]\Big)\Big(\E_{Z_j}[\E_X^{(j)}]+\E_{Z_j}[\E_Y^{(j)}]\Big).
\end{align}
Then, we can compute the $\textnormal{cov}(h_1(Z_1),h_1(Z_j))$.
\begin{align}
\label{eq:cov}
    \textnormal{cov}(h_1(Z_1),h_1(Z_j))
    =&\E_{Z_1,Z_j}[(\tilde{h}_1(Z_1)-\theta)(\tilde{h}_1(Z_j)-\theta)] \nonumber \\
    =&\E_{Z_1,Z_j}[\tilde{h}_1(Z_1)\tilde{h}_1(Z_j)-\theta \tilde{h}_1(Z_j)-\theta \tilde{h}_1(Z_1)+\theta^2] \nonumber \\
    =&\E_{Z_1,Z_j}[\tilde{h}_1(Z_1)\tilde{h}_1(Z_j)]-\theta^2 \nonumber \\
    =&\E_{Z_1,Z_j}[(\E_X^{(1)}+\E_Y^{(1)})(\E_X^{(j)}+\E_Y^{(j)})]-\theta^2 \nonumber \\
    =&\E_{Z_1,Z_j}[\E_X^{(1)}\E_X^{(j)}+\E_X^{(1)}\E_Y^{(j)}+\E_Y^{(1)}\E_X^{(j)}+\E_Y^{(1)}\E_Y^{(j)}]-\theta^2 \nonumber \\
    =&\E_{Z_1}[\E_X^{(1)}]\E_{Z_j}[\E_X^{(j)}]+\E_{Z_1}[\E_X^{(1)}]\E_{Z_j}[\E_Y^{(j)}]+\E_{Z_1}[\E_Y^{(1)}]\E_{Z_j}[\E_X^{(j)}]+\E_{Z_1,Z_j}[\E_Y^{(1)}\E_Y^{(j)}]-\theta^2.
\end{align}
Substituting Eq.~(\ref{eq:theta2}) into  Eq.~(\ref{eq:cov}), we have
\begin{align}
\label{eq:cov1}
    \textnormal{cov}(h_1(Z_1),h_1(Z_j))
    =\E_{Z_1,Z_j}[\E_Y^{(1)}\E_Y^{(j)}]-\E_{Z_1}[\E_Y^{(1)}]\E_{Z_j}[\E_Y^{(j)}].
\end{align}
Then, substituting Eq.~(\ref{eq:EYi}) into  Eq.~(\ref{eq:cov1}), we have
\begin{align}
\label{eq:cov2}
    \textnormal{cov}(h_1(Z_1),h_1(Z_j))
    =&\E_{Y_1,Y_j}\big[\E_Y\E_Y[k_\omega(Y_1,Y)k_\omega(Y_j,Y)]\big] -\E_{Y_1}\big[\E_Y[k_\omega(Y_1,Y)]\big]\E_{Y_j}\big[\E_Y[k_\omega(Y_j,Y)]\big] \nonumber \\
    =&\E_Y\E_Y\big[\E_{Y_1,Y_j}[k_\omega(Y_1,Y)k_\omega(Y_j,Y)]-\E_{Y_1}[k_\omega(Y_1,Y)]\E_{Y_j}[k_\omega(Y_j,Y)]\big].
\end{align}
Since $k_\omega(\cdot,\cdot)\leq 1$, according to Lemma $1$ in \cite{yoshihara1976limiting}, we have $\textnormal{cov}(h_1(Z_1),h_1(Z_j))<4\beta_j^{\delta/(2+\delta)}$. Because $\sum_{k=1}^{+\infty}\beta_k^{\delta/(2+\delta)}<+\infty$, we know, $\forall \epsilon'\in(0,1)$, there exists an $N$ such that $\sum_{k=N+1}^{+\infty}\beta_k^{\delta/(2+\delta)}<\epsilon'$. Hence
\begin{align}
    \sum_{j=1}^{+\infty}\textnormal{cov}(h_1(Z_1),h_1(Z_j))=\sum_{j=1}^{N}\E_Y\E_Y\big[\E_{Y_1,Y_j}[k_\omega(Y_1,Y)k_\omega(Y_j,Y)]-\E_{Y_1}[k_\omega(Y_1,Y)]\E_{Y_j}[k_\omega(Y_j,Y)]\big] + c',
\end{align}
where $c'$ is a small constant. Without loss of generality, we assume the small constant $c'$ is smaller than $\E[h_1(Z_1)]^2$. Thus, there exists a constant $C_1^2-1$ such that $2\sum_{j=1}^{+\infty}\textnormal{cov}(h_1(Z_1),h_1(Z_j))=(C_1^2-1)\E[h_1(Z_1)]^2$. Namely, $\sigma^2=C_1^2\sigma^2_{H_1}$, which completes the proof.
\end{proof}

Next, we will show that the bootstrapped SAMMD (shown in the following) has the same asymptotic null distribution as the empirical SAMMD. First, we restate the bootstrapped SAMMD in the following.
\begin{align}
    &\widehat{\textnormal{SAMMD}}_w(S_X,S_Y;k_\omega)\nonumber \\
    =~&\frac{1}{n^2}\sum_{i=1}^n\sum_{j=1}^n\tilde{W}_i^{x}\tilde{W}_j^{x}k_\omega(x_i,x_j) + \frac{1}{m^2}\sum_{i=1}^m\sum_{j=1}^m\tilde{W}_i^{y}\tilde{W}_j^{y}k_\omega(y_i,y_j) - \frac{2}{nm}\sum_{i=1}^n\sum_{j=1}^m\tilde{W}_i^{x}\tilde{W}_j^{y}k_\omega(x_i,y_j),
\end{align}
where 
\begin{align}
    \{\tilde{W}^X_i\}_{i=1}^n= \{{W}^X_i\}_{i=1}^n - \frac{1}{n}\sum_{i=1}^n{W}^X_i,~\{\tilde{W}^Y_i\}_{i=1}^m =\{{W}^Y_i\}_{i=1}^m - \frac{1}{m}\sum_{i=1}^m{W}^Y_i.
\end{align}
The ${W}^X_i$ and ${W}^Y_j$ are generated by
\begin{align}
    W_t = e^{-1/l}W_{t-1} + \sqrt{1-e^{-2/l}}\epsilon_t,
\end{align}
where $W_0, \epsilon_0,\dots,\epsilon_t$ are independent standard normal random variables. Then, following the Proposition $1$ in \cite{Kacper14wildbtp}, we can directly obtain the following proposition using the relation between $\beta$-mixing and $\tau$-mixing presented in Eq. (18) in \cite{Kacper14wildbtp}.
\begin{prop}
Let $\{Y_i\}_{i=1}^{+\infty}$ be an absolutely regular process with mixing coefficients
$\{\beta_k\}_{k>0}$ satisfying $\beta_k=O(k^{(-6-\epsilon'')(1+\delta)})$ for some $\epsilon>0$ and $\delta>0$, $n = \rho_xn'$ and $m = \rho_yn'$, where $n'=n+m$. Then, under the null hypothesis $Y_i\sim\P$, $\psi(\rho_x\rho_yn'\widehat{\textnormal{SAMMD}}_w(S_X,S_Y;k_\omega),\rho_x\rho_yn'\widehat{\textnormal{SAMMD}}(S_X,S_Y;k_\omega))\rightarrow 0$ in probability as $n'\rightarrow +\infty$, where $\psi$ is the Prokhorov metric.
\end{prop}




\begin{figure*}[tp]
    \begin{center}
        \subfigure[RN18-Natural]
        {\includegraphics[width=0.135\textwidth]{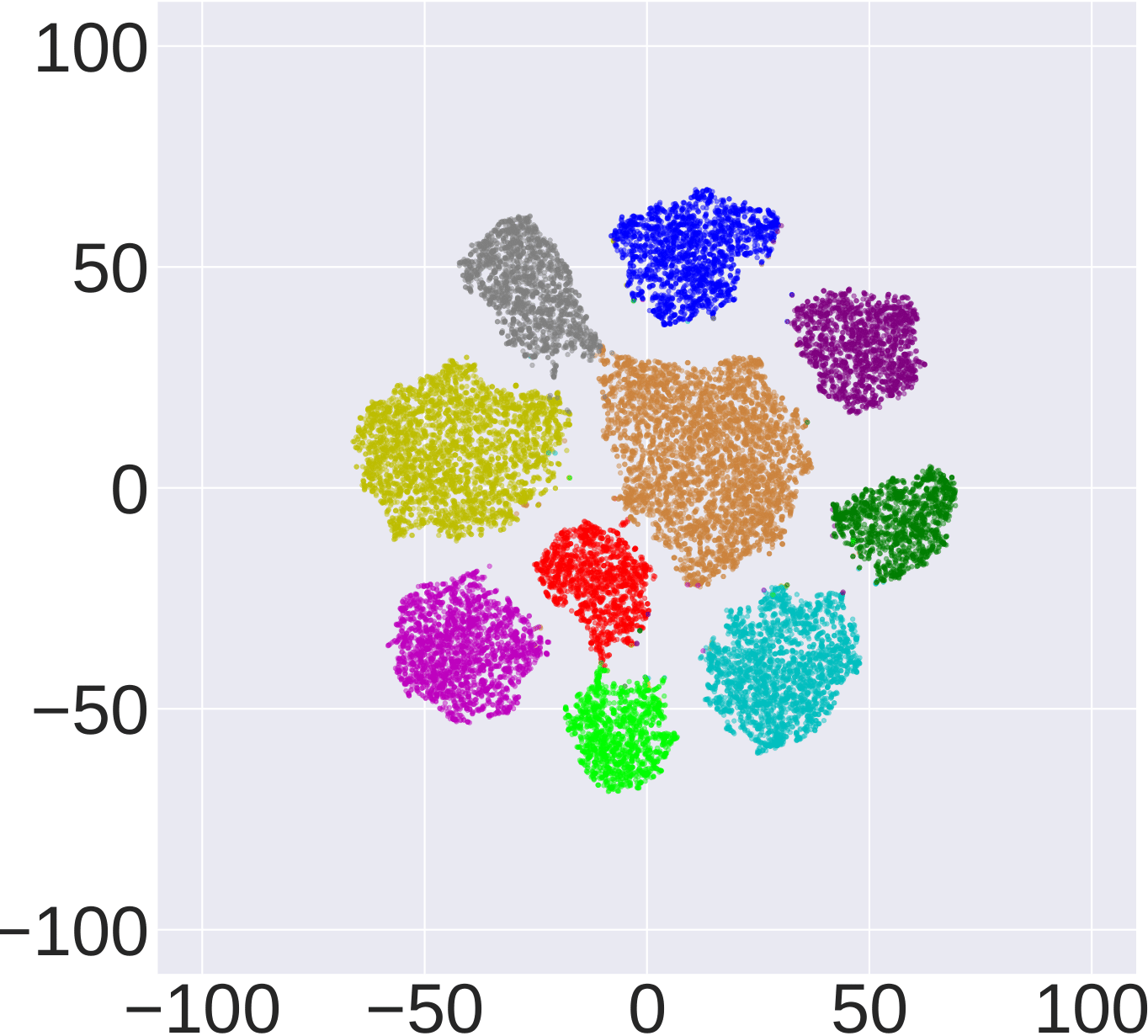}}
        \subfigure[RN18-FGSM]
        {\includegraphics[width=0.135\textwidth]{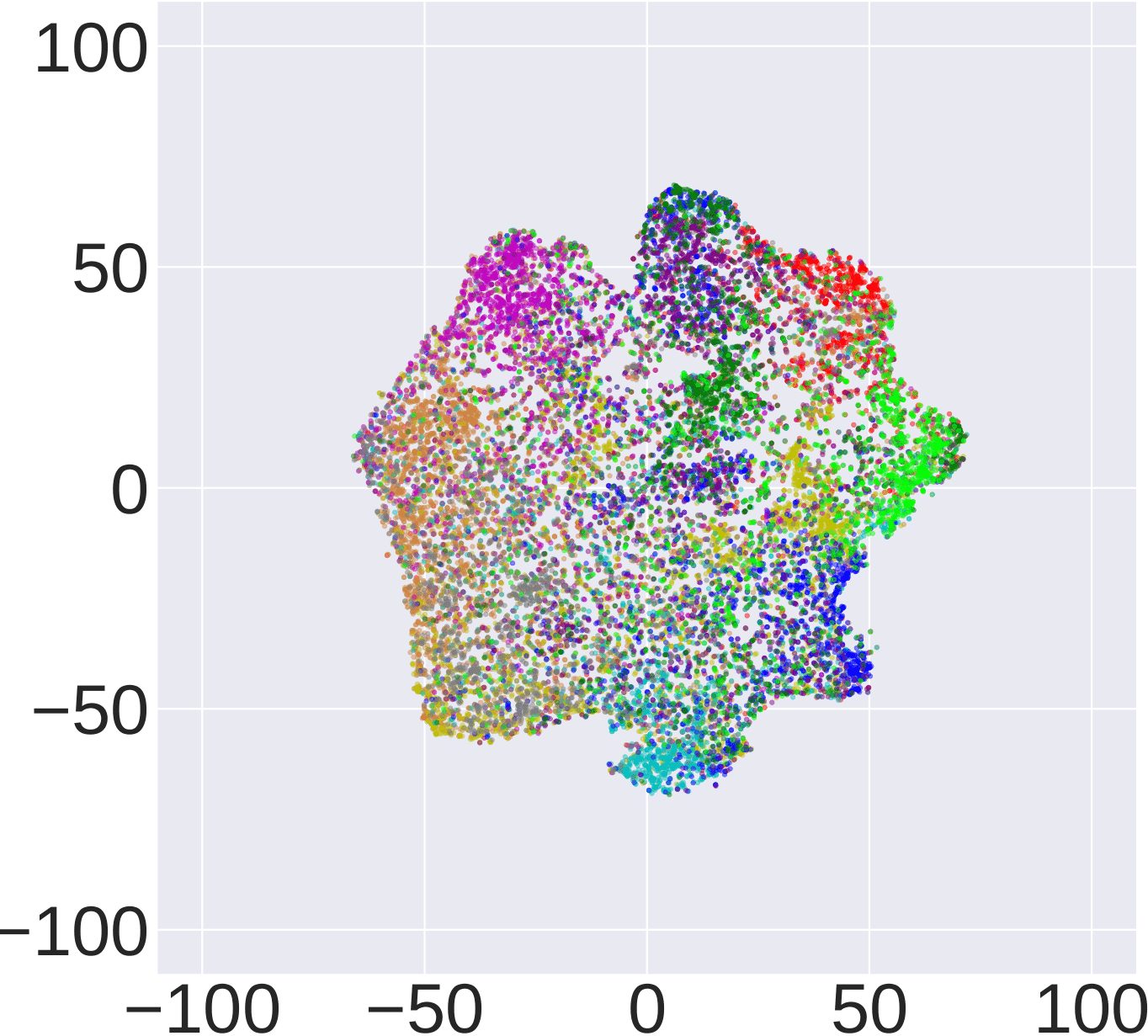}}
        \subfigure[RN18-BIM]
        {\includegraphics[width=0.135\textwidth]{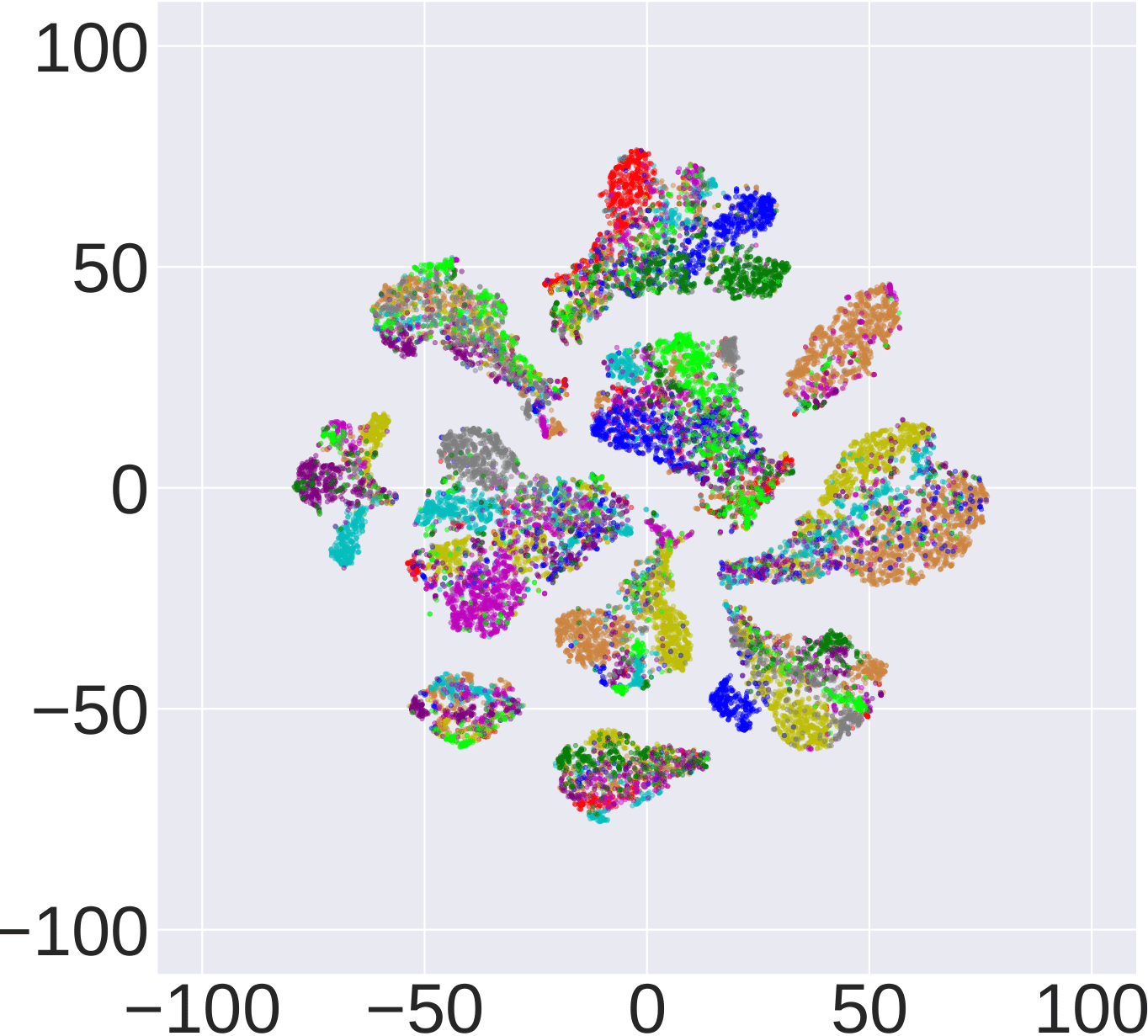}}
        \subfigure[RN18-PGD]
        {\includegraphics[width=0.135\textwidth]{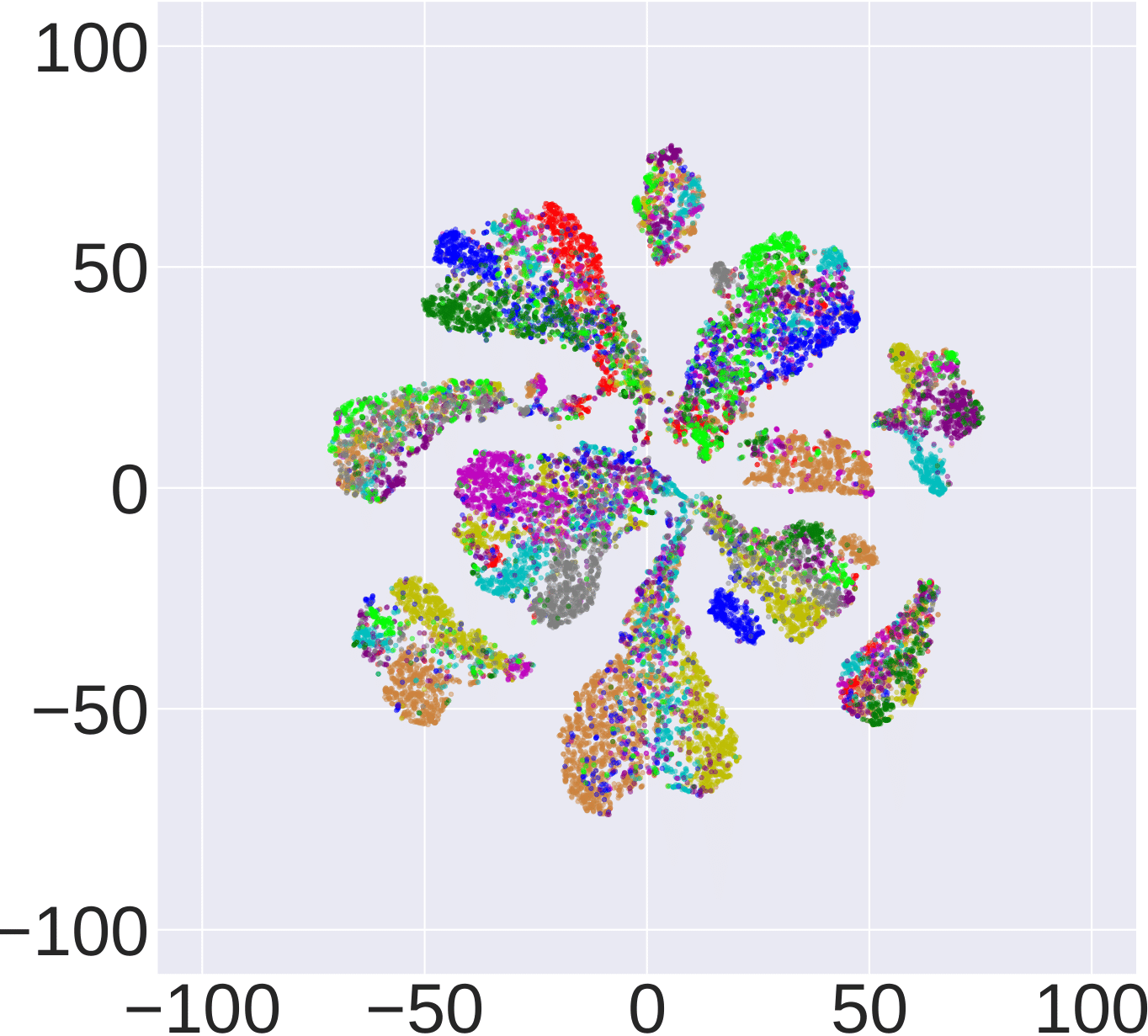}}
        \subfigure[RN18-CW]
        {\includegraphics[width=0.135\textwidth]{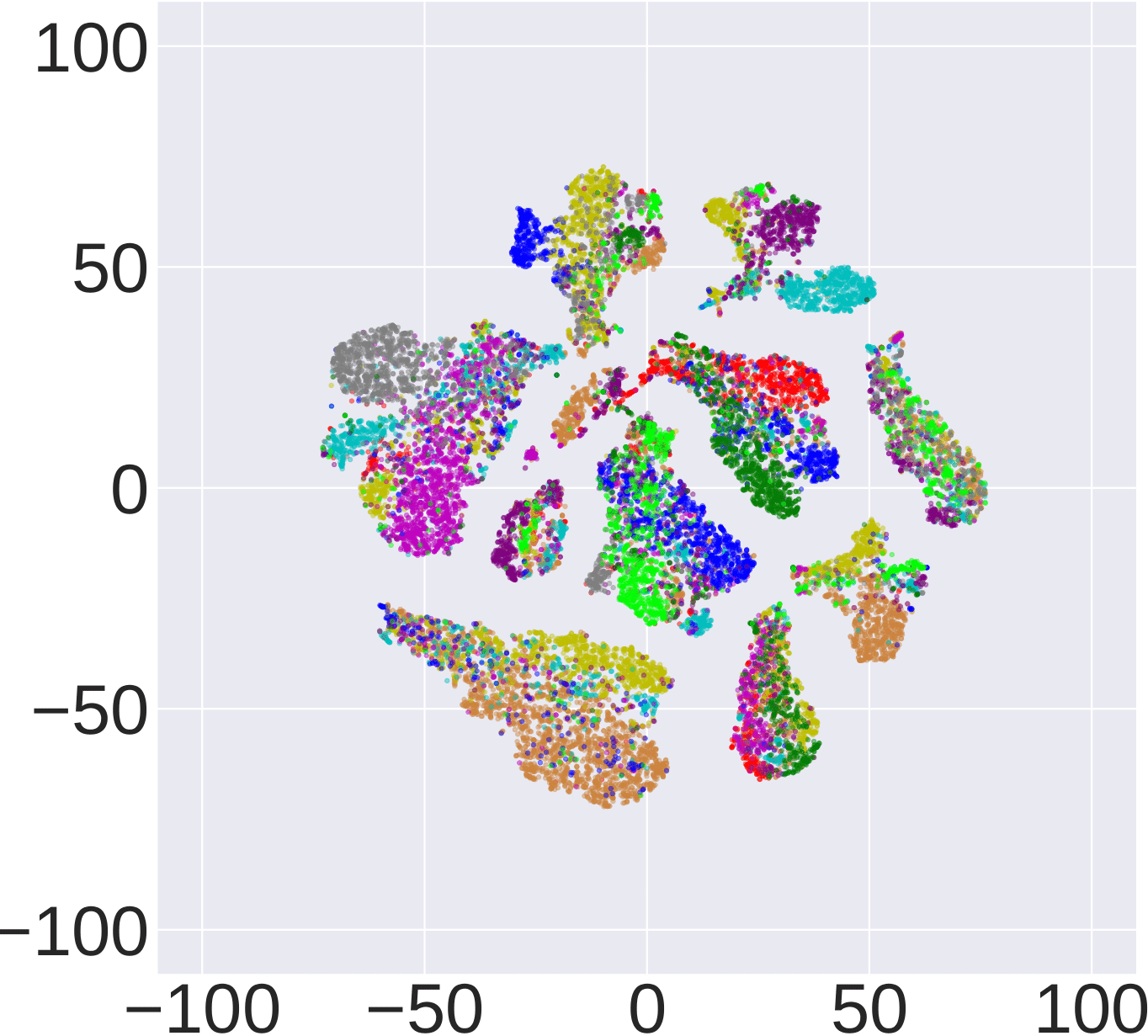}}
        \subfigure[RN18-AA]
        {\includegraphics[width=0.135\textwidth]{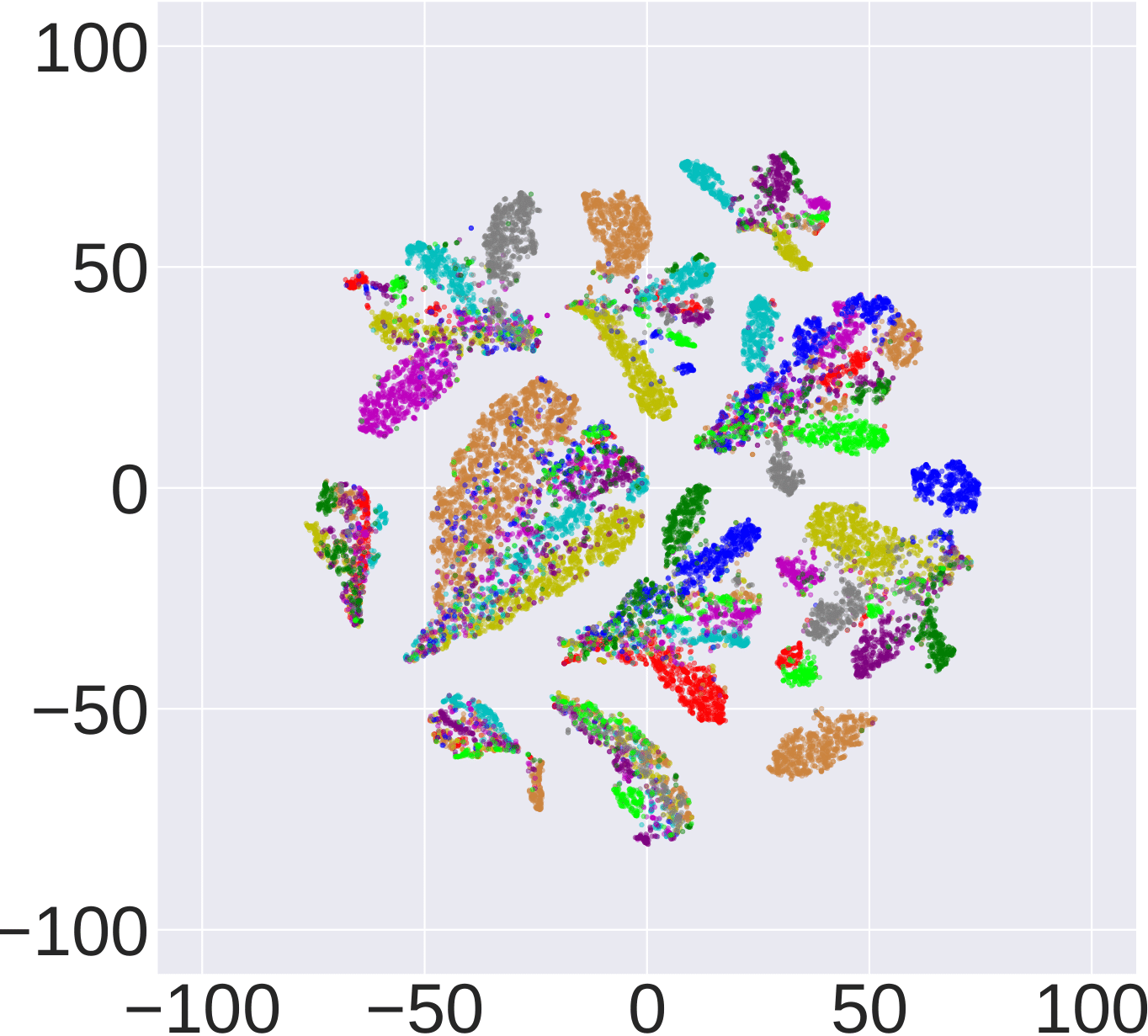}}
        \subfigure[RN18-Square]
        {\includegraphics[width=0.135\textwidth]{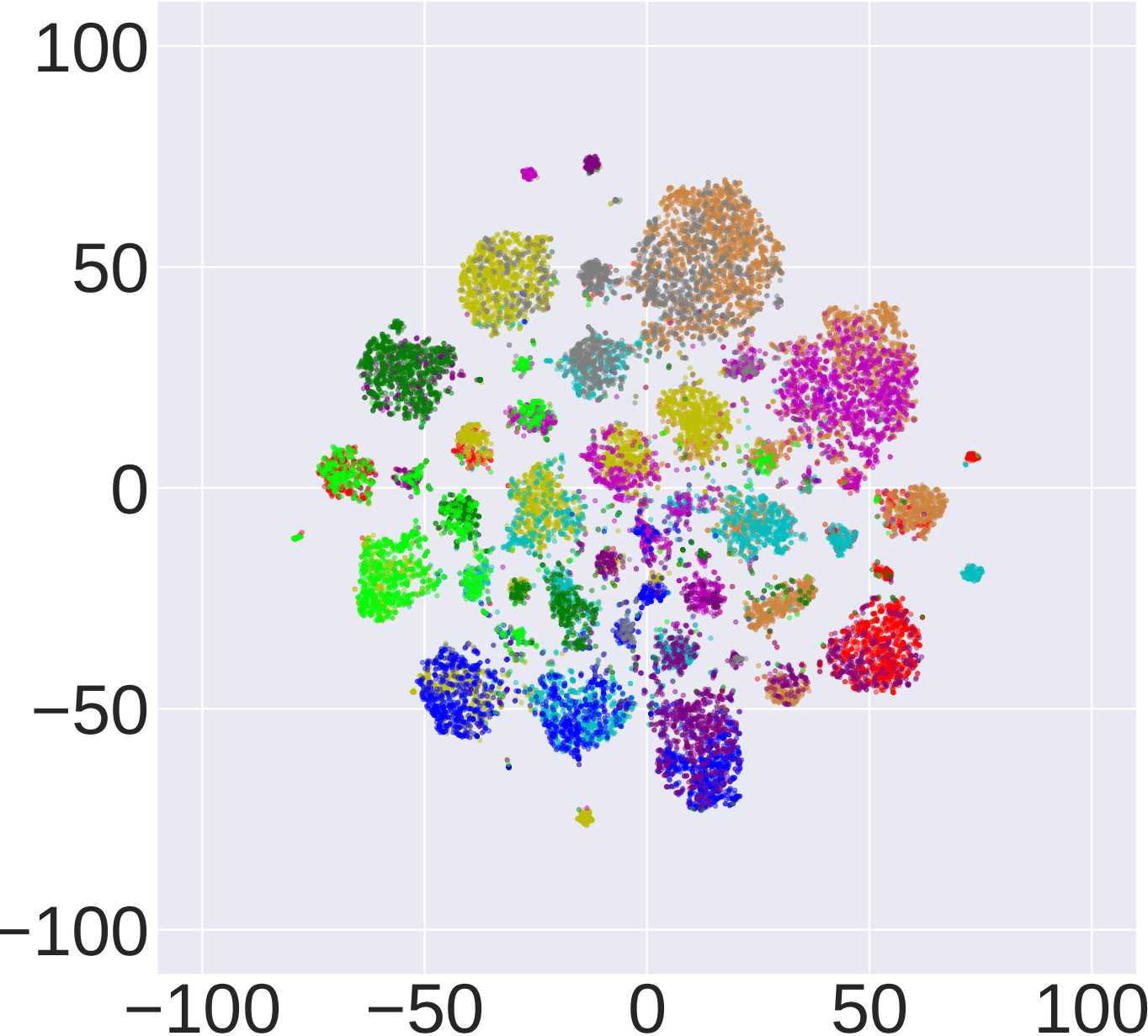}}
        \subfigure[RN34-Natural]
        {\includegraphics[width=0.135\textwidth]{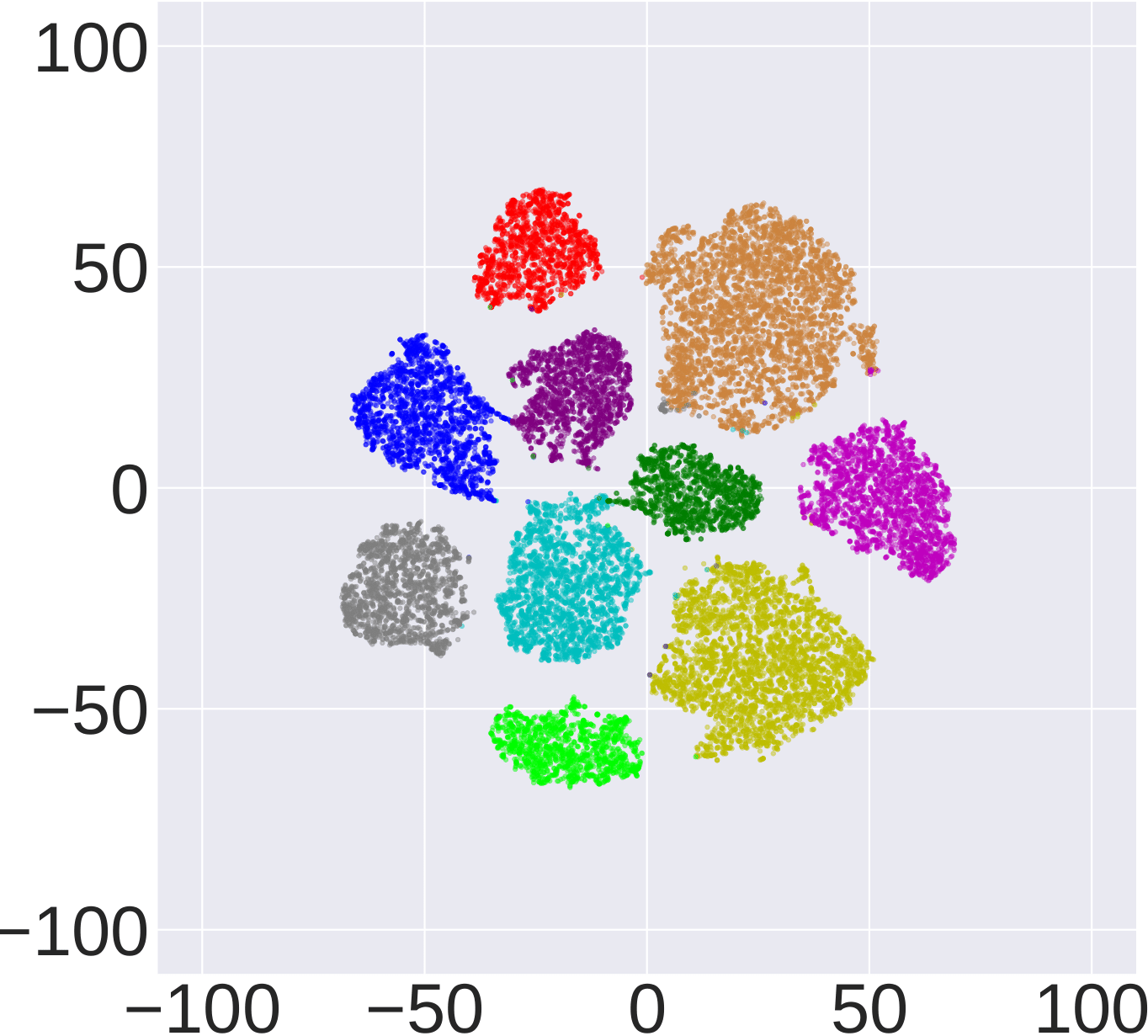}}
        \subfigure[RN34-FGSM]
        {\includegraphics[width=0.135\textwidth]{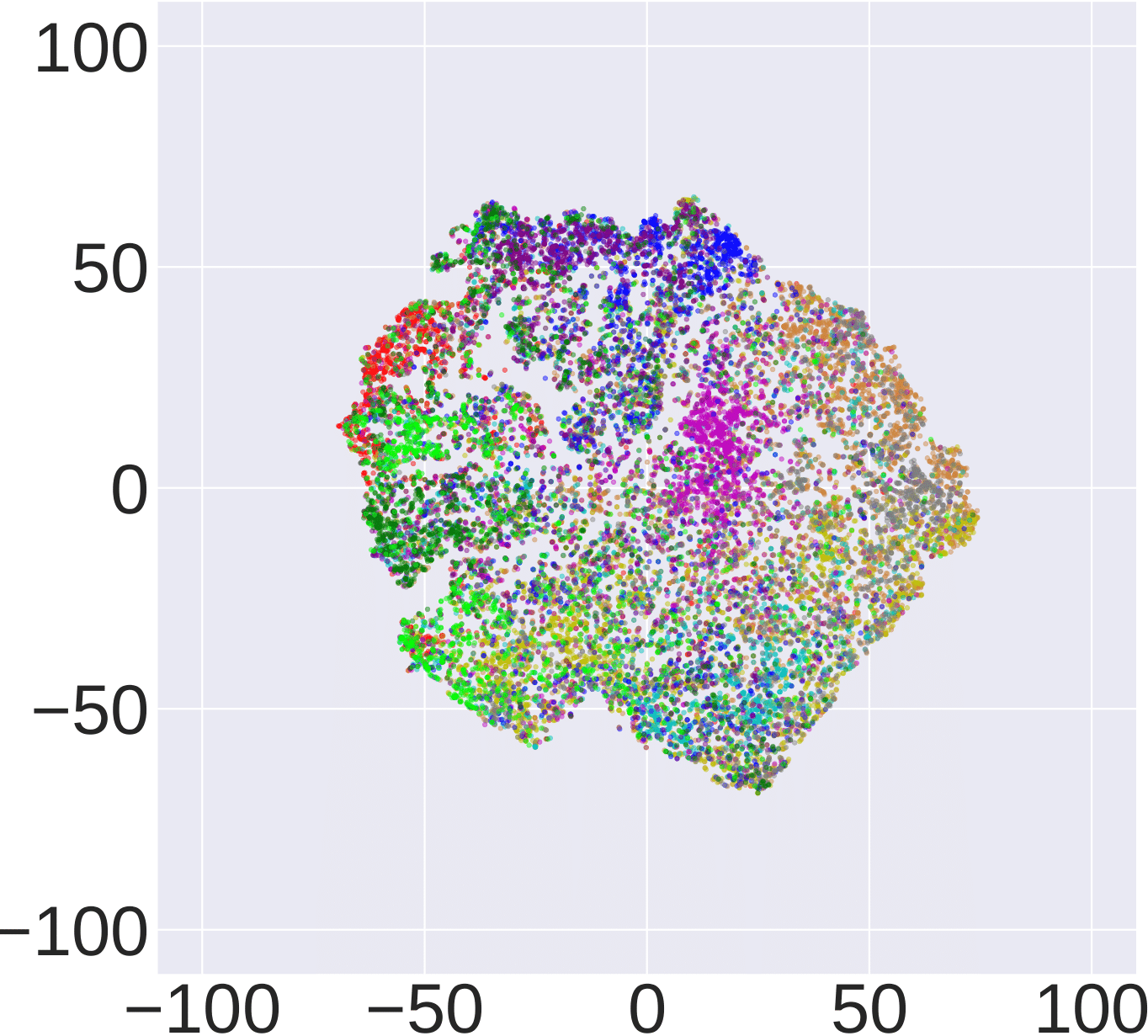}}
        \subfigure[RN34-BIM]
        {\includegraphics[width=0.135\textwidth]{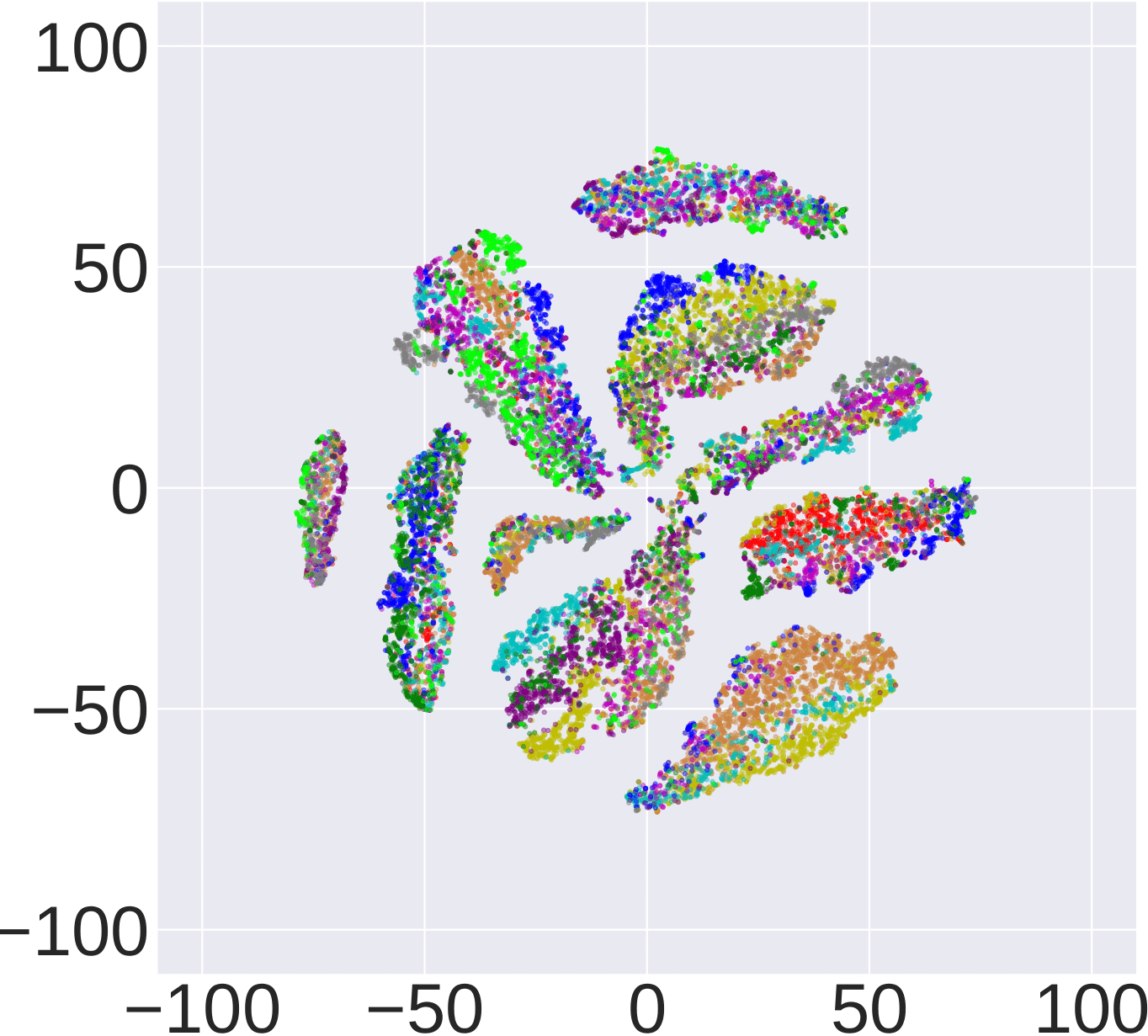}}
        \subfigure[RN34-PGD]
        {\includegraphics[width=0.135\textwidth]{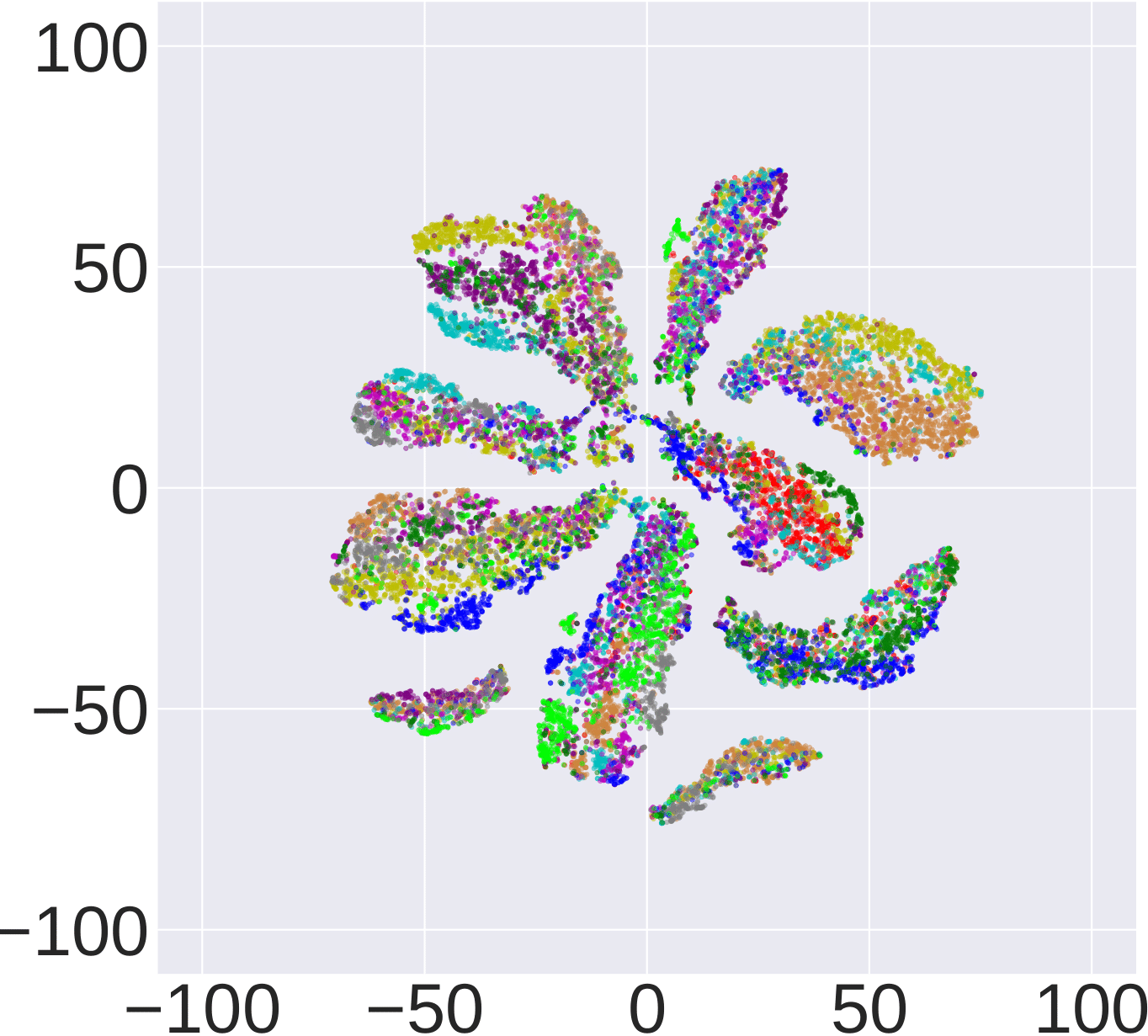}}
        \subfigure[RN34-CW]
        {\includegraphics[width=0.135\textwidth]{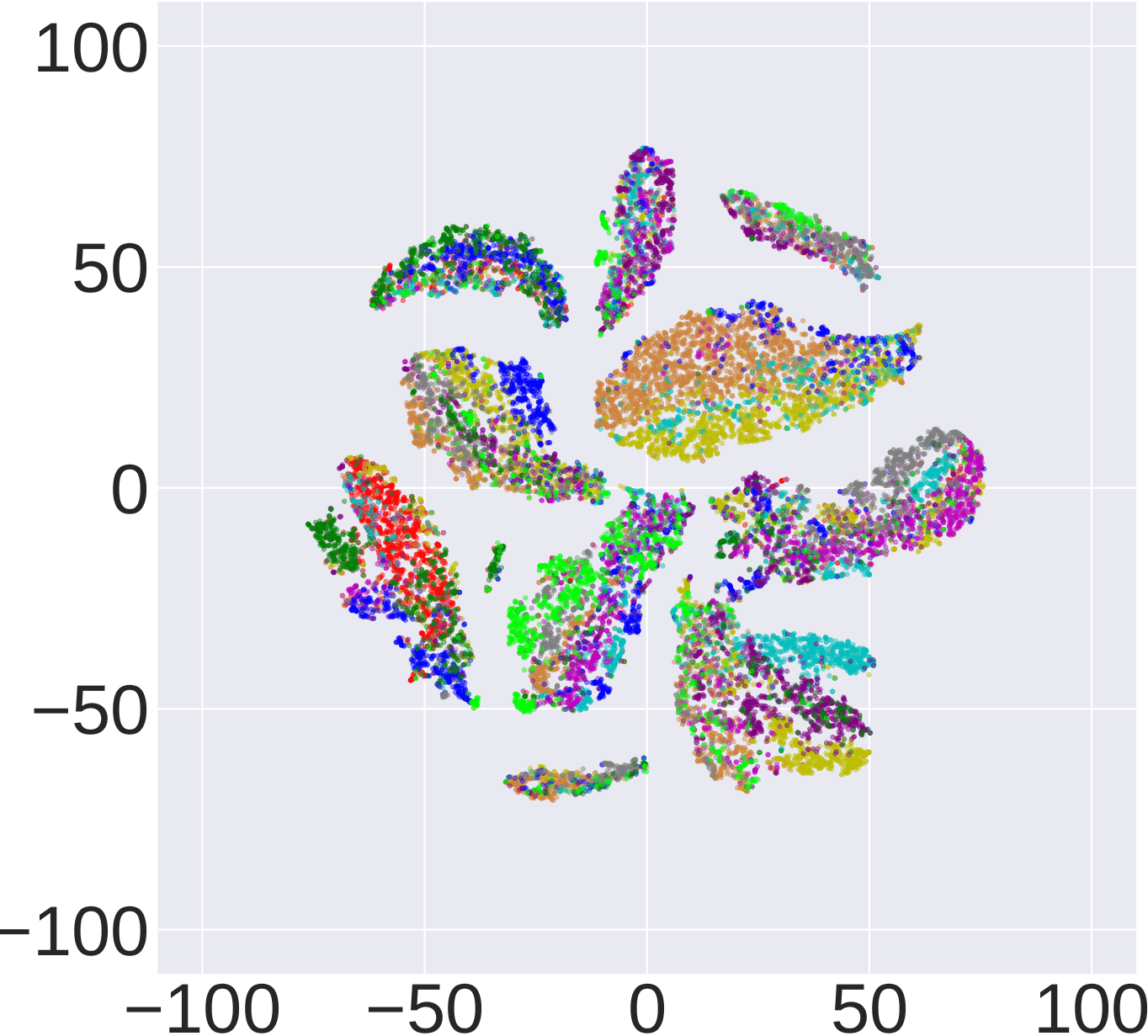}}
        \subfigure[RN34-AA]
        {\includegraphics[width=0.135\textwidth]{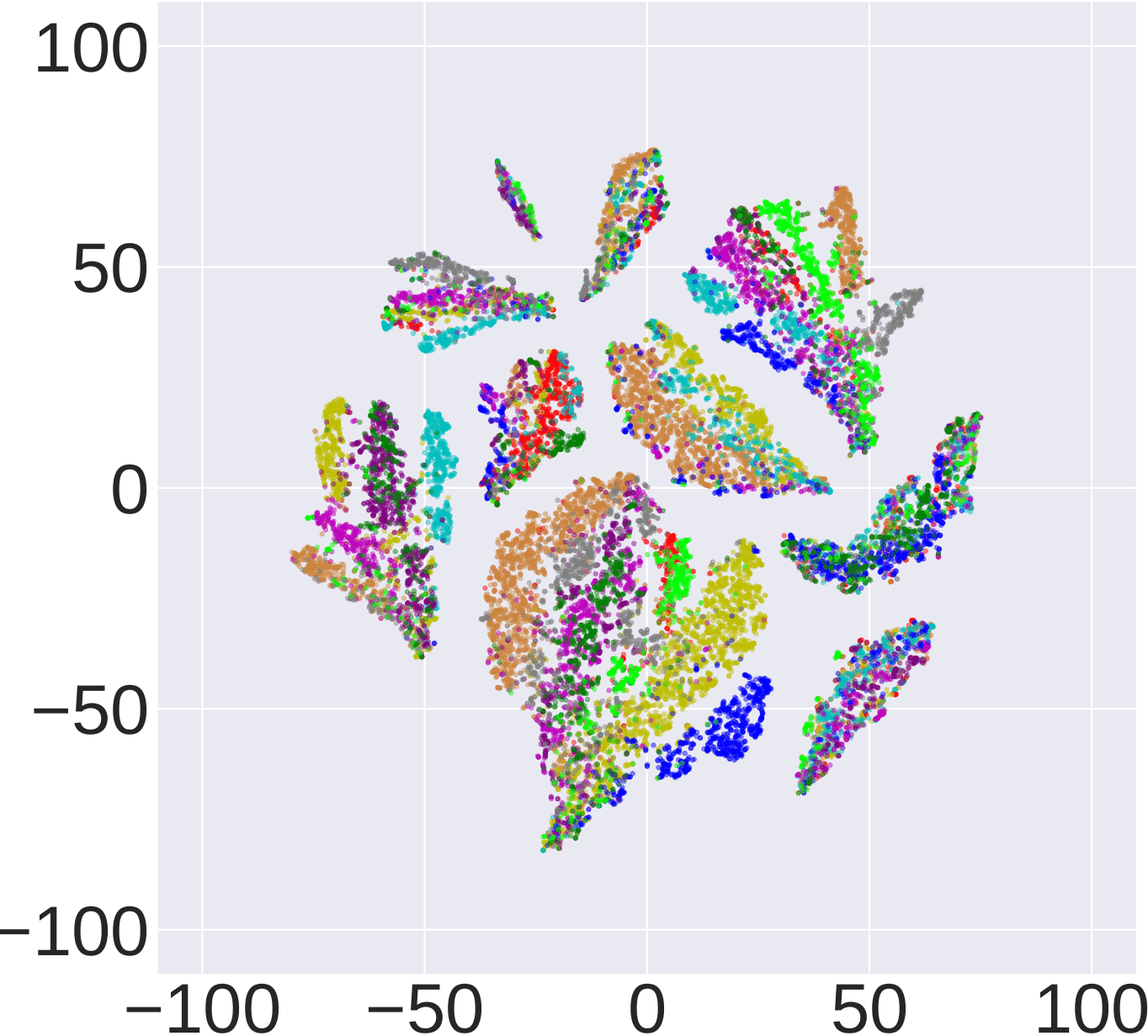}}
        \subfigure[RN34-Square]
        {\includegraphics[width=0.135\textwidth]{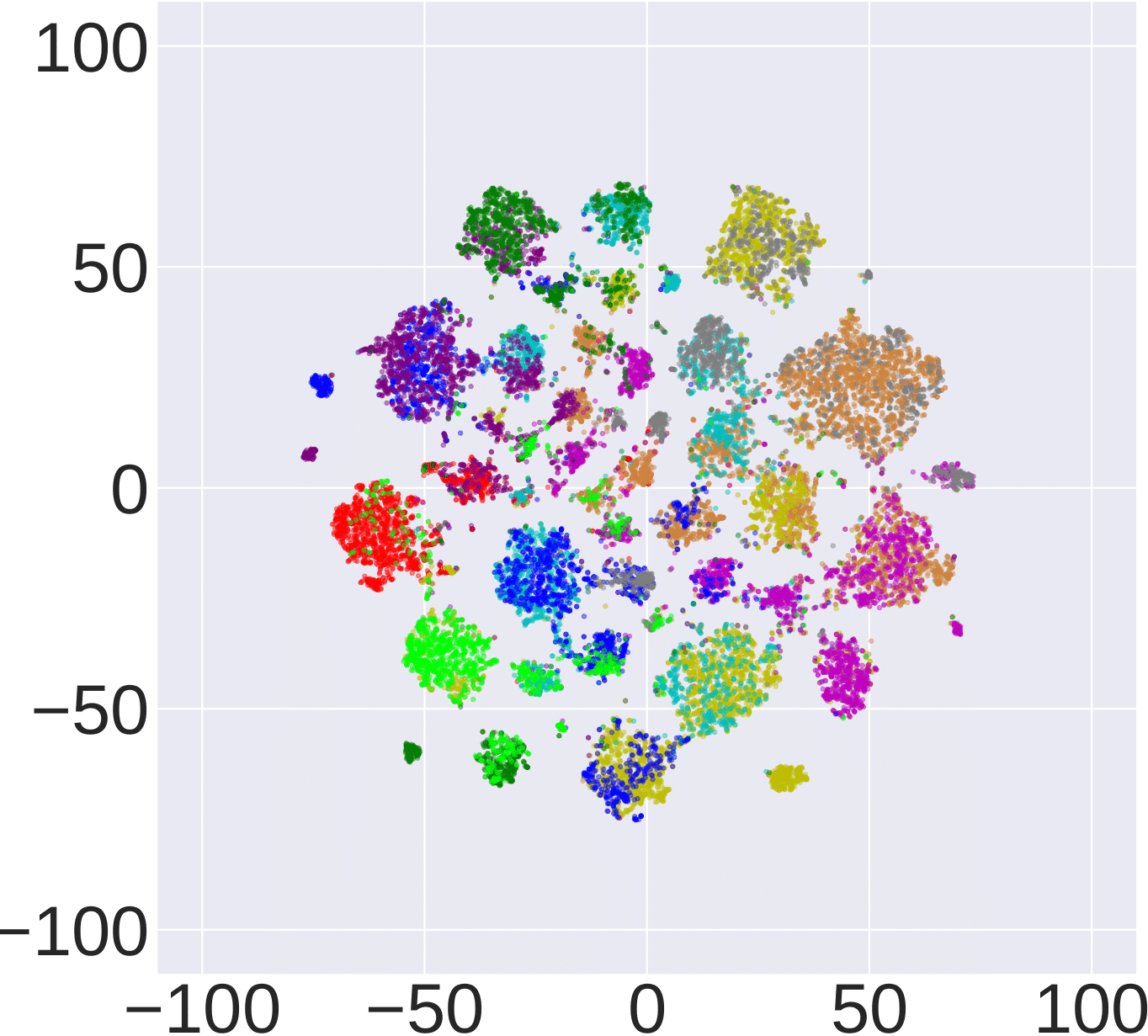}}
        \caption{\footnotesize
       Visualization of outputs using t-SNE. This figure visualizes outputs of the second to last layers in ResNet-18 and ResNet-34. Different colors represent different semantic meanings (i.e., different classes in the testing set of the \textit{SVHN}).}
        \label{fig:svhn}
    \end{center}
\end{figure*}

\section{Experiments Setup}
\label{Asec:exp_set}

We implement all methods on Python $3.7$ (Pytorch $1.1$) with a NVIDIA GeForce RTX2080 Ti GPU. The \textit{CIFAR-10} dataset and the \textit{SVHN} dataset can be downloaded via Pytorch. See the codes submitted. Given the $50,000$ images from the \textit{CIFAR-10} training set and $73,257$ digits from the \textit{SVHN} training set, we conduct a standard training on ResNet-18 and ResNet-34 for classification. Given the $100,000$ images from the \textit{Tiny-Imagenet} training set, we conduct a standard training on WRN-32-10 classification. DNNS are trained using SGD with $0.9$ momentum, the initial learning rate of $0.01$ and the batch size of $128$ for $150$ epochs. Based on these pre-trained models, adversarial data is generated from \emph{fast gradient sign method} (FGSM) \cite{goodfellow2014explaining}, \emph{basic iterative methods} (BIM) \cite{kurakin2016adversarial}, \emph{project gradient descent} (PGD) \cite{Madry18PGD}, \emph{Carlini and Wagner attack} (CW) \cite{carlini2017towards}, \emph{AutoAttack} (AA) \cite{croce2020reliable} and \emph{Square attack} (Square) \cite{andriushchenko2020square}. 

For each attack method, we generate eight adversarial datasets with the $L_\infty$-norm bounded perturbation $\epsilon \in [0.0235,0.0784]$. For attack methods except FGSM, maximum step $K = 20$ and step size $\alpha=\epsilon/10$. For the $10,000$ images from the \textit{CIFAR-10} testing set and $26,032$ digits from the \textit{SVHN} testing set, we only choose the adversarial data, whose original images are correctly classified by the pre-trained models. We extract semantic features from the second to last full connected layer of the well-trained ResNet-18 and ResNet-34. In the wild bootstrap process of SAMMD, given an alternative value set $\{0.1,0.2,0.5,1,5,10,15,20\}$, we choose the optimal $l$ (in Eq.~(\ref{eq:wb_generator})) depending on whose type I error on natural data is the most close to $\alpha$ ($\alpha = 0.05$ in the paper). The optimal value of $l$ we choose is $0.2$. For images from the \textit{CIFAR-10} (or the \textit{SVHN}) testing set and adversarial datasets generated above, we select a subset containing $500$ images of the each for $S_p^{tr}$ and $S_q^{tr}$, and train on that; we then evaluate on $100$ random subsets of each, disjoint from the training set, of the remaining data. We repeat this full process $20$ times, and report the mean rejection rate of each test. The learning rate of our SAMMD test and all baselines is $0.0002$.

\section{Additional Experiments}
\label{Asec:add_exp}

\begin{figure*}[!t]

    \begin{center}
        \subfigure[Six different attacks]
        {\includegraphics[width=0.245\textwidth]{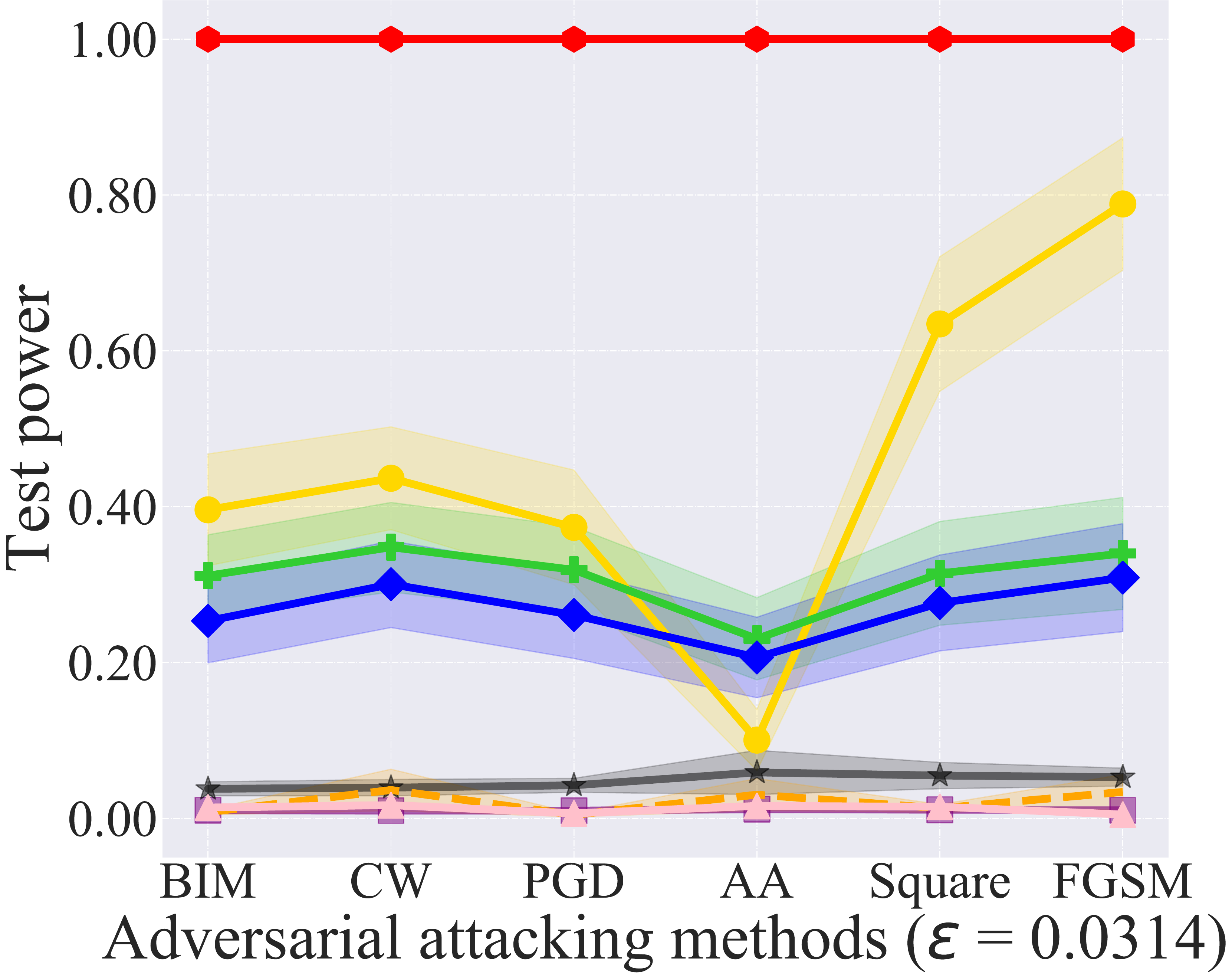}}
        \subfigure[Different $\epsilon$ of FGSM]
        {\includegraphics[width=0.245\textwidth]{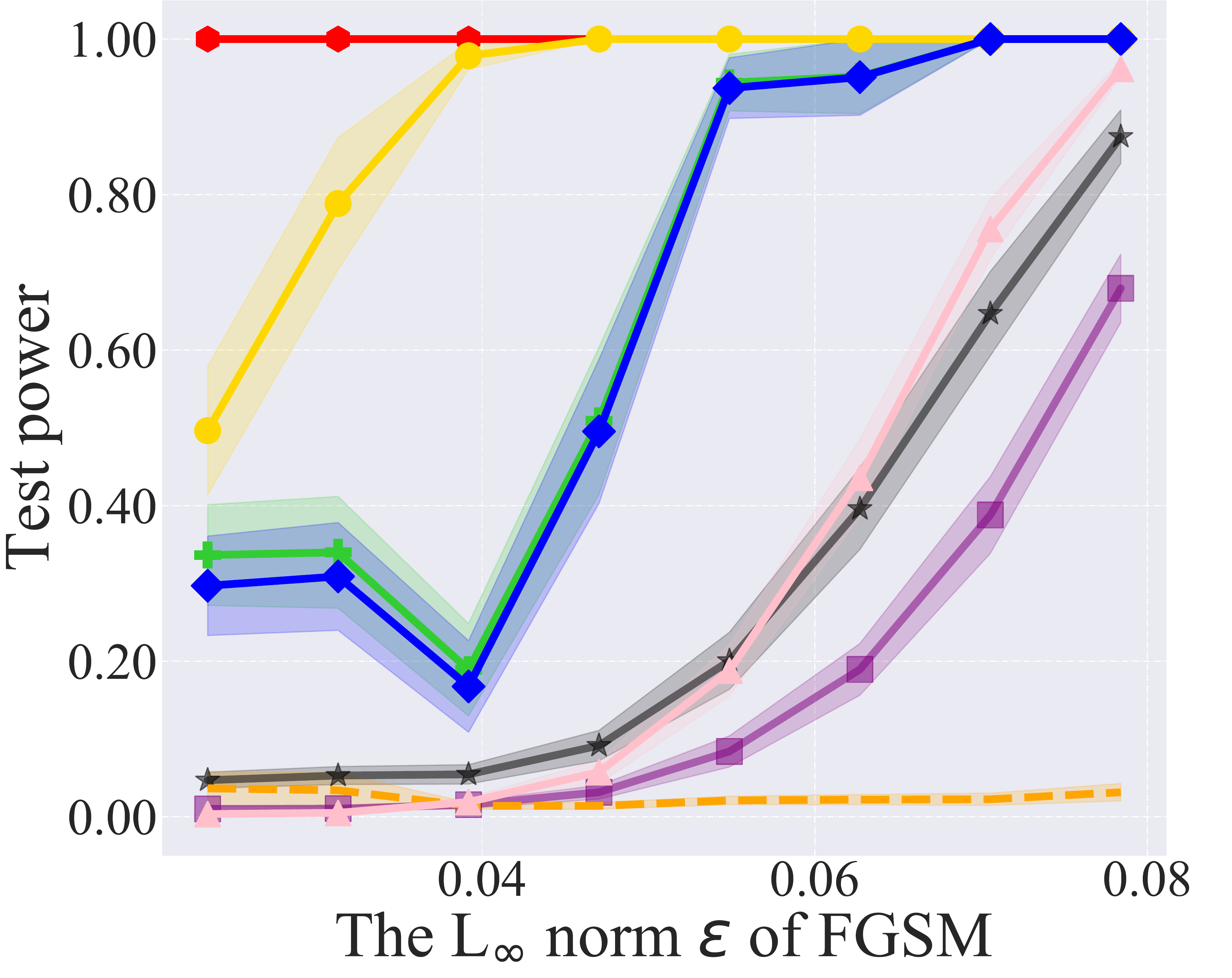}}
        \subfigure[Different $\epsilon$ of BIM]
        {\includegraphics[width=0.245\textwidth]{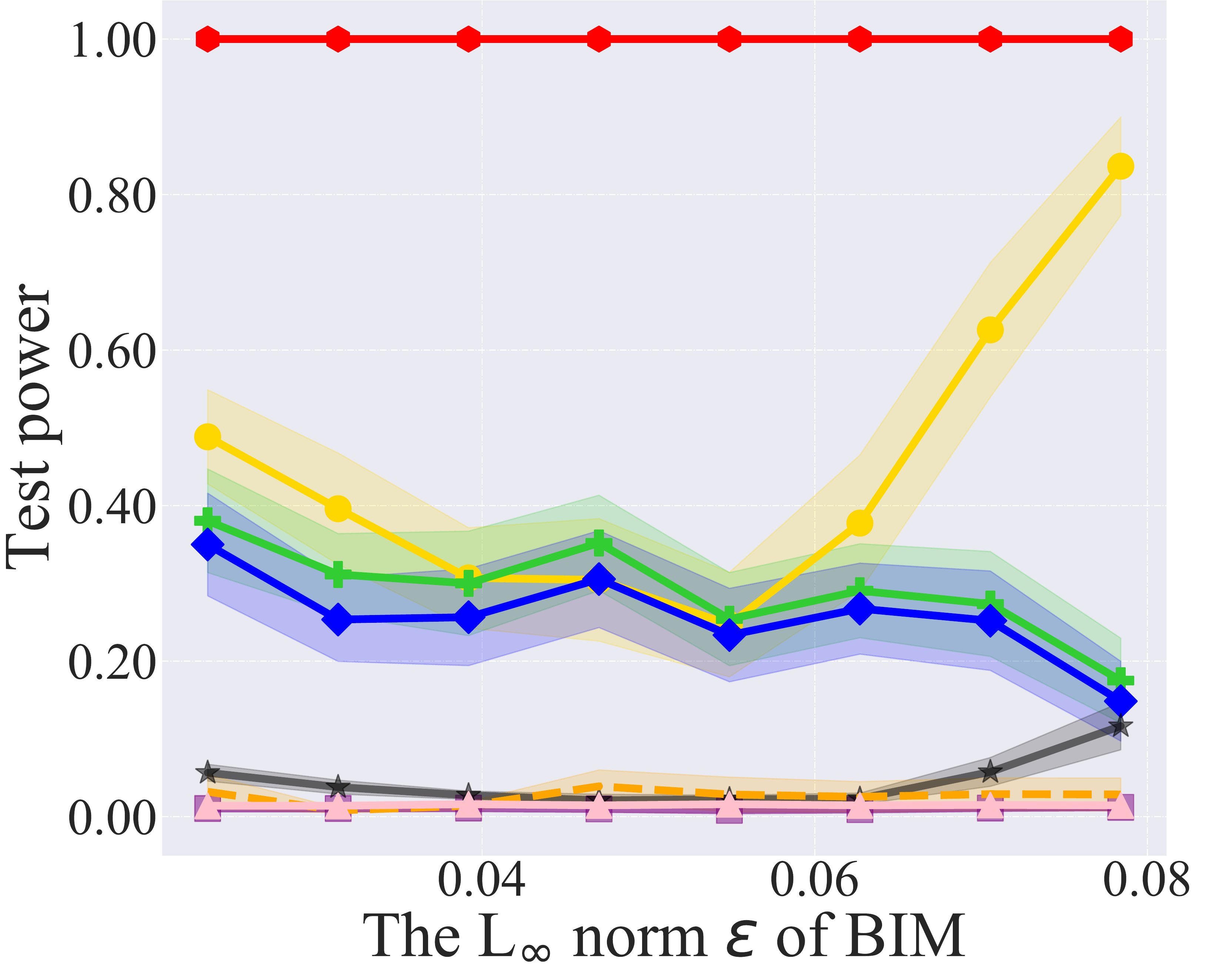}}
        \subfigure[Different $\epsilon$ of CW]
        {\includegraphics[width=0.245\textwidth]{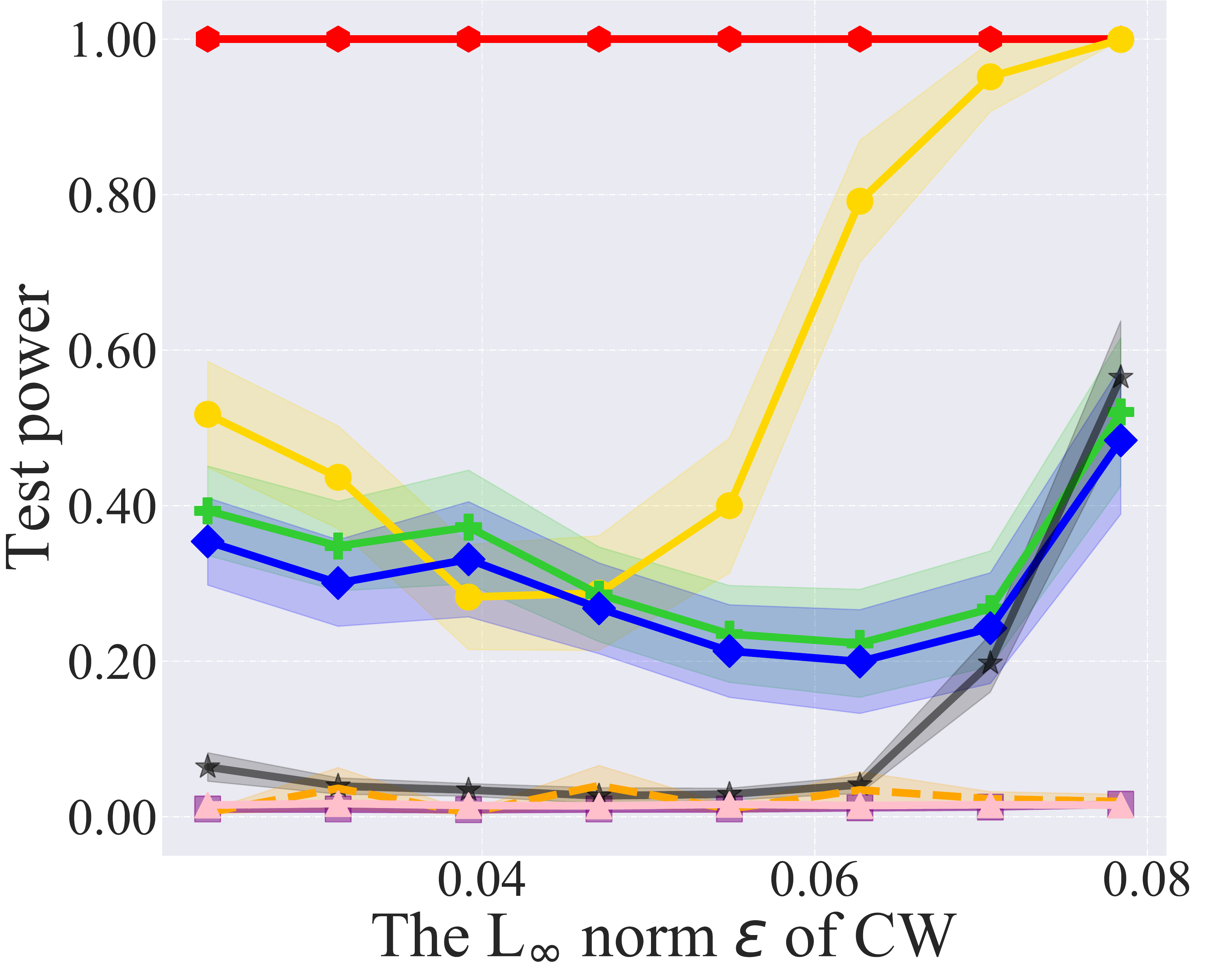}}
        \subfigure[Different $\epsilon$ of AA]
        {\includegraphics[width=0.245\textwidth]{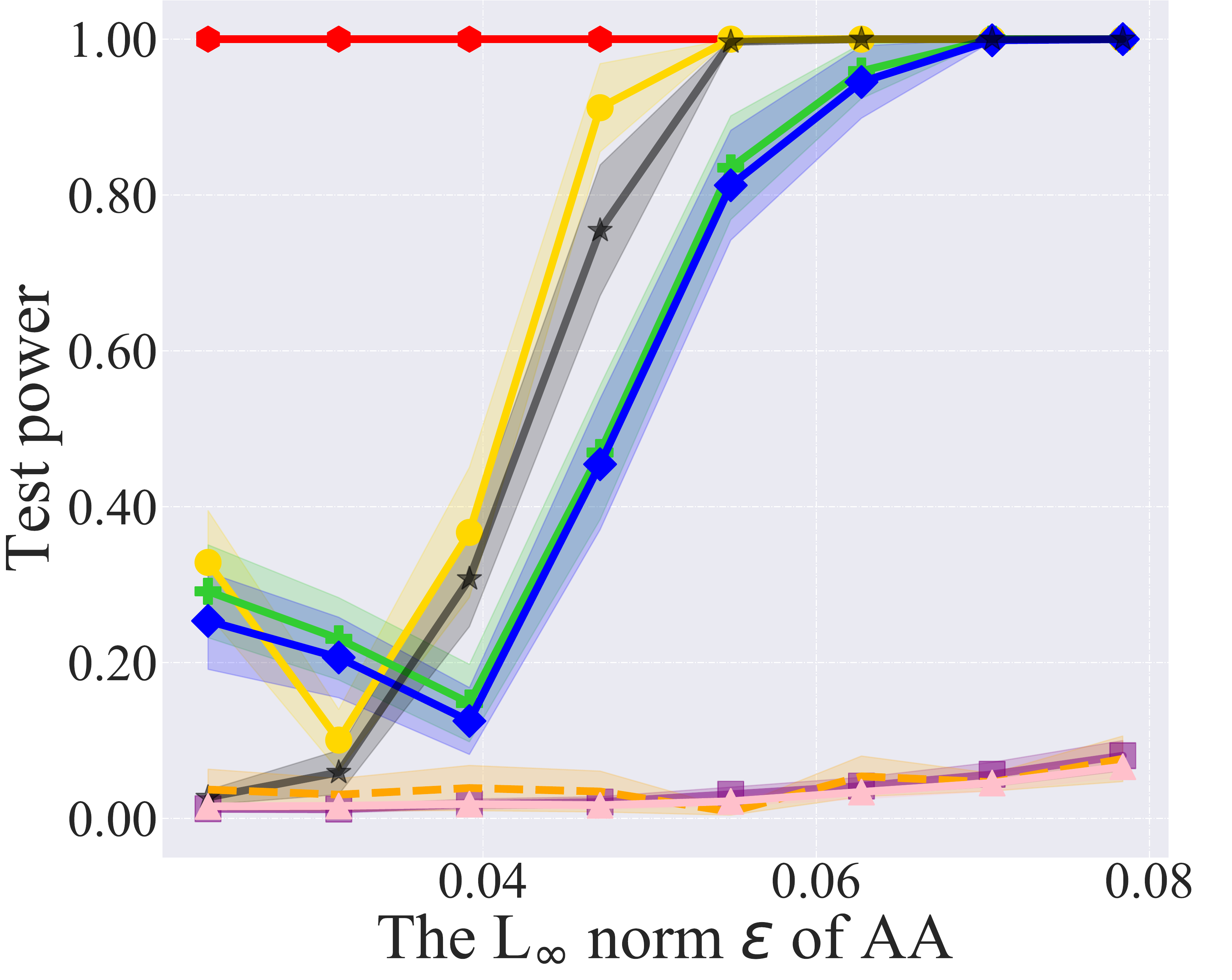}}
        \subfigure[Different $\epsilon$ of PGD]
        {\includegraphics[width=0.245\textwidth]{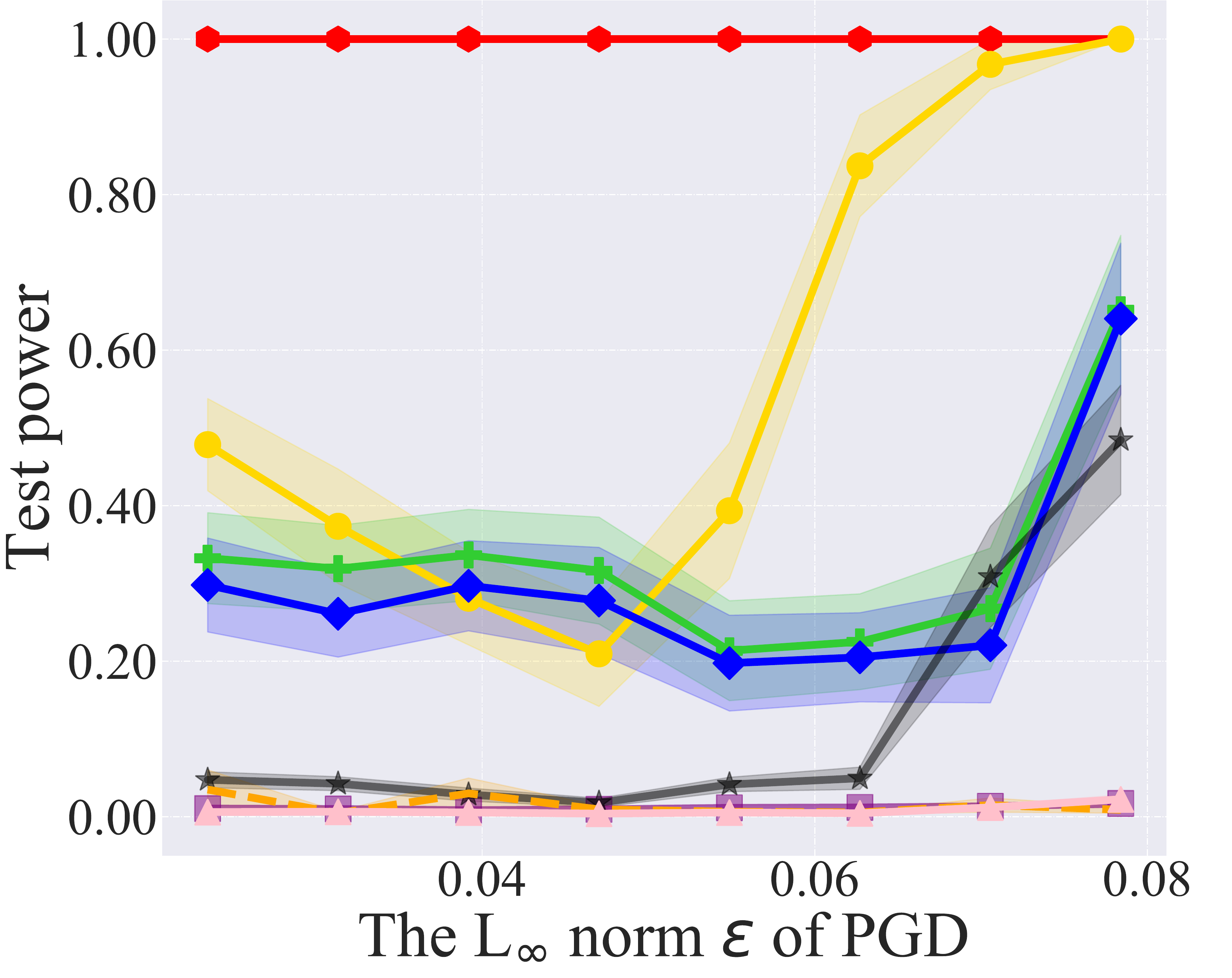}}
        \subfigure[Different set sizes]
        {\includegraphics[width=0.245\textwidth]{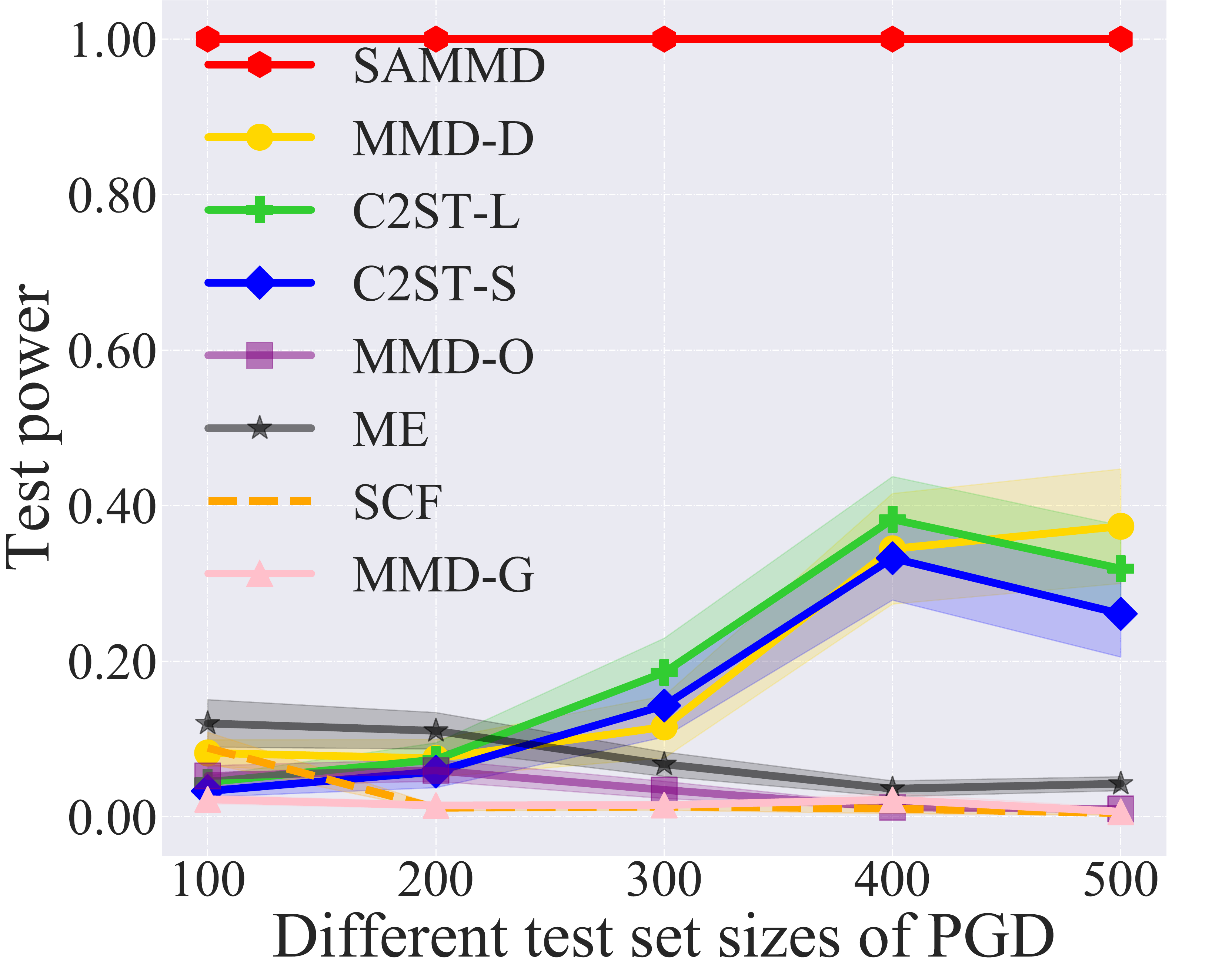}}
        \subfigure[Mixture of adv and natural]
        {\includegraphics[width=0.245\textwidth]{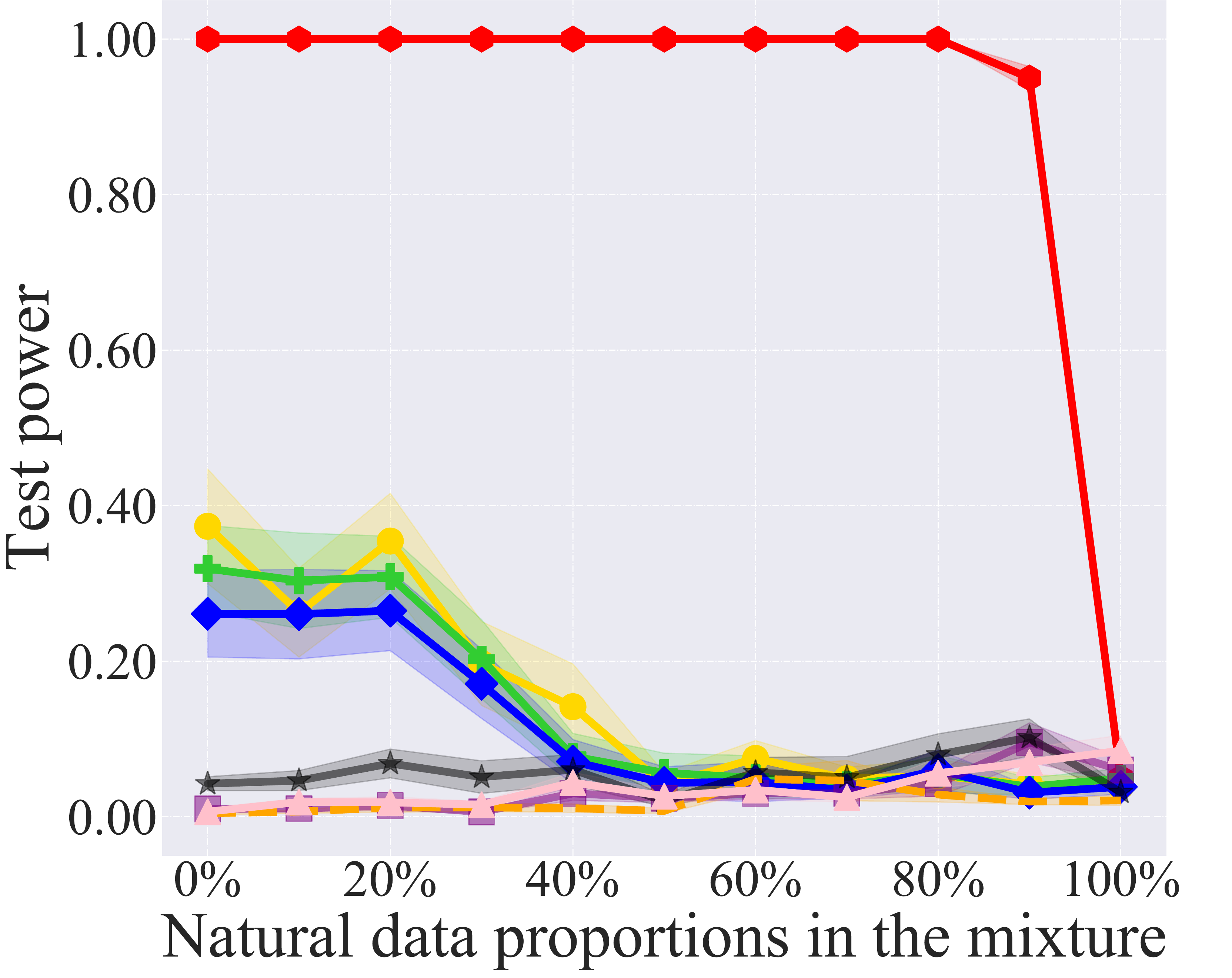}}
        \subfigure[Non-IID (a): PGD]
        {\includegraphics[width=0.245\textwidth]{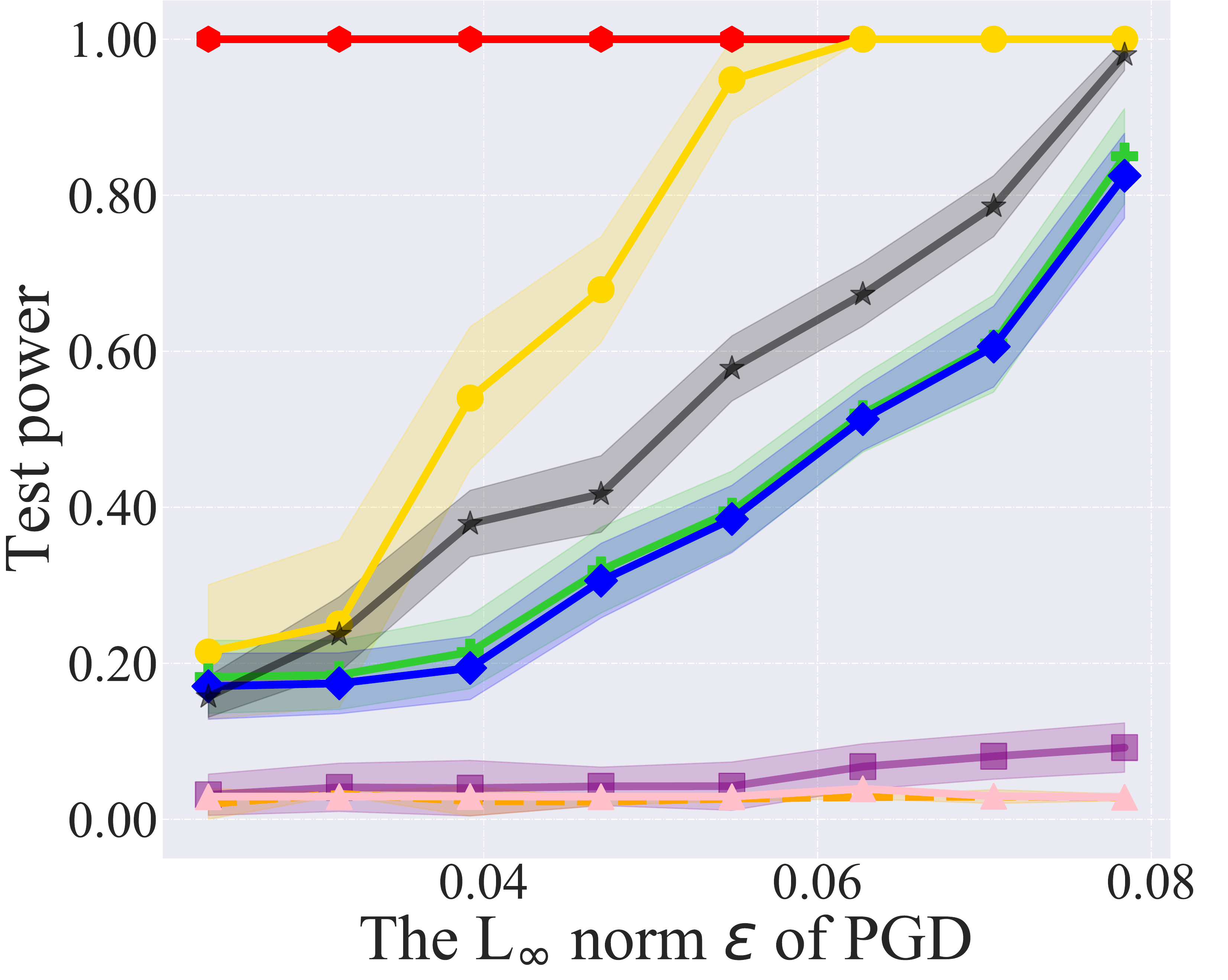}}
        \subfigure[Non-IID (b): Square]
        {\includegraphics[width=0.245\textwidth]{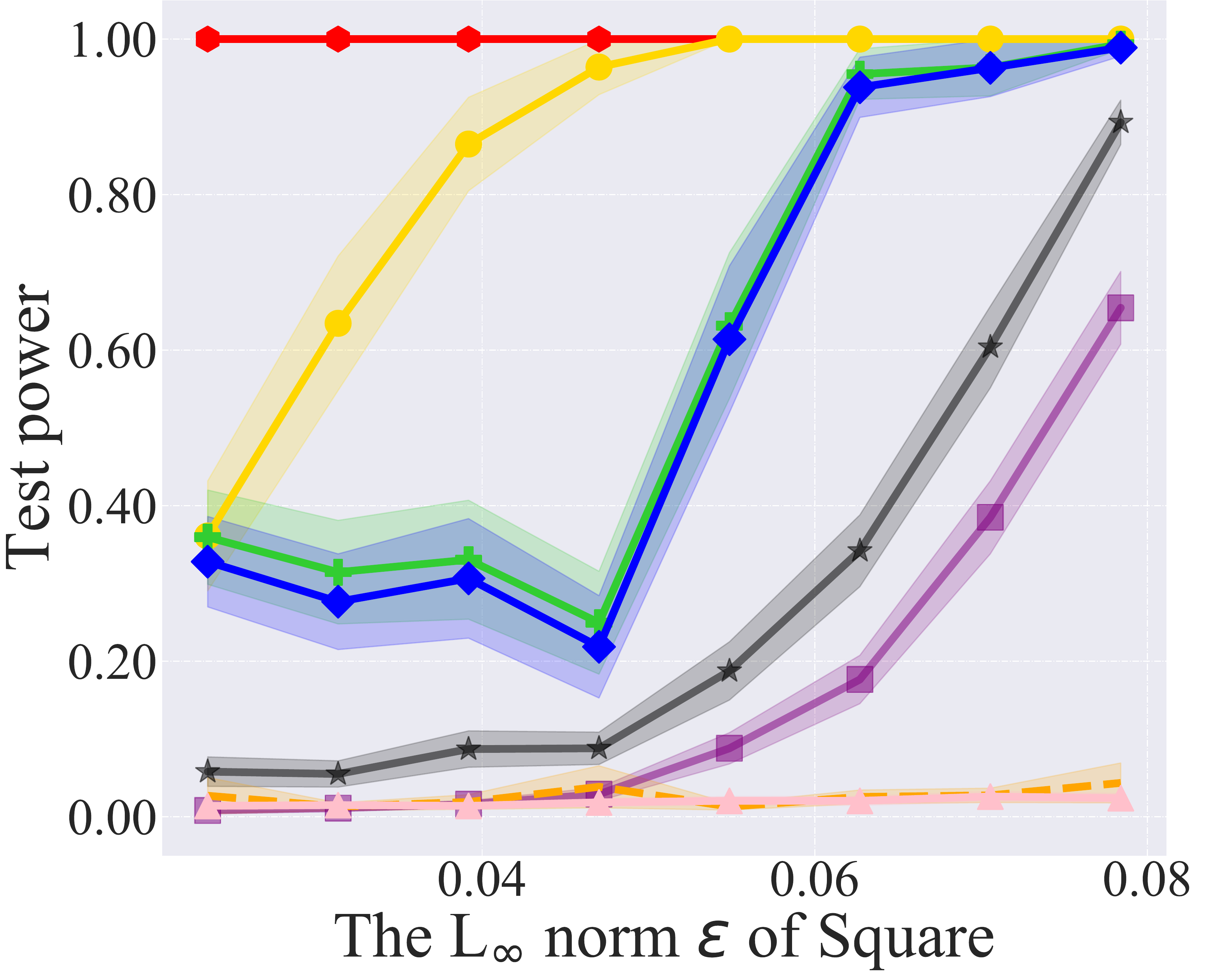}}
        \subfigure[WRN-Tiny: PGD]
        {\includegraphics[width=0.245\textwidth]{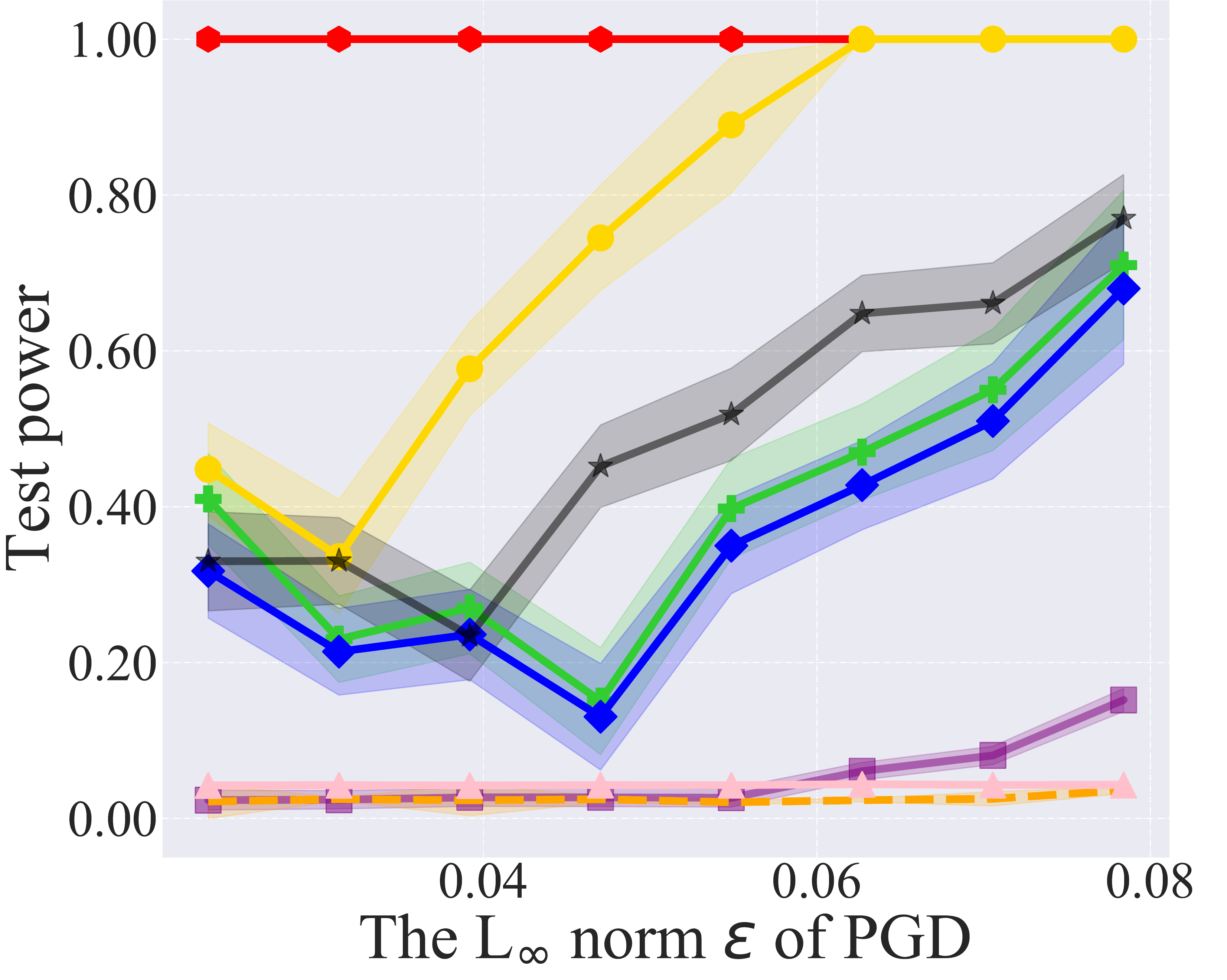}}
        \subfigure[WRN-Tiny: AA]
        {\includegraphics[width=0.245\textwidth]{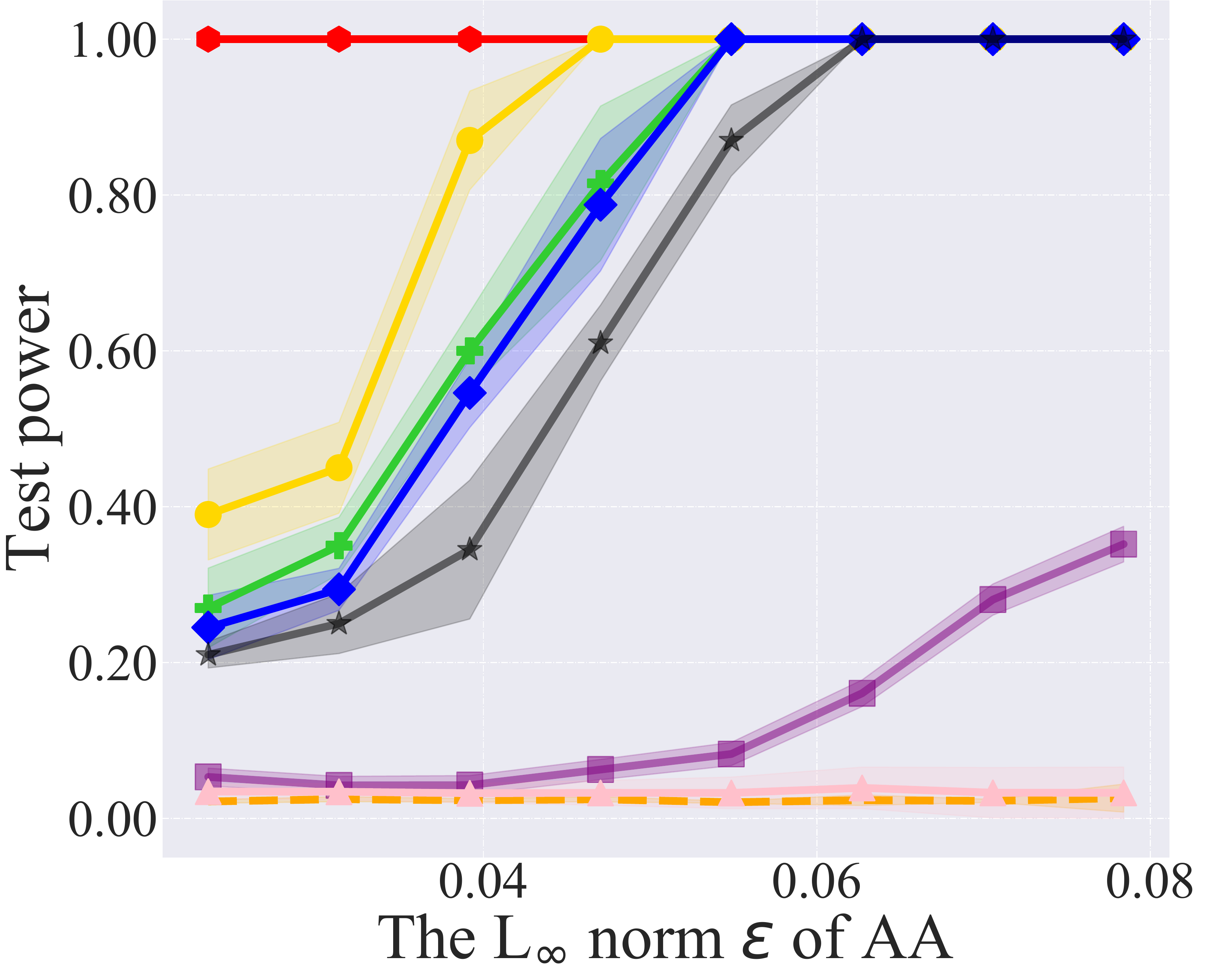}}
        \caption{\footnotesize Results of adversarial data detection. Subfigures (a)-(l) report the test power (i.e., the detection rate) when $S_Y$ are adversarial data. The ideal test power is $1$ (i.e., $100\%$ detection rate). Subfigure (a) - (j) are the experiments on the adversarial data of the \textit{CIFAR-10} acquired by ResNet-34. Subfigure (k)-(l) are the experiments on the adversarial data of the \textit{Tiny-Imagenet} acquired by WRN-32-10.
      }
        \label{fig:supp_results}
    \end{center}
\end{figure*}


\paragraph{Results of ResNet-34 on the \textit{CIFAR-10}.}
For the adversarial data acquired by ResNet-34 on \textit{CIFAR-10}, we also compare our SAMMD test with baselines in Figure~\ref{fig:supp_results}. Figure~\ref{fig:supp_results}{a} is a supplement of ResNet-18 to Figure~\ref{fig:results} that different $\epsilon$ of PGD. For $6$ different attacks, FGSM, BIM, PGD, AA, CW and Square (the Non-IID (b)), Figure~\ref{fig:supp_results}{b} reports the test power of all tests when $S_Y$ are adversarial data ($L_\infty$ norm $\epsilon = 0.0314$; set size $= 500$). Figure~\ref{fig:supp_results}(c)-(h) report the average test power on different $\epsilon$ of FGSM, BIM, AA, CW and Square (set size $=500$). Figure~\ref{fig:supp_results}{i} reports the average test power on different set sizes. Figure~\ref{fig:supp_results}{j} reports the average test power when adversarial data and natural data mix. Results show that our SAMMD test also achieves the highest test power.

\begin{figure*}[!t]
    \begin{center}
        \subfigure[Different $\epsilon$ of FGSM (RN18)]
        {\includegraphics[width=0.245\textwidth]{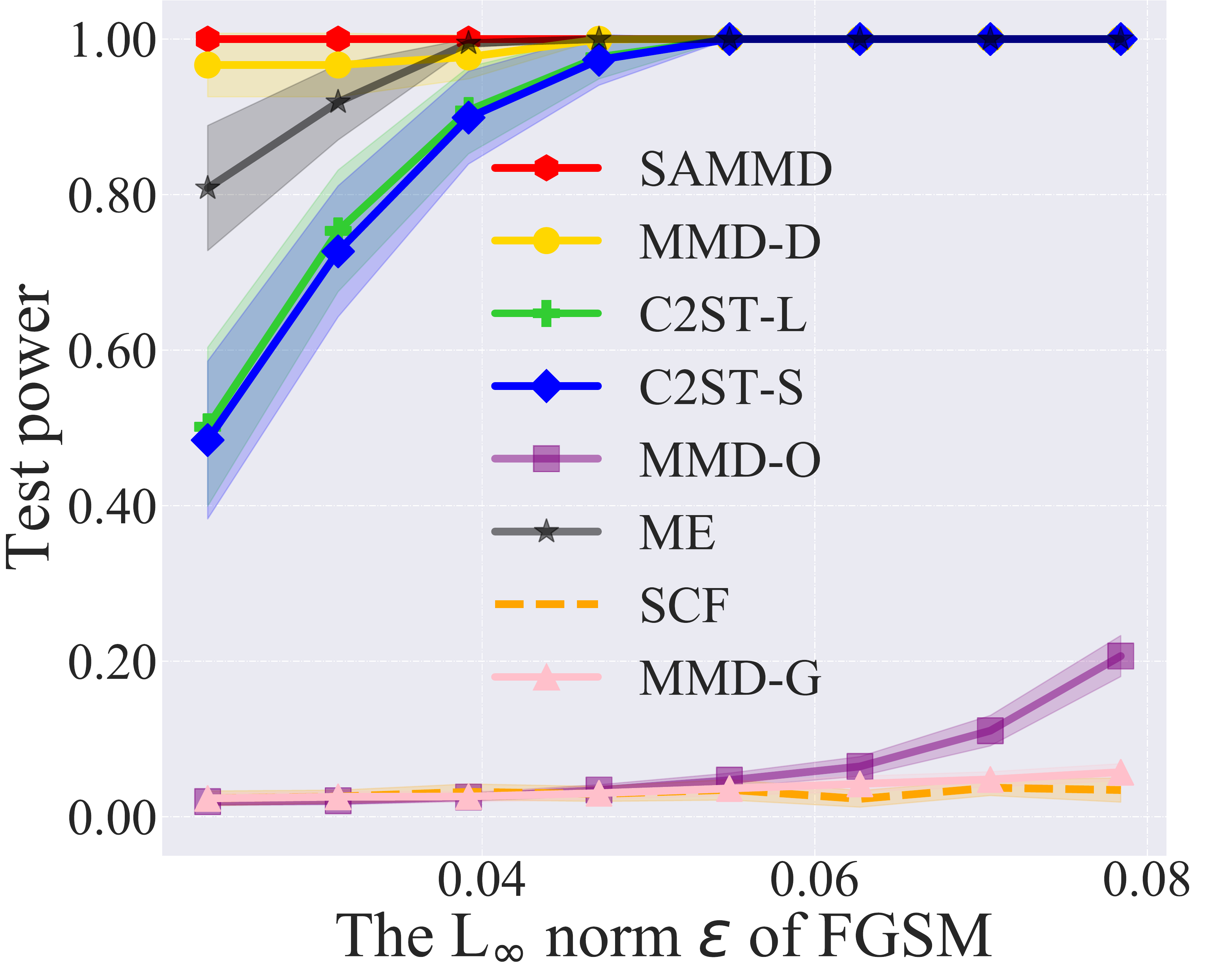}}
        \subfigure[Different $\epsilon$ of BIM (RN18)]
        {\includegraphics[width=0.245\textwidth]{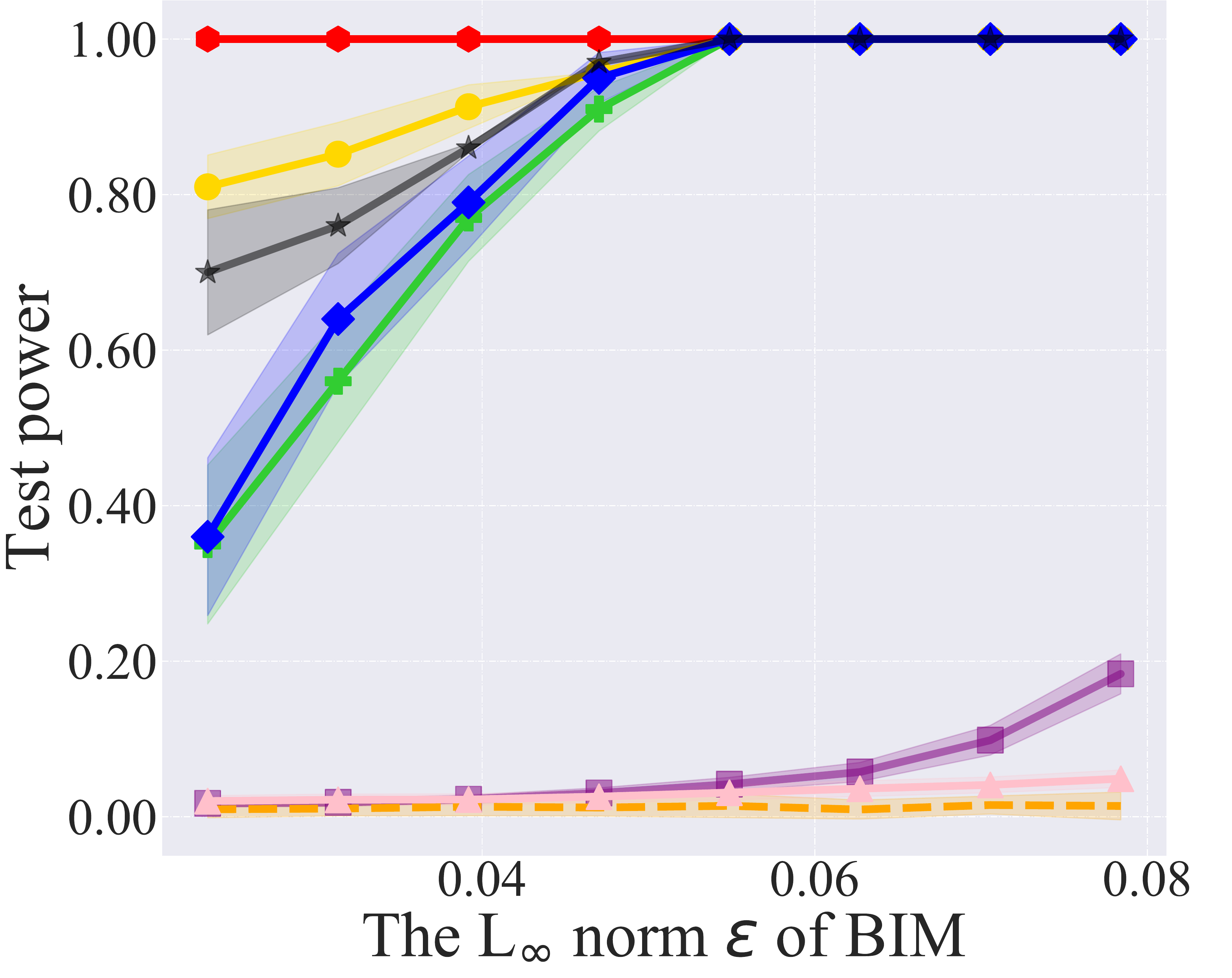}}
        \subfigure[Different $\epsilon$ of CW (RN18)]
        {\includegraphics[width=0.245\textwidth]{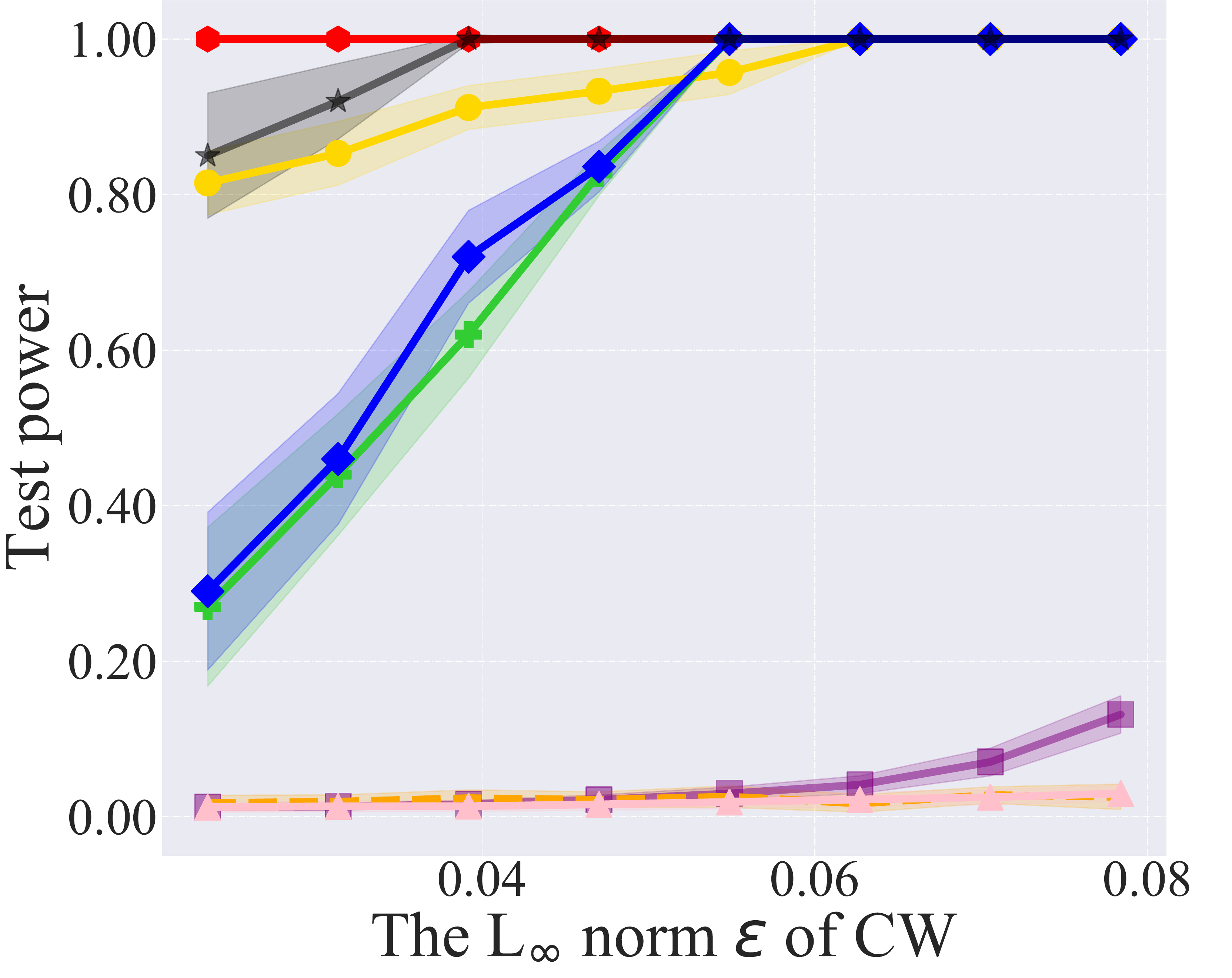}}
        \subfigure[Different $\epsilon$ of AA (RN18)]
        {\includegraphics[width=0.245\textwidth]{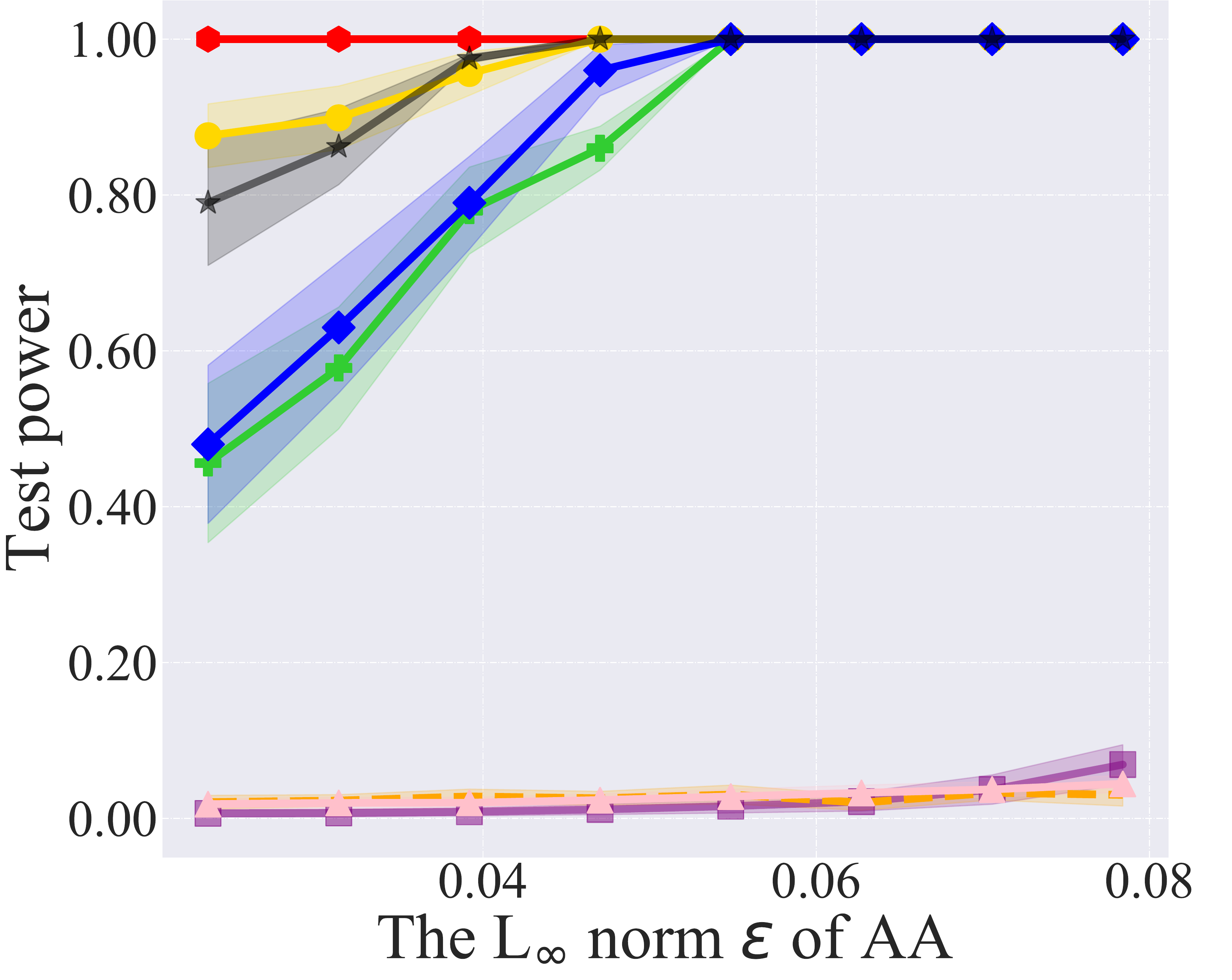}}
        \subfigure[Different $\epsilon$ of PGD (RN18)]
        {\includegraphics[width=0.245\textwidth]{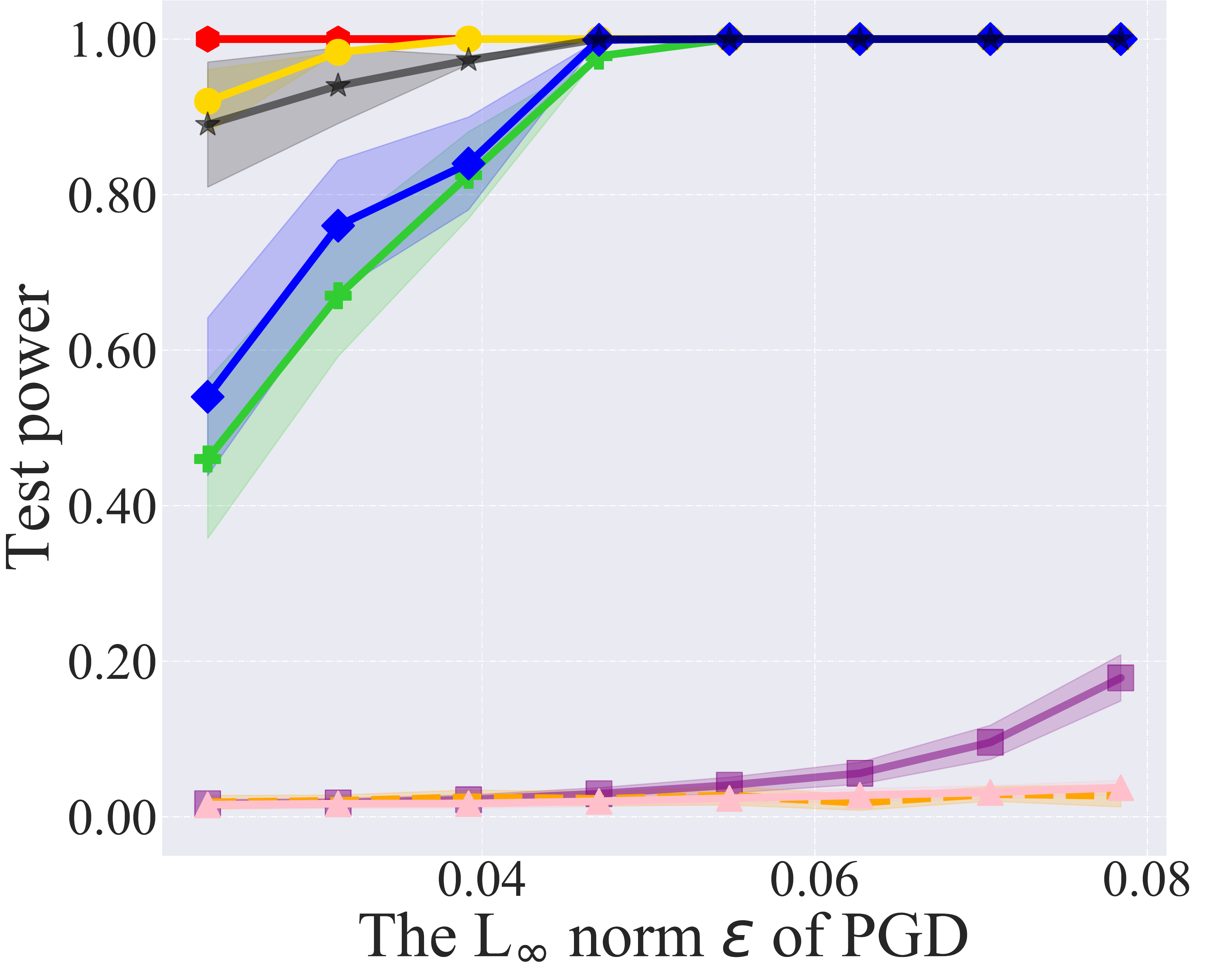}}
        \subfigure[Non-IID (b): Square (RN18)]
        {\includegraphics[width=0.245\textwidth]{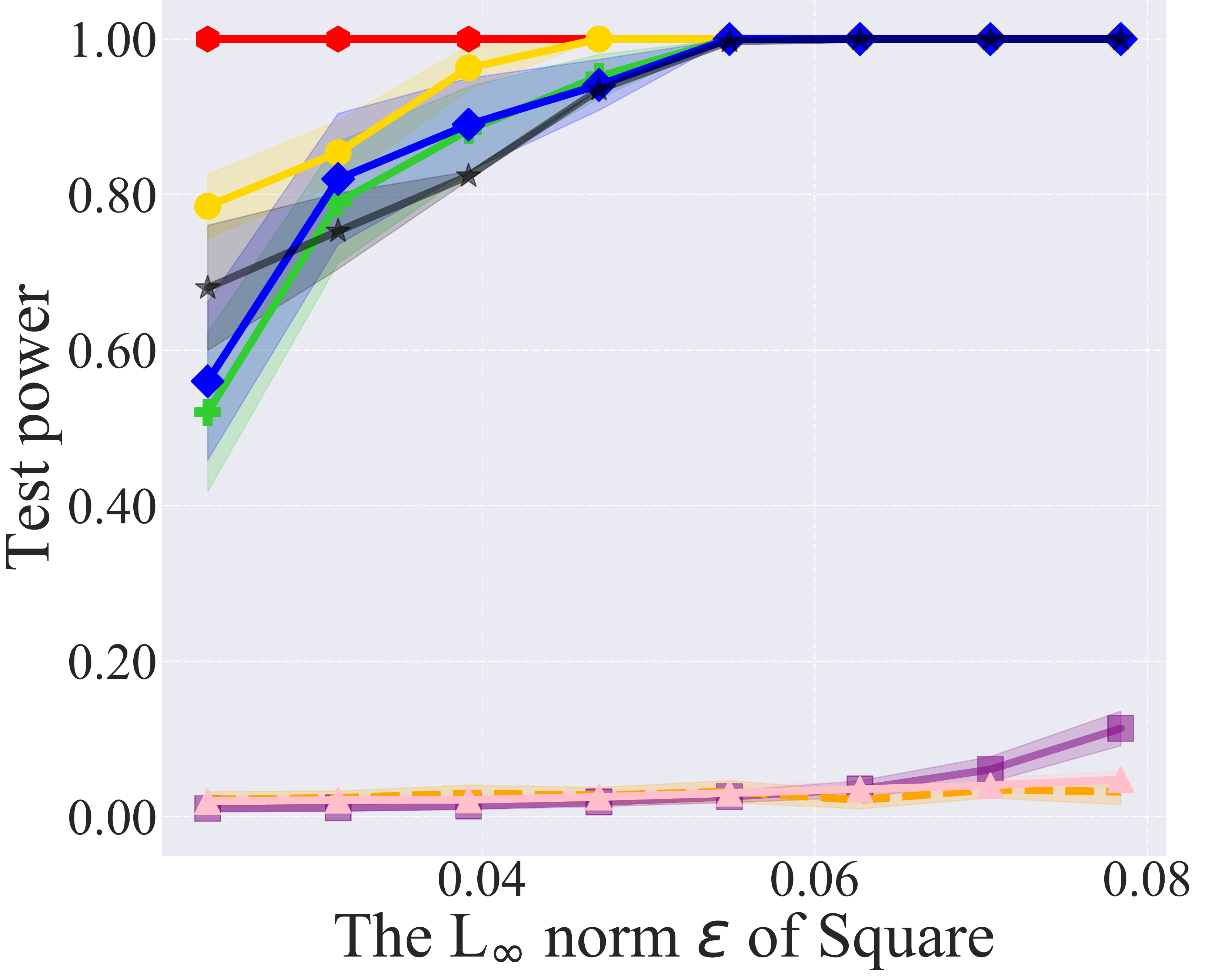}}
        \subfigure[Different $\epsilon$ of FGSM (RN34)]
        {\includegraphics[width=0.245\textwidth]{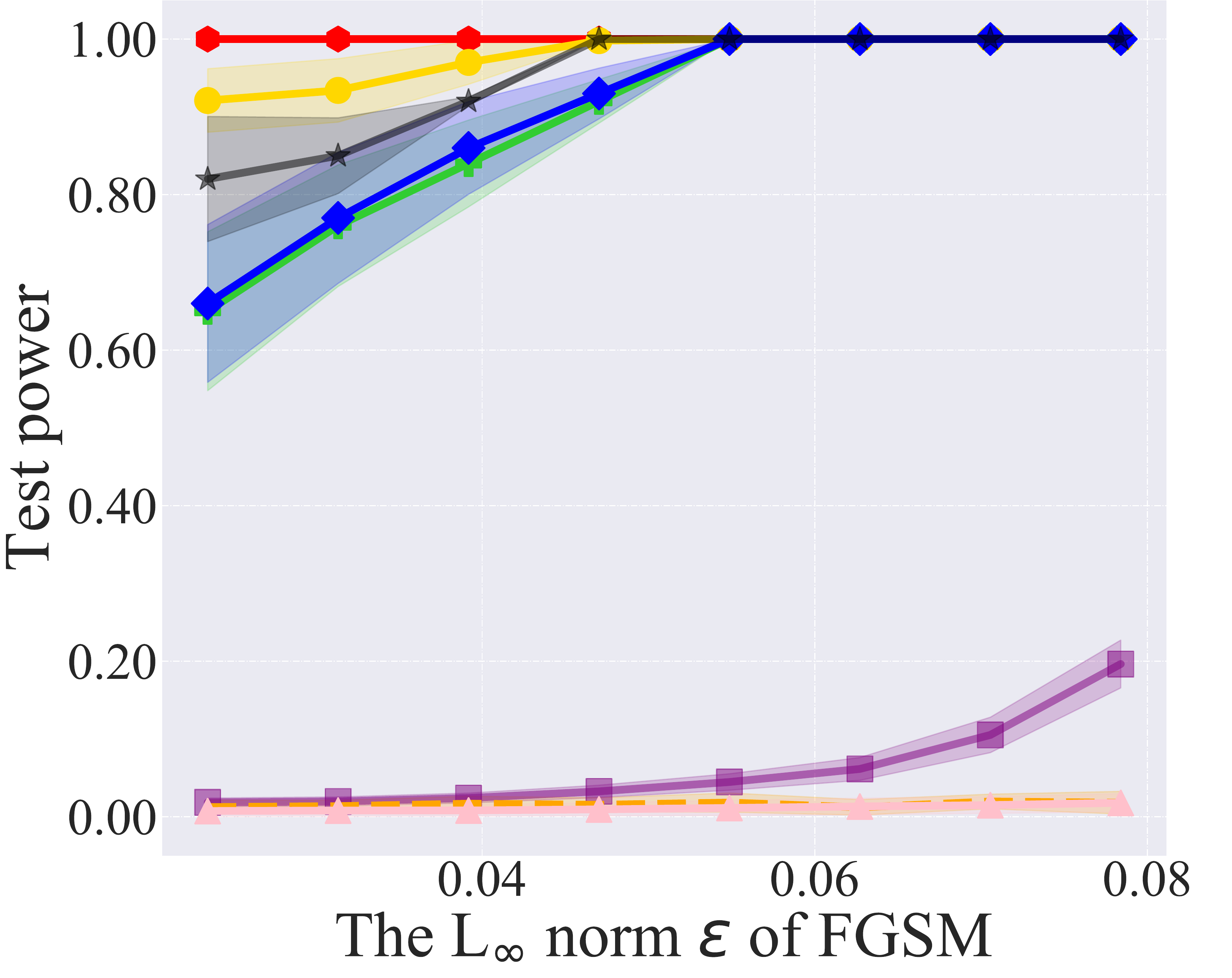}}
        \subfigure[Different $\epsilon$ of BIM (RN34)]
        {\includegraphics[width=0.245\textwidth]{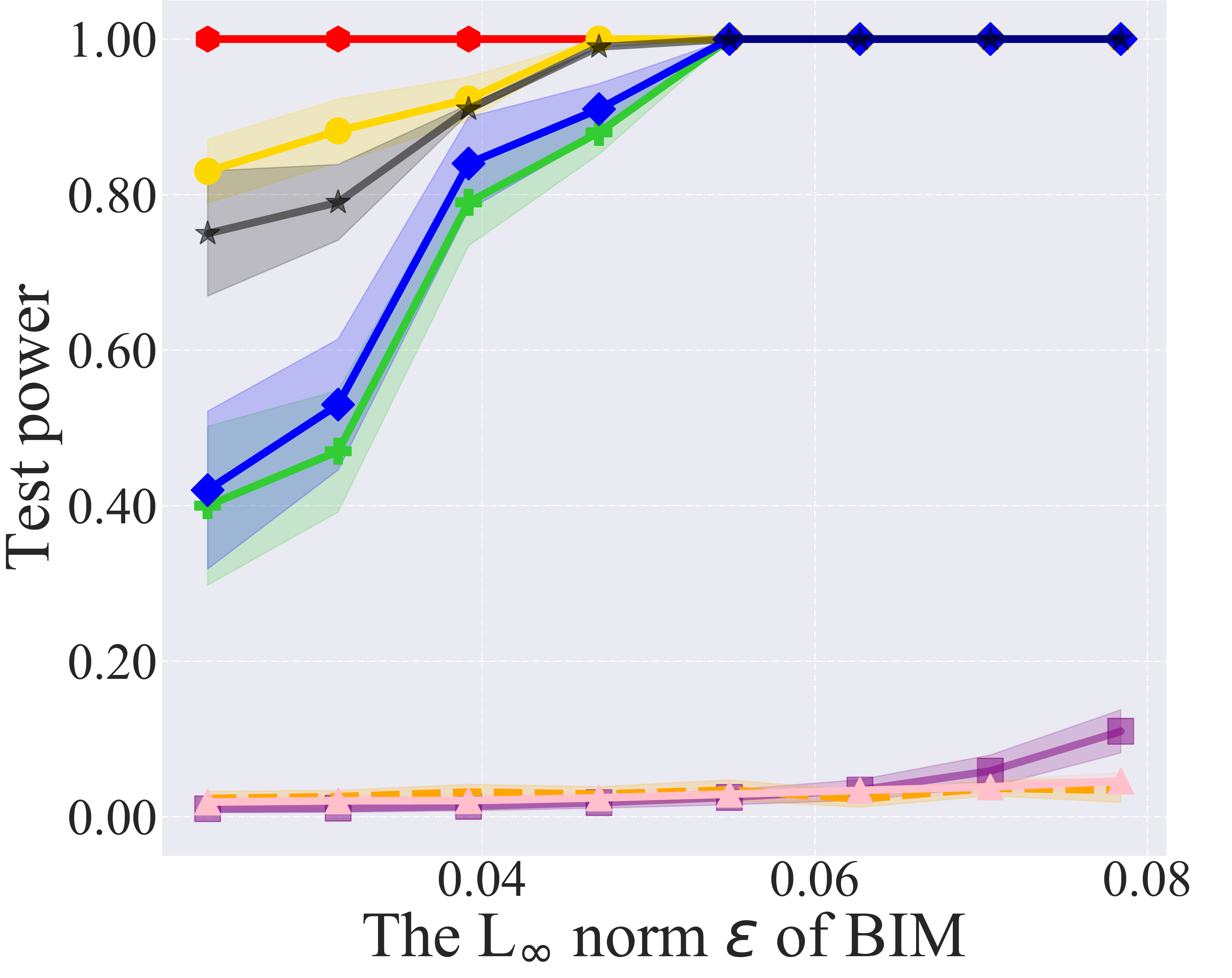}}
        \subfigure[Different $\epsilon$ of CW (RN34)]
        {\includegraphics[width=0.245\textwidth]{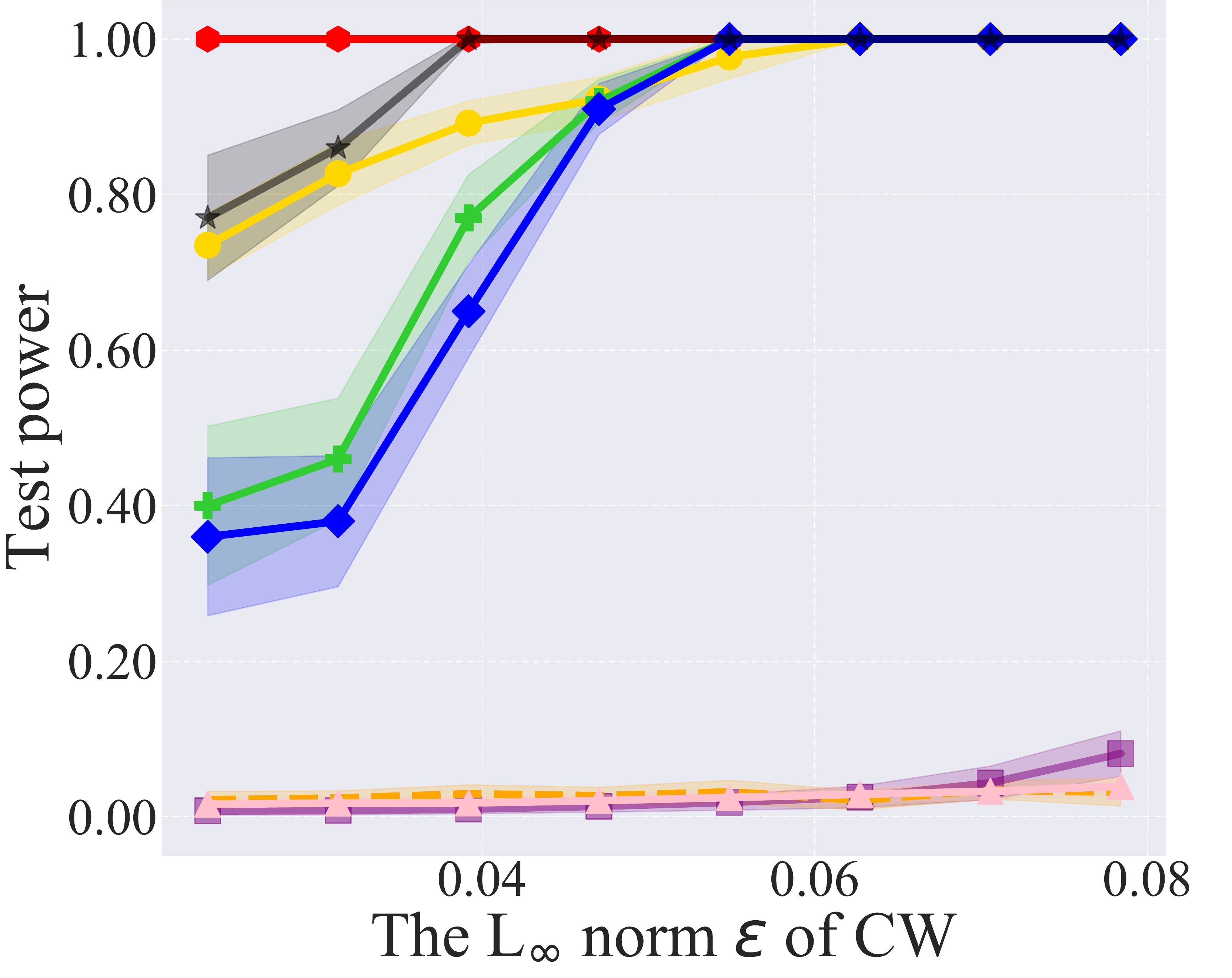}}
        \subfigure[Different $\epsilon$ of AA (RN34)]
        {\includegraphics[width=0.245\textwidth]{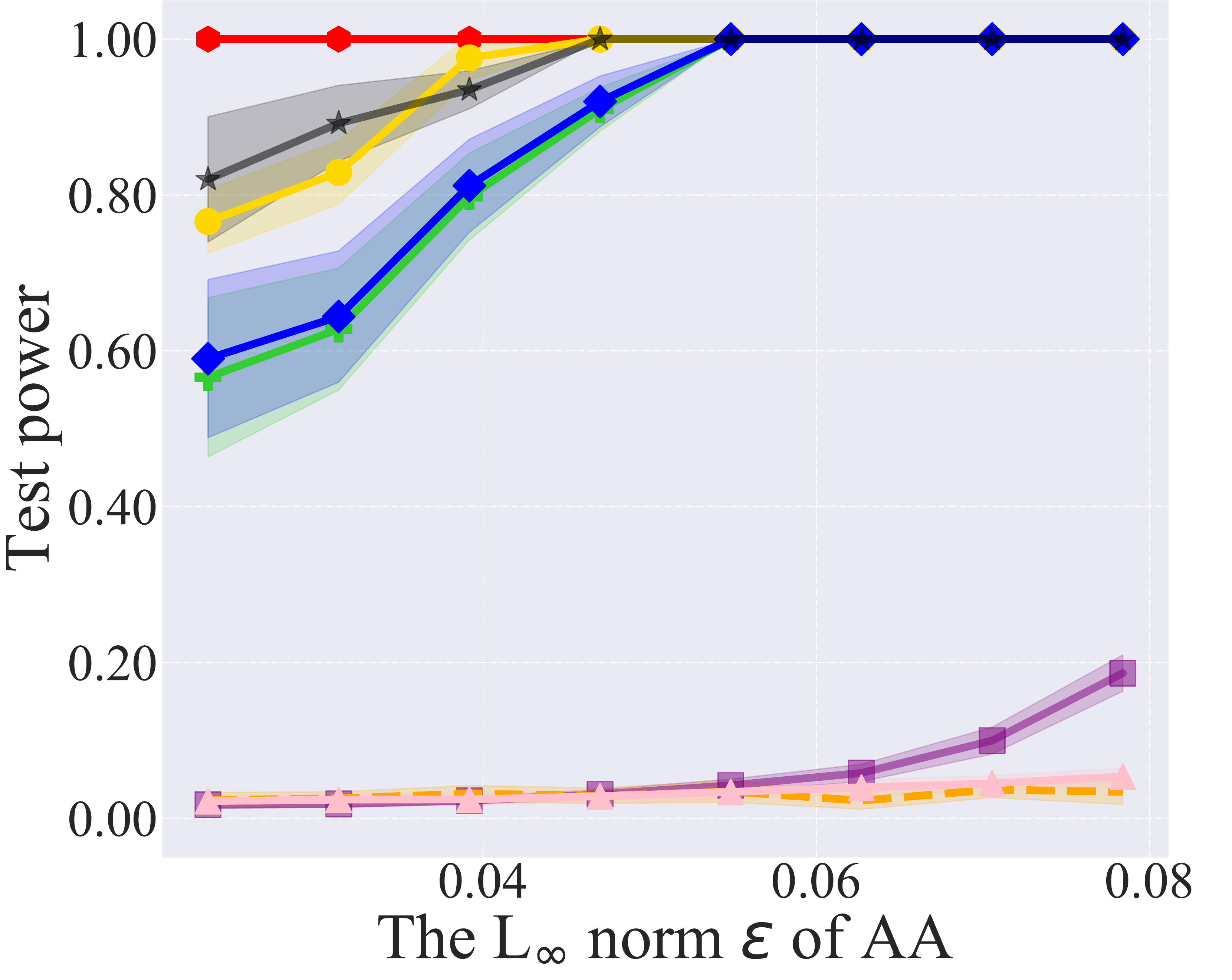}}
        \subfigure[Different $\epsilon$ of PGD (RN34)]
        {\includegraphics[width=0.245\textwidth]{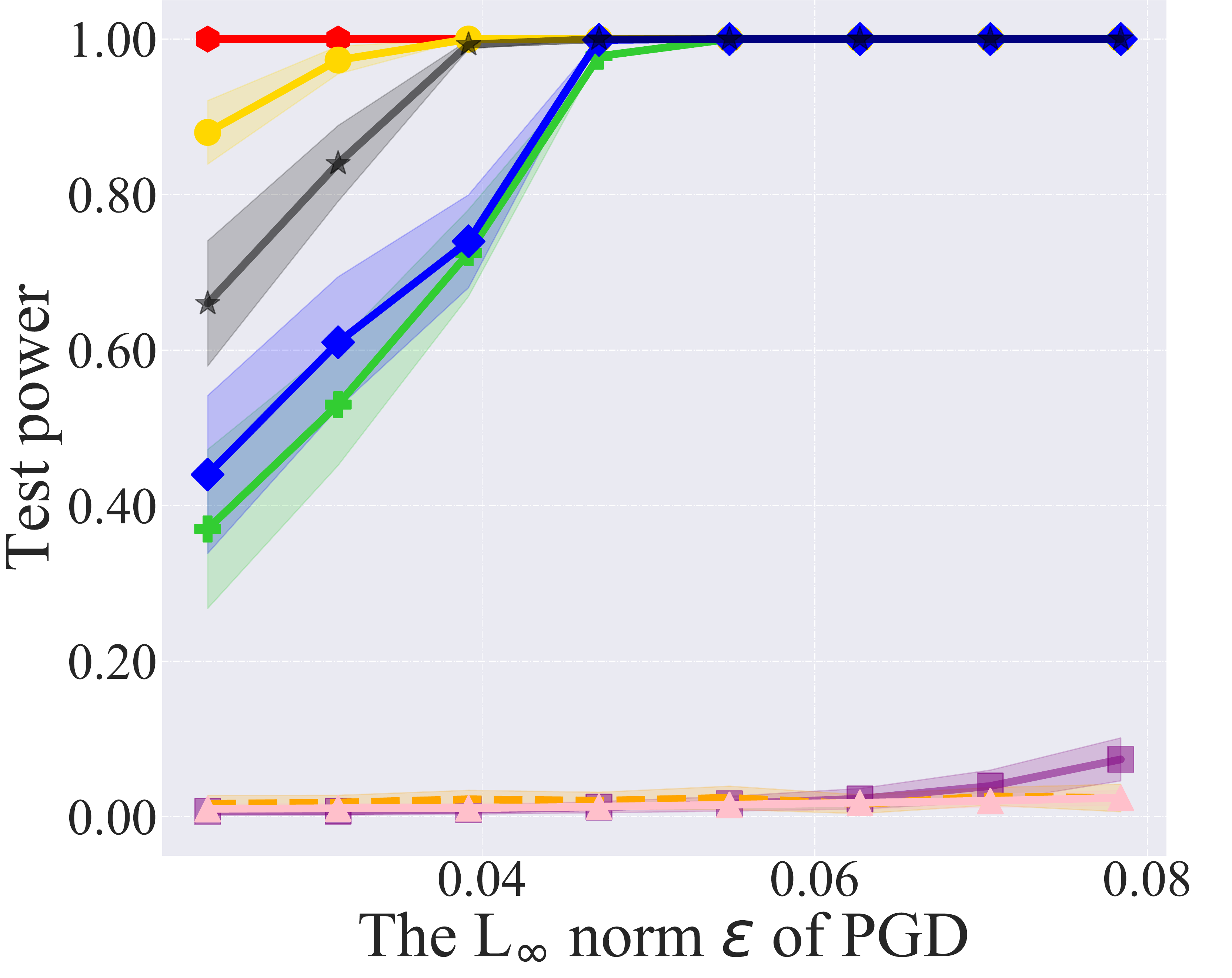}}
        \subfigure[Non-IID (b): Square (RN34)]
        {\includegraphics[width=0.245\textwidth]{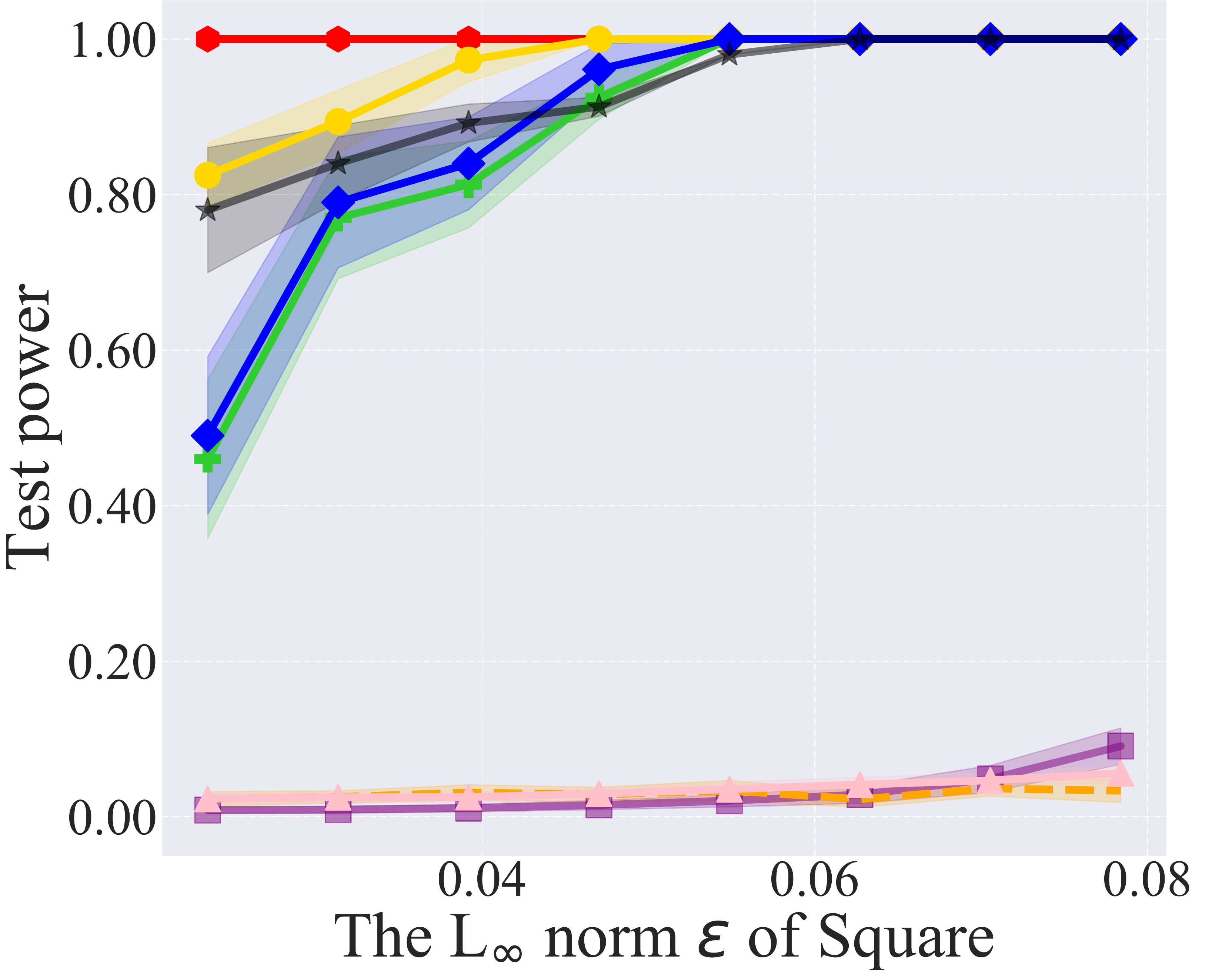}}

        \caption{\footnotesize Results of adversarial data detection on the \textit{SVHN}. Subfigure (a)-(l) report the test power (i.e., the detection rate) when $S_Y$ are adversarial data. The ideal test power is $1$ (i.e., $100\%$ detection rate).  
      }
        \label{fig:svhn_results}
    \end{center}
\end{figure*}

\begin{table*}[!t]
\caption{Average type I error within natural data and natural data on the \textit{SVHN} and \textit{Tiny-ImageNet}.}
\label{tab:SVHN_typeI}
\vspace{1mm}
\scriptsize
\centering

\begin{tabular}{p{7em}p{5.5em}p{5.5em}p{5.5em}p{5.5em}p{5.5em}p{5.5em}p{5.5em}p{5.5em}}

\toprule
Attack & SAMMD & MMD-D & C2ST-L & C2ST-S & MMD-O & ME & SCF & MMD-G\\
\midrule[0.6pt]
\midrule[0.6pt]
\textit{SVHN} & 0.053$\pm$0.017 & 0.037$\pm$0.011 & 0.036$\pm$0.010 & 0.043$\pm$0.012 & 0.017$\pm$0.004 & 0.022$\pm$0.005 & 0.022$\pm$0.006 & 0.015$\pm$0.004\\
\midrule
\textit{Tiny-ImageNet} & 0.049$\pm$0.015 & 0.046$\pm$0.019 & 0.051$\pm$0.023 & 0.052$\pm$0.016 & 0.048$\pm$0.010 & 0.039$\pm$0.007 & 0.021$\pm$0.008 & 0.047$\pm$0.013\\
\bottomrule
\end{tabular}
\end{table*}

\paragraph{Results of ResNet-18 and ResNet-34 on the \textit{SVHN}.} We compare the SAMMD test with $6$ existing two-sample tests on the \textit{SVHN}. All baselines and experiments setting are the same as those stated in Section~\ref{experiments}. We report the type I error in Table~\ref{tab:SVHN_typeI}. The ideal type I error should be around $\alpha$ ($0.05$ in this paper). For $6$ different attacks that FGSM, BIM, PGD, CW, AA, Square and different $L_\infty$-norm bounded perturbation $\epsilon$, we report the test power of all tests when $S_Y$ are adversarial data in Figure~\ref{fig:svhn_results}. Results show that our SAMMD test also performs the best. Compared to results on \textit{CIFAR-10}, adversarial data generated on the \textit{SVHN} is more easily detected by these state-of-the-art tests. 

\paragraph{Results of Wide ResNet on the \textit{Tiny-Imagenet}.} We also validate the effectiveness of SAMMD on the larger network WRN-32-10 and the larger dataset \textit{Tiny-Imagenet}. All baselines and experiments setting are the same as those stated in Section~\ref{experiments}. We report the type I error in Table~\ref{tab:SVHN_typeI}. For different attacks PGD and AA, we report the test power of all tests when $S_Y$ are adversarial data in Figure~\ref{fig:supp_results} (k)-(l). Results show that our SAMMD test also performs the best.

\paragraph{Time complexity of the SAMMD test.} Let $E$ denote the cost of computing an embedding $\phi_p(\vx)$, and $K$ denote the cost of computing $s_{\hat{f}}(\vx,\vy)$ given $\phi_p(\vx)$, $\phi_p(\vy)$ in Eq.~(\ref{eq:deepkernel_SAMMD}). Then each iteration of training in Algorithm~\ref{alg:learn_deep_kernel} costs $O(mE +m^2K)$, where $m$ is the minibatch size.

\paragraph{The average runtime.}
For images from the \textit{CIFAR-10} testing set and adversarial datasets generated by PGD, we select the subset containing $500$ images of the each for $S_p^{tr}$ and $S_q^{tr}$, and train on that; we then evaluate on $100$ random subsets of each, disjoint from the training set, of the remaining data. We repeat this full process $1$ times and report the average runtime of our SAMMD test and baselines in Table~\ref{tab:runtime}, and the units are seconds.
\begin{table*}[!t]
\caption{The average runtime of the SAMMD test and baselines.}

\label{tab:runtime}
\vspace{1mm}
\scriptsize
\centering

\begin{tabular}{p{7em}p{5.5em}p{5.5em}p{5.5em}p{5.5em}p{5.5em}p{5.5em}p{5.5em}p{5.5em}}

\toprule
Attack & SAMMD & MMD-D & C2ST-L & C2ST-S & MMD-O & ME & SCF & MMD-G\\
\midrule
Runtime$(s)$ & 12.51$\pm$2.97 & 47.26$\pm$5.92 & 48.82$\pm$4.28 & 160.78$\pm$13.47 & 11.13$\pm$2.15 & 56.25$\pm$8.34 & 3.59$\pm$1.08 & 1.23$\pm$0.17\\
\bottomrule
\end{tabular}
\end{table*}
\end{document}